\documentclass{article}
\usepackage{iclr2027_conference,times}
\usepackage{amsmath,amssymb,amsthm,graphicx,booktabs,longtable,array,multirow,float}
\usepackage[table]{xcolor}
\definecolor{rivatrow}{gray}{0.94}
\usepackage{hyperref,url}
\hypersetup{hidelinks,
  pdftitle={RiVaT-Fuse: Reliability-Calibrated Variational Tensor Fusion with Matrix-Valued Trust for Multimodal Prediction under Modality Uncertainty},
  pdfauthor={Yingfan Xu; Tieming Liu; Ye Liang},
  pdfsubject={Multimodal prediction under modality uncertainty}
}
\newcommand{\R}{\mathbb{R}}
\newcommand{\diag}{\operatorname{diag}}

\newcommand{\sg}{\operatorname{sg}}
\newcommand{\method}{RiVaT-Fuse}
\newtheorem{proposition}{Proposition}

\newcommand{\assetfigure}[3]{\begin{figure}[H]\centering\includegraphics[width=\linewidth,height=.34\textheight,keepaspectratio]{#1}\caption{#2}\label{#3}\end{figure}}
\newcommand{\assetfigurelarge}[3]{\begin{figure}[H]\centering\includegraphics[width=\linewidth,height=.60\textheight,keepaspectratio]{#1}\caption{#2}\label{#3}\end{figure}}
\title{\raggedright\hyphenpenalty=10000\exhyphenpenalty=10000 RiVaT-Fuse: Reliability-Calibrated Variational Tensor Fusion with Matrix-Valued Trust for Multimodal Prediction under Modality Uncertainty}
\author{%
  Yingfan Xu$^{1}$\thanks{Corresponding author.}\quad
  Tieming Liu$^{1}$\quad
  Ye Liang$^{2}$\\[0.5em]
  {\normalfont\small $^{1}$School of Industrial Engineering and Management}\\
  {\normalfont\small Oklahoma State University}\\[0.3em]
  {\normalfont\small $^{2}$Department of Statistics, Oklahoma State University}
}
\iclrfinalcopy
\begin{document}
\maketitle
% Override the conference publication header for this preprint.
\lhead{Preprint}
\begin{abstract}
Multimodal prediction from images and structured metadata requires integrating complementary evidence whose reliability can vary across samples and latent directions. A single confidence weight per modality cannot capture this directional variation. We propose RiVaT-Fuse, a reliability-calibrated variational tensor fusion framework that formulates fusion as sample-wise latent-state estimation. For each image-metadata pair, the fused representation minimizes a quadratic objective combining agreement with modality embeddings, structured cross-modal interactions, and regularization. Positive-definite, low-rank-plus-diagonal trust matrices are conditioned on learned state descriptors and metadata completeness, allowing modality contributions to vary across latent directions. Additive, multiplicative, and relational interactions model cross-modal dependencies within the latent estimation objective. The resulting system admits a unique solution computed through a differentiable linear solve. A first-order analysis with fixed trust operators relates latent sensitivity to system conditioning and perturbations in modality embeddings and interactions. The framework further incorporates a state-binned entropic surrogate for conditional distributionally robust learning and task-coupled quadratic prediction heads. We instantiate RiVaT-Fuse on mBRSET, pairing retinal images with clinical and demographic metadata for diabetic retinopathy grading, diabetic macular edema detection, and referable-status prediction. Comparisons with unimodal and representation-level fusion baselines assess the predictive utility of the complete framework across these related clinical tasks.
\end{abstract}

\section{Introduction}
Images and structured metadata provide different kinds of evidence. A retinal photograph records visible morphology, whereas demographic and clinical variables supply context shared across observations. Their usefulness can vary across samples and across directions within a representation. A photograph can preserve some visual cues while obscuring others; a metadata record can be complete without being informative for every prediction. This motivates a fusion mechanism that expresses \emph{which latent directions to trust}, as well as which cross-modal interactions to retain.

Common feature-fusion mechanisms use concatenation, gates, modulation or attention to construct a joint representation. These operations offer complementary inductive biases: gating adjusts feature contributions, FiLM conditions features on another input, and tensor factorizations model multiplicative interactions \citep{arevalo2017gmu,perez2018film,zadeh2017tensor,liu2018efficient}. However, such a construction need not specify an estimation objective for the resulting latent vector. We investigate a different organizing principle: define the fused representation as the minimizer of a sample-conditioned quadratic energy, with positive-definite matrices determining agreement with each modality and structured interactions contributing to the estimation target.

\method{} combines three elements. First, low-rank-plus-diagonal trust matrices represent a direction-dependent geometry conditioned on learned state descriptors and metadata completeness. Second, additive, multiplicative and metadata-to-image relational interactions enter a shared variational objective. Third, a differentiable linear solve produces the representation used by related classification tasks. The same state descriptors also index a grouped entropic robust-training surrogate. This separates the inference mechanism, which has an exact quadratic minimizer, from the training approximation, which is implemented on mini-batch state groups.

We instantiate the framework on mBRSET \citep{mbrset2024,wu2025mbrset}, with diabetic retinopathy (DR) grading as the primary evaluation focus and diabetic macular edema (DME) and referable-status prediction as related auxiliary tasks. On a patient-disjoint test partition, \method{} attains DR accuracy 0.8693, quadratic weighted kappa (QWK) 0.8472 and macro-F1 0.5737. These are the highest DR point estimates among the evaluated HEALNet, MADDi and TIP adaptations under the shared supervised image protocol. We complement predictive comparisons with retrained ablations, fixed-weight matrix and state interventions, numerical sensitivity verification, and input-attribution analyses. Together, the results characterize both the predictive value and the internal behavior of the proposed fusion rule.

Our contributions are: (i) a variational formulation of image--metadata fusion with positive-definite, matrix-valued trust and structured interactions; (ii) explicit well-posedness and local sensitivity relations that can be checked against the implemented solve; and (iii) a patient-disjoint empirical study with complete comparator, ablation and mechanism records. The study supports strong DR performance in this setting; it does not assume that learned state values are calibrated probabilities or that every auxiliary metric improves.

\section{Related Work}
\paragraph{Multimodal representations and interactions.}
Learning shared representations from complementary inputs is a central theme of multimodal learning \citep{ngiam2011multimodal,baltrusaitis2018multimodal}. Gated multimodal units and feature-wise modulation learn conditional feature combinations \citep{arevalo2017gmu,perez2018film}. Tensor Fusion Networks and low-rank multimodal fusion explicitly represent cross-modal interactions \citep{zadeh2017tensor,liu2018efficient}. Compact bilinear pooling, Tucker fusion and bilinear attention provide different ways to control the cost and structure of multiplicative coupling \citep{fukui2016multimodal,benyounes2017mutan,kim2018bilinear}. Our multiplicative branch follows this interaction perspective, but contributes a vector to an estimation objective rather than directly defining the final joint representation. Attention-based representations \citep{vaswani2017attention} and latent-array architectures such as Perceiver \citep{jaegle2021perceiver} motivate flexible exchange between heterogeneous feature sets. For short-video classification, VLMAE conditions attention over audio features on a vision--language representation \citep{sun2025audio}. In multimodal in-context learning, CAMA adjusts attention logits without training to improve visual grounding and query-dependent information allocation across demonstrations \citep{li2026make}. In cross-video procedure planning, OSEF preserves query-conditioned evidence tokens from all candidate videos for planning, avoiding preliminary hard video or window selection \citep{ye2026osef}. Here, \emph{early fusion} means feature integration before final task predictions; some taxonomies call this intermediate fusion.

\paragraph{Optimization as a representation layer.}
OptNet differentiates through quadratic programs, and differentiable convex optimization layers extend this principle to structured convex problems \citep{amos2017optnet,agrawal2019differentiable}. Thus neither differentiating a solve nor using an optimization layer is itself our novelty. In discrete-choice modeling, a two-stage adapter combines foundation-model predictions with constrained utilities to preserve specified economic properties under a fixed-prediction protocol \citep{wang2026embedding}. Our contribution is the particular fusion energy: sample-conditioned positive-definite trust operators weight modality agreement, while additive, multiplicative and relational interactions shift the estimation target. Its unconstrained quadratic form yields an explicit linear system and permits direct checks of the resulting geometry and sensitivity.

\paragraph{Complete biomedical fusion frameworks.}
Image--clinical fusion has diverse input representations and evaluation protocols \citep{huang2020fusion}. HEALNet uses shared latent attention for heterogeneous biomedical inputs \citep{hemker2024healnet}. MADDi integrates imaging, clinical and genetic information with attention for Alzheimer's disease diagnosis \citep{golovanevsky2022maddi}. TIP learns image--tabular representations through contrastive, matching and masked-reconstruction pretraining \citep{du2024tip}. Our two-input, three-task adaptations use a shared supervised protocol; they do not reproduce the methods' original settings or directly compare their published cross-dataset scores. OC-Distill takes a complementary route, combining ontology-aware contrastive pretraining with cross-modal distillation from clinical notes to support ICU prediction using only vital signs at inference \citep{liang2026ocdistill}. ConChecker uses visual evidence and generated pathological descriptions to verify and softly correct concept-bottleneck predictions \citep{zhao2027conchecker}. In single-cell multi-omics, CelLink uses iterative optimal transport to address weak feature linkage and imbalanced cell populations \citep{luo2025cellink}. VOICE combines direct morphology-based and reference-cell retrieval predictions of gene expression using gene-specific scalar weights fitted within each tissue and fixed at evaluation \citep{luo2026voice}. CelLink addresses correspondence and VOICE combines predictions; we fuse paired image--metadata features using sample-conditioned matrix operators.

\paragraph{Uncertainty, calibration and robustness.}
Dropout-based Bayesian approximation, learned observation uncertainty and deep ensembles offer distinct approaches to predictive uncertainty \citep{gal2016dropout,kendall2017uncertainties,lakshminarayanan2017simple}. Our deterministic state descriptors do not implement these estimators or identify an epistemic--aleatoric decomposition. Probability calibration and uncertainty under shift require separate evaluation \citep{guo2017calibration,ovadia2019trust}; a descriptor in $[0,1]$ is not automatically a correctness probability. Detectron uses constrained classifier disagreement to test for potentially harmful covariate shift at the dataset level \citep{ginsberg2023detectron}. Visual-prompt-based test-time adaptation includes optimal-transport alignment in OT-VP \citep{zhang2025otvp} and dynamic prompt coresets for evolving domains in DPCore \citep{zhang2025dpcore}. RiVaT-Fuse conditions its fusion geometry on each input using fixed model parameters at evaluation. We distinguish state-conditioned trust geometry, validation-fitted output calibration and empirical state diagnostics. In language modeling, SCOPE trains selective trust through balanced context preference pairs, reducing susceptibility to misleading signals while preserving performance with useful context \citep{sun2026scope}. Here, trust is implemented through the geometry of image--metadata fusion. Group-based robust learning also depends on the objective and regularization \citep{sagawa2020groupdro}. Our mini-batch state-binned entropic surrogate is stated explicitly, with its distinction from broader distributionally robust objectives discussed in Appendix~\ref{app:theory}.

\paragraph{Related tasks and retinal prediction.}
Multitask learning shares inductive information across related targets \citep{caruana1997multitask}. For retinal structural analysis, ISOSNet uses a shared encoder--decoder for cone detection and layer segmentation, enabling photoreceptor localization and inner/outer segment length measurement from adaptive-optics OCT B-scans \citep{zhou2025isosnet}. We combine DR grading with DME and referable-status prediction, using explicit label masks and a shared representation; adaptive loss-balancing methods are distinguished from our executed coefficients in Appendix~\ref{app:protocol}. Retinal deep learning has demonstrated strong DR detection across large evaluation programs \citep{gulshan2016development,ting2017development}, while RETFound investigates self-supervised retinal foundation representations \citep{zhou2023retfound}. HAPM uses LLM-guided pathological semantics to refine visual prototypes for multicenter DR grading \citep{zhu2025pathology}. These studies provide retinal prediction context rather than numerical benchmarks for our cohort; we study photo-level image--metadata fusion under patient-disjoint evaluation.

\section{Method}
\subsection{Inputs, state and trust geometry}
Each example is a photograph--metadata pair $x=(I,m)$ with partially observed targets $y_t$. Encoders produce $z_I,z_M\in\R^d$ and image tokens $Z_I\in\R^{L\times d}$. We use a ResNet-50 image encoder \citep{he2016deep} and a categorical-embedding/MLP metadata encoder. In the evaluated model, $d=256$ and $L=256$.

The state is $s=(q,u,r)\in[0,1]^3$. Sigmoid MLPs produce $q$ and $u$ from the image embedding, and a third produces $r_{\mathrm{learned}}$ from metadata. With raw metadata completeness $r_{\mathrm{raw}}$,
\begin{equation}
 q=\sigma(f_q(\sg(z_I))),\quad u=\sigma(f_u(\sg(z_I))),\quad
 r=0.8r_{\mathrm{raw}}+0.2\sigma(f_r(\sg(z_M))).
 \label{eq:state}
\end{equation}
The stop-gradient operator applies to encoder features on this training path; descriptor networks remain trainable. The intended roles of $q$ and $u$ are image reliability and uncertainty-related state conditioning. They are learned through the task objective, without direct quality or error-probability supervision. In the method name, \emph{reliability-calibrated} denotes this sample-conditioned adjustment of trust, not a guarantee of probabilistic calibration of $q,u,r$.

For modality $a\in\{I,M\}$, define
\begin{equation}
 P_a(s)=\diag(\operatorname{softplus}(D_as+b_a)+2\epsilon)
 +U_a\diag(\operatorname{softplus}(G_as+c_a)+\epsilon)U_a^\top,
 \label{eq:trust}
\end{equation}
where $U_a\in\R^{d\times k}$, $k=32$ and $\epsilon=10^{-3}$. Thus $P_a\succ0$. The diagonal provides coordinate-specific weighting; the low-rank correction permits coupled latent directions. The complete matrix remains full rank.

\begin{figure}[t]
\centering
\includegraphics[width=\linewidth]{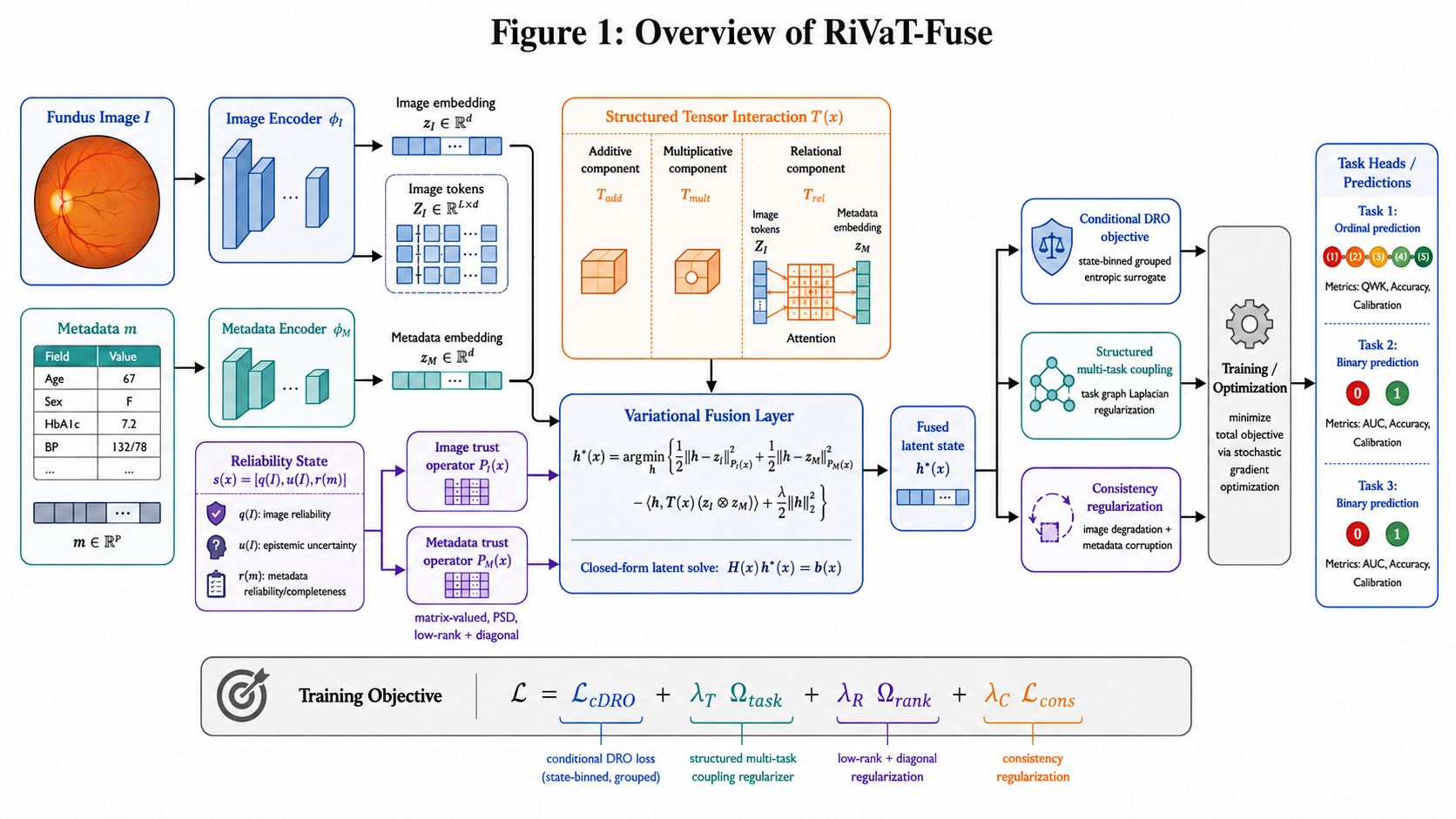}
\caption{Conceptual overview of \method{}. Metadata examples and grade labels are schematic: the experiment uses the 16 fields in Appendix~\ref{app:protocol} and DR grades 0--4. The displayed uncertainty descriptor is a learned state, not a validated epistemic estimator. Equations~\ref{eq:mult}--\ref{eq:solve} specify the executed nonlinear interactions and latent solve; loss blocks and factor regularization apply during training, not as inference stages.}
\label{fig:overview}
\end{figure}

\subsection{Structured interactions and variational estimation}
Let $\Psi=\Psi_{\mathrm{add}}+\Psi_{\mathrm{mult}}+\Psi_{\mathrm{rel}}\in\R^d$. The additive component combines learned affine projections with vector gates conditioned on $s$. The multiplicative component is
\begin{equation}
 \Psi_{\mathrm{mult}}=\sum_{j=1}^{K}\alpha_j(z_I,z_M,s)
 R_j\!\left[\tanh(A_jz_I+a_j)\odot\tanh(B_jz_M+b_j)\right],
 \label{eq:mult}
\end{equation}
where $K=4$, the factors have width 64, $R_j$ includes a learned bias, and $\alpha$ is a softmax gate. The relational component uses a metadata query to attend to image tokens, followed by normalization, an MLP and a state-dependent vector gate. These are structured tensor-style interactions; because of nonlinearities and biases, the complete $\Psi$ is not asserted to be one fixed bilinear map of $z_I\otimes z_M$.

For a given input, the fused representation minimizes
\begin{equation}
 h^\star=\arg\min_h\left\{\tfrac12\|h-z_I\|_{P_I}^2+
 \tfrac12\|h-z_M\|_{P_M}^2-\langle h,\Psi\rangle+
 \tfrac\lambda2\|h\|_2^2\right\},\qquad\lambda=0.1.
 \label{eq:energy}
\end{equation}
Writing $H=P_I+P_M+\lambda I$ and $b=P_Iz_I+P_Mz_M+\Psi$ yields
\begin{equation}Hh^\star=b.\label{eq:solve}\end{equation}
We compute a differentiable linear solve rather than an explicit inverse. The interaction vector can shift the consensus beyond a convex average of embeddings, while $H$ supplies the geometry in which that estimate is determined.

\begin{proposition}[Well-posedness and local sensitivity]
For the trust construction in Eq.~\ref{eq:trust}, the energy in Eq.~\ref{eq:energy} is strongly convex and has the unique solution in Eq.~\ref{eq:solve}. At a differentiable input or state, its differential satisfies
\begin{equation}
 dh^\star=H^{-1}(db-dHh^\star),\qquad
 \|dh^\star\|_2\leq\frac{\|db\|_2+\|dH\|_2\|h^\star\|_2}{\lambda_{\min}(H)}.
 \label{eq:differential}
\end{equation}
\end{proposition}
The Hessian is $H\succeq(\lambda+4\epsilon)I$; differentiating $Hh=b$ proves the local identity. Appendix~\ref{app:theory} gives the full argument, finite perturbation bounds and the distinction between fixed and changing trust operators. This is a statement about the latent system, not an unconditional accuracy guarantee under corrupted inputs.

\subsection{Prediction and training}
For task $t$ and class $c$, a quadratic head gives
\begin{equation}
 f_{tc}(h)=w_{tc}^\top h+b_{tc}+
 \gamma_{tc}\left(\sum_j d_{tcj}h_j^2+\|V_{tc}h\|_2^2\right).
 \label{eq:head}
\end{equation}
The coefficients $d_{tcj}$ and $\gamma_{tc}$ are signed parameters; head curvature need not be positive semidefinite. Task probabilities are softmax outputs. DR is trained as five-class classification on grades 0--4; ordinal performance is evaluated by QWK.

The empirical per-example loss sums observed-label weighted cross-entropies with coefficients $(1,1.2,1)$ for DR, DME and Ref. Unknown targets contribute zero. Clipped inverse-frequency class weights are fitted on training labels. For robust training, define
\begin{equation}
 \rho(s)=0.05+0.2(1-q)+0.3u+0.2(1-r).
\end{equation}
Mini-batch examples are grouped by $((1-q)+u+(1-r))/3$. With group mean radius $\bar\rho_g$ and a bounded learned temperature $\eta_g$, the surrogate is
\begin{equation}
 \mathcal L_{\mathrm{rob}}=\sum_g\frac{|g|}{B}
 \left[\eta_g\bar\rho_g+\eta_g\log\!\left(\frac1{|g|}\sum_{i\in g}e^{\ell_i/\eta_g}\right)\right].
 \label{eq:robust}
\end{equation}
Group membership, group-mean state inputs and radii are detached on this loss path. A scheduled blend of empirical and robust loss is augmented with graph-Laplacian coupling of mean linear head weights, multiplicative-factor regularization, and symmetric-KL consistency under image and metadata perturbations. Appendix~\ref{app:protocol} specifies reductions, coefficients and schedules. No robust optimization or radius correction is performed at inference.

\section{Experiments}
\subsection{Dataset and comparison protocol}
mBRSET contains 5,164 retinal photographs from 1,291 patients with diabetes \citep{mbrset2024}. We retain 4,899 photographs with at least one observed photo-level target and split by patient: 3,422 training photographs from 900 patients, 734 validation photographs from 193 patients, and 743 test photographs from 194 patients. There is no patient overlap. Prediction remains photo-level; repeated photographs are not treated as independent clusters for interval estimation. Ref is the partially observed logical OR of DR$\geq2$ and positive DME, so the tasks are related by construction. Metadata comprises two continuous and fourteen categorical fields; preprocessing uses training data only.

We organize comparisons in two groups. The \textbf{clean reference group} contains Metadata MLP, Image ResNet and four conventional feature-fusion baselines, with deterministic 224-pixel image preprocessing and no augmentation. The two completed unimodal references use patience 20; the four fusion baselines use patience 10. The \textbf{matched supervised group} compares \method{} with HEALNet-R50, MADDi-R50 and TIP-R50 using 512-pixel images, ImageNet-initialized ResNet-50 encoders, shared training splits, class weights, task loss, augmentation, batch sizes and supervised stopping rules. TIP also performs method-specific self-supervised pretraining, so total computation is not matched. Baseline adapters and protocol differences are detailed in Appendix~\ref{app:comparators}.

The main model uses seed 2028, AdamW \citep{loshchilov2019decoupled}, batch size 16, maximum 200 epochs and patience 10. Backbone gradients are frozen for three epochs, with BatchNorm updates retained; the optimizer is not reset at unfreezing. Gradient norm is clipped at 2. The selected checkpoint is epoch 13. Its validation score is the implemented
\begin{equation}
 S=\operatorname{AUROC}_{\mathrm{DME}}+\tfrac12\left[0.3\operatorname{Acc}_{\mathrm{DR}}+0.7\frac{\operatorname{QWK}_{\mathrm{DR}}+1}{2}\right].
\end{equation}
Checkpoint selection thus combines DR and DME validation performance. Binary operating thresholds are chosen on validation data; fixed 0.5 results are also reported. Intervals use 2,000 paired patient-cluster bootstrap replicates and describe held-out patient sampling, not repeated-training variation.

\subsection{Predictive performance}
\begin{table}[!htbp]
\begin{center}\small
\setlength{\tabcolsep}{3pt}
\begin{tabular}{lrrrrrrr}
\toprule
Method & DR Acc. & DR QWK & DR F1 & DME AUC & DME AP & Ref AUC & Ref AP \\
\midrule
\rowcolor{rivatrow} RiVaT-Fuse & \textbf{0.8693} & \textbf{0.8472} & \textbf{0.5737} & 0.9628 & \textbf{0.8182} & 0.9315 & 0.8892 \\
HEALNet-R50 & 0.7978 & 0.7791 & 0.3971 & \textbf{0.9719} & 0.7201 & 0.9258 & 0.8606 \\
MADDi-R50 & 0.8423 & 0.8117 & 0.5118 & 0.9589 & 0.7976 & \textbf{0.9510} & \textbf{0.9025} \\
TIP-R50 & 0.8154 & 0.7198 & 0.4442 & 0.9663 & 0.6371 & 0.9039 & 0.8179 \\
\bottomrule\end{tabular}\end{center}
\caption{Complete external-method adaptations under the matched supervised protocol. TIP additionally uses self-supervised pretraining. Test point estimates; bold is column-best within this group, not a significance marker. DR uses 742 labeled photos, DME 743 and Ref 742.}\label{tab:external}\end{table}

\begin{table}[!htbp]
\begin{center}\small
\setlength{\tabcolsep}{3pt}
\begin{tabular}{lrrrrrrr}
\toprule
Method & DR Acc. & DR QWK & DR F1 & DME AUC & DME AP & Ref AUC & Ref AP \\
\midrule
Metadata MLP & 0.7345 & 0.6036 & 0.3060 & 0.8020 & 0.2115 & 0.8237 & 0.5317 \\
Image ResNet & \textbf{0.8342} & 0.7931 & 0.4748 & 0.9392 & 0.6859 & 0.9222 & 0.8593 \\
Concat Fusion & 0.8275 & 0.7870 & 0.4422 & 0.9387 & \textbf{0.7590} & 0.9244 & 0.8532 \\
FiLM Fusion & 0.8154 & 0.7701 & 0.4353 & 0.9473 & 0.6627 & 0.9171 & 0.7751 \\
Gated Fusion & 0.8100 & \textbf{0.8033} & \textbf{0.5209} & 0.9599 & 0.7320 & \textbf{0.9457} & \textbf{0.8638} \\
Cross-/Co-Attention & 0.8275 & 0.7901 & 0.4752 & \textbf{0.9624} & 0.7398 & 0.9291 & 0.8446 \\
\bottomrule\end{tabular}\end{center}
\caption{Clean conventional references. Image inputs use deterministic preprocessing without augmentation. The two unimodal runs use patience 20; fusion runs use patience 10. This table is contextual and is not an architecture-only comparison with Table~\ref{tab:external}.}\label{tab:clean}\end{table}

\paragraph{DR performance against complete external methods.}
Table~\ref{tab:external} shows consistent DR point-estimate gains over all three adapted external frameworks. Relative to MADDi-R50, the strongest external DR comparator, accuracy increases by 2.70 percentage points, QWK by 0.0355 and macro-F1 by 0.0620. Relative to TIP-R50, the gains are 5.39 points, 0.1274 and 0.1296, respectively. The paired 95\% patient intervals for the TIP accuracy and QWK differences are $[0.0230,0.0876]$ and $[0.0663,0.2010]$. Macro-F1 intervals are wider because rare grades contribute equally. These are unadjusted exploratory intervals; complete intervals appear in Appendix~\ref{app:comparators}.

\paragraph{Auxiliary tasks and conventional references.}
\method{} also attains the highest DME average precision in the matched group (0.8182), including a 0.1811 gain over TIP. The DME AUROC ordering differs, with HEALNet-R50 at 0.9719 versus 0.9628 for \method{}; MADDi-R50 achieves the highest Ref AUROC and AP. Table~\ref{tab:clean} separately describes the conventional references. Image ResNet exceeds the evaluated clean fusion models in DR accuracy, while Gated Fusion has the highest clean-group DR macro-F1. These results motivate testing a structured fusion rule but do not isolate its effect from resolution and training differences across the two groups.

\paragraph{Ablations.}
\begin{table}[!htbp]
\begin{center}\small
\setlength{\tabcolsep}{3pt}
\begin{tabular}{lrrrrr}
\toprule
Component setting & DR Acc. & DR QWK & DR F1 & DME AUC & Ref AUC \\
\midrule
\rowcolor{rivatrow} RiVaT-Fuse & 0.8693 & 0.8472 & 0.5737 & 0.9628 & 0.9315 \\
Diagonal trust & 0.8315 & 0.8170 & 0.5280 & 0.9711 & 0.9521 \\
Scalar trust & 0.8450 & 0.8228 & 0.5319 & 0.9813 & 0.9559 \\
No robust loss & 0.8437 & 0.8502 & 0.5258 & 0.9804 & 0.9508 \\
No multiplicative & 0.8585 & 0.8356 & 0.5383 & 0.9720 & 0.9504 \\
No relational & 0.8464 & 0.8201 & 0.5293 & 0.9561 & 0.9405 \\
Linear heads & 0.8464 & 0.8489 & 0.5254 & 0.9653 & 0.9471 \\
No additive & 0.8423 & 0.8104 & 0.5013 & 0.9699 & 0.9634 \\
No consistency & 0.8369 & 0.8343 & 0.5246 & 0.9666 & 0.9443 \\
No task coupling & 0.8760 & 0.8422 & 0.5510 & 0.9497 & 0.9366 \\
Global learned state & 0.8464 & 0.8507 & 0.5734 & 0.9676 & 0.9564 \\
\bottomrule\end{tabular}\end{center}
\caption{Retrained component ablations, sharing the main data and supervised protocol. Full has the largest DR macro-F1 point estimate among these variants. Additional task metrics are in the appendix.}\label{tab:ablations}\end{table}

The ten component ablations in Table~\ref{tab:ablations} localize the DR benefits. Full trust improves accuracy by 3.77 points over diagonal trust and 2.43 points over scalar trust. Removing additive or relational interactions reduces accuracy and macro-F1. The full model has the largest macro-F1 point estimate among these ten component variants, although its difference from a learned global state is only 0.0003. Coupling and robust training involve metric trade-offs: removing coupling improves DR accuracy slightly, and several variants improve auxiliary AUROC. Additional component metrics are in Appendix~\ref{app:comparators}. The evidence therefore supports a useful complete configuration rather than universal necessity of every component.

\subsection{Mechanism, reliability and interpretation}
\begin{figure}[t]
\centering\includegraphics[width=\linewidth]{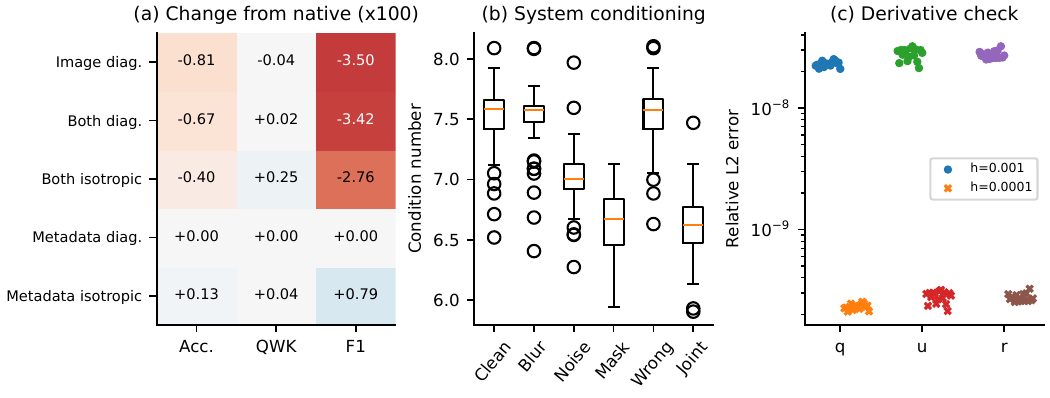}
\caption{Frozen-model mechanism diagnostics. Left: clean test DR performance under trace-preserving trust substitutions; these are interventions, not retrained ablations. Middle: H condition numbers for the same 64 test patients across six conditions. Right: central finite differences agree with the differentiated solve for q, u and r (eight patients per split; two step sizes).}
\label{fig:mechanism}
\end{figure}
\paragraph{Geometry and a verified solve.}
The test clean system has median condition number 7.58 on the 64-patient mechanism sample. Across the evaluated full-cohort condition branches, the largest recorded relative solve residual is $2.00\times10^{-7}$. Whitened image, metadata and regularization operators sum to the identity to within $9.55\times10^{-15}$. Local state derivatives agree with Eq.~\ref{eq:differential}; the largest finite-difference relative error is $3.23\times10^{-8}$. Replacing image trust by its diagonal decreases clean DR macro-F1 from 0.5737 to 0.5387; metadata-only replacements have much smaller, mixed effects. This connects the matrix construction to measurable decision changes without equating every nonzero direction with predictive benefit.

\paragraph{Controlled uncertainty.}
We evaluate 34 conditions per split, including blur, three-seed noise, metadata masking/replacement and a joint condition. The same corrupted representations are evaluated with adaptive states and with selected coordinates fixed to each photograph's clean state. At noise SD 0.10, adaptive DR QWK averages 0.3626 versus 0.3341 with all states fixed, whereas DME AUROC is lower (0.6558 versus 0.7044). Clean-state fixing uses privileged reference information and is a diagnostic, not a deployment baseline. State semantics are similarly task-dependent: q declines under noise, but u also declines; the latter is not a validated general epistemic-uncertainty estimator. Appendix~\ref{app:analysis} reports all conditions and state coordinates.

\begin{figure}[t]
\centering\includegraphics[width=\linewidth]{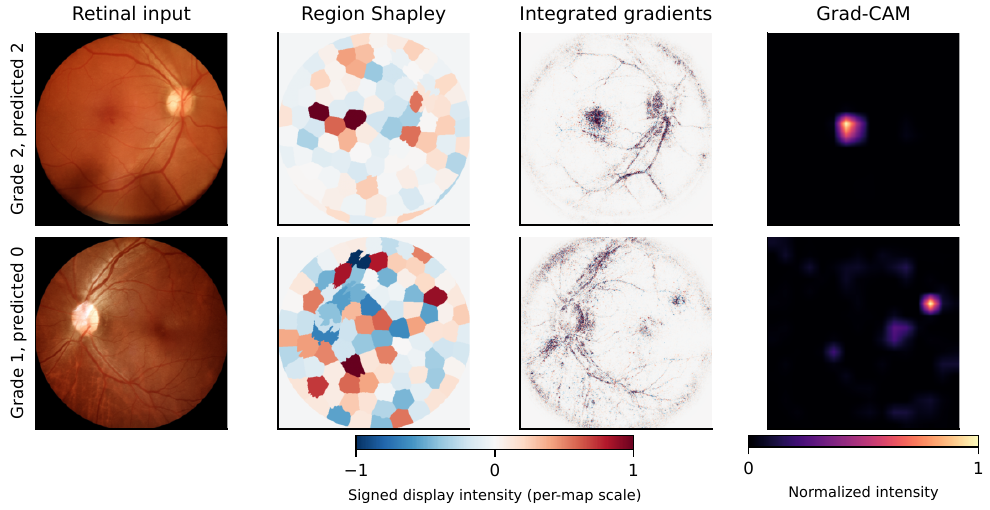}
\caption{DR expected-grade explanations from the fixed deep-case panel: one grade-2 correct prediction and one grade-1 error. Region Shapley and integrated gradients use the blurred-image reference; signed colors indicate increases or decreases in expected grade. Grad-CAM shows nonnegative layer-4 activation weighting. Signed maps are scaled by their own 99th-percentile absolute magnitude and Grad-CAM by its maximum, so colors compare spatial patterns rather than absolute effect sizes. All 32 cases and numerical checks are provided in Appendix~\ref{app:caseatlas}.}
\label{fig:attribution}
\end{figure}
\paragraph{Attribution at two levels.}
We separate algebraic latent/head decomposition from input attributions. Metadata permutation attributions cover every test photograph under two disjoint training-patient backgrounds. DR global absolute-importance ranks have Spearman correlation 0.891 across backgrounds; diabetes duration, age, insulin, education and sex form the leading five fields. Image analysis covers 32 distinct patients using integrated gradients, region Shapley estimates and Grad-CAM \citep{sundararajan2017axiomatic,lundberg2017unified,selvaraju2017gradcam}; Figure~\ref{fig:attribution} illustrates distinct spatial explanations of the same output. The 24 output combinations in the eight deep cases have blur-reference, 256-node IG relative completeness residuals below 2.6\%. Broader lower-budget cases have larger residuals and are marked individually. These analyses inspect the predictor; they do not establish causal risk factors or validated lesion localization.

\section{Conclusion and Limitations}
\method{} formulates multimodal fusion as a well-posed latent estimation problem with direction-dependent trust and structured interactions. It achieves the strongest DR point estimates among the evaluated complete external adaptations and offers an explicitly testable link between state, geometry and predictions. The study is limited to one dataset and mainly one matched training seed; protocol differences, rare-grade uncertainty and mixed auxiliary/perturbation outcomes constrain broader claims. The complete results and precision diagnostics support future replication and refinement without changing the scope of the demonstrated DR gains.

\label{main:end}
\section*{AI Use Statement}
In this work, we used generative AI tools for aiding or polishing writing and drafting sections of the paper. We have not used generative AI tools for generating synthetic datasets or proving mathematical claims, and the remaining required disclosure tasks are not applicable to this work. Additionally, we used generative AI tools for research ideation or execution and retrieval and discovery. We have reviewed all AI-assisted work. AI-assisted code, analyses, retrieved literature, and manuscript text were manually checked and verified by the authors, including through executable pipelines, independent metric recomputation, and source verification. We take responsibility for the final content of this work, including text, claims or artifacts produced with the aid of generative AI.

\section*{Ethics Statement}
This study analyzes an existing retinal-image dataset with clinical and demographic metadata. Dataset access is governed by the source repository's credentialing and data-use requirements \citep{mbrset2024,pollard2026physionet}. The work does not recruit participants or prospectively evaluate clinical decisions. Patient-disjoint evaluation reduces leakage from repeated photographs, while patient-level resampling acknowledges dependence within subjects. Rare-grade performance, subgroup uncertainty and the relation between DR, DME and the derived referable target limit clinical interpretation. The models and visual explanations are research tools; no claim of autonomous clinical readiness is made. Redistribution of source data remains subject to the original data-use terms.

\section*{Reproducibility Statement}
Appendices provide the fusion equations, executed optimization and preprocessing rules, comparator adaptations, partition-specific metrics, checkpoint identity and analysis definitions. Comparisons use saved labels, predictions and immutable configurations. The manuscript source package contains editable TeX sources, generated tables and figures; full figure regeneration additionally requires the retained analysis code and experimental arrays. Access to source images is through mBRSET. Frozen experimental bundles retain trained weights and analysis arrays; their publication requires an anonymized, license-compatible release.

\bibliography{references}
\bibliographystyle{iclr2027_conference}
\clearpage
\appendix
\section{Variational Formulation and Mathematical Details}\label{app:theory}
\subsection{Notation and scope}
\begin{center}\small
\begin{tabular}{p{.19\linewidth}p{.73\linewidth}}\toprule
Symbol & Meaning \\\midrule
$I,m$ & One retinal photograph and its paired structured record.\\
$z_I,z_M$ & Modality embeddings in $\R^{256}$.\\
$Z_I$ & A $256\times256$ array of spatial image tokens.\\
$s=(q,u,r)$ & Learned image-derived state coordinates and completeness-anchored metadata state.\\
$P_I,P_M$ & Symmetric positive-definite trust matrices, each with a diagonal and rank-32 update.\\
$\Psi$ & Sum of additive, nonlinear factorized multiplicative and relational interaction vectors.\\
$H,b$ & Hessian and right-hand side of the latent quadratic problem.\\
$h$ & Unique latent minimizer, found by a linear solve.\\
$\rho,\eta_g$ & State-conditioned robust-training radius and group temperature.\\
$Q_{tc}$ & Symmetric, potentially indefinite quadratic form for one output class.\\
\bottomrule\end{tabular}
\end{center}
The variational terminology refers to optimization over the latent vector for each input. It does not denote a variational autoencoder, latent sampling, or an evidence lower bound. The model is deterministic at evaluation. A matrix-valued trust operator specifies a geometry of disagreement; it is not identified with an estimated noise precision or a clinically measured uncertainty covariance.

\subsection{Expanding the energy}
For fixed encoder outputs, state and parameters, write
\begin{align}
 E(h)&=\tfrac12(h-z_I)^\top P_I(h-z_I)+\tfrac12(h-z_M)^\top P_M(h-z_M)
 -h^\top\Psi+\tfrac\lambda2h^\top h\\
 &=\tfrac12 h^\top Hh-h^\top b+C,
\end{align}
where $C=\tfrac12z_I^\top P_Iz_I+\tfrac12z_M^\top P_Mz_M$ is independent of $h$.
Consequently, $\nabla_h E=Hh-b$ and $\nabla_h^2E=H$. The state and interaction networks depend on the input but do not depend on the optimization variable $h$. This last condition is essential: if either were recomputed as a nonlinear function of the candidate $h$, the present quadratic argument would no longer describe that different model.

For any nonzero $v$, Eq.~\ref{eq:trust} gives
\begin{align}
v^\top P_av&=\sum_j\left[\operatorname{softplus}((D_as+b_a)_j)+2\epsilon\right]v_j^2\\
&\quad+\sum_l\left[\operatorname{softplus}((G_as+c_a)_l)+\epsilon\right](U_a^\top v)_l^2
\geq2\epsilon\|v\|_2^2.
\end{align}
Thus $H\succeq(\lambda+4\epsilon)I$, and the energy is coercive and strongly convex. Solving its first-order condition produces the unique global minimizer $h=H^{-1}b$. For the executed $\lambda=0.1$ and $\epsilon=0.001$, this construction supplies a conservative lower bound of 0.104. Observed eigenvalues can be appreciably larger. This bound alone does not bound the condition number from above, since the largest eigenvalue depends on learned parameters.

\subsection{Finite perturbations and fixed-operator special case}
Let a second input or intervention produce $H'=H+\Delta H$, $b'=b+\Delta b$, and $h'=h+\Delta h$. Subtracting the two linear systems yields the exact relation
\begin{equation}
(H+\Delta H)\Delta h=\Delta b-\Delta Hh.
\end{equation}
If $H'$ is positive definite with minimum eigenvalue $\mu'>0$, then
\begin{equation}
\|\Delta h\|_2\leq\frac{\|\Delta b\|_2+\|\Delta H\|_2\|h\|_2}{\mu'}.
\label{eq:finite_bound}
\end{equation}
Alternatively, if $\|\Delta H\|_2<\mu=\lambda_{\min}(H)$, Weyl's bound gives $\mu'\geq\mu-\|\Delta H\|_2$. This makes the dependence on both geometric change and evidence change explicit. It would be incorrect to discard $\Delta H$ when the trust matrices adapt to a perturbed state.

When the trust operators are held fixed, $\Delta H=0$ and
\begin{equation}
\Delta h=H^{-1}\left(P_I\Delta z_I+P_M\Delta z_M+\Delta\Psi\right).
\end{equation}
Hence
\begin{equation}
\|\Delta h\|_2\leq\|H^{-1}\|_2\left(\|P_I\|_2\|\Delta z_I\|_2+
\|P_M\|_2\|\Delta z_M\|_2+\|\Delta\Psi\|_2\right).
\end{equation}
Even this special case permits interaction changes: holding $s$ fixed does not hold image tokens or embeddings fixed. The clean-state intervention experiments therefore distinguish state adaptation from changed encoded evidence; they are not examples of perturbing only a scalar confidence value.

\subsection{Differentiating the implemented solve}
Differentiating $Hh=b$ gives $dH\,h+H\,dh=db$ and thus
\begin{equation}
dh=\underbrace{H^{-1}db}_{\text{right-hand-side path}}
\;\underbrace{-H^{-1}dH\,h}_{\text{system-matrix path}}.
\end{equation}
The two terms can align or cancel. Their norms are not additive importance fractions. With one state coordinate $s_j$, the numerical analysis evaluates these paths while holding $z_I,z_M,Z_I$ fixed. It computes FP64 Jacobian--vector products, reconstructs the two paths by linear solves, and compares their sum with central finite differences at steps $10^{-3}$ and $10^{-4}$. The 96 recorded checks cover eight selected patients per split, three coordinates and two steps. Their maximum relative finite-difference discrepancy is $3.23\times10^{-8}$; the derivative-identity discrepancy is below $1.34\times10^{-15}$.

These checks concern differentiability of the frozen evaluation function. They do not estimate the gradients encountered during historical training. ReLU components are only piecewise differentiable; a finite-difference interval crossing a kink would require a one-sided or piecewise interpretation. No such general smoothness guarantee is inferred from agreement at the sampled states.

For a scalar downstream loss $L(h)$, implicit differentiation also gives a useful computational identity. If $v$ solves $H^\top v=\nabla_h L$, then
\begin{equation}dL=v^\top db-v^\top(dH)h.\end{equation}
The implemented automatic differentiation through the linear solve follows this dependence. No inverse matrix needs to be explicitly materialized for the prediction itself.

\subsection{Prediction sensitivity and margins}
The implemented head can be written as $f_c(h)=w_c^\top h+b_c+h^\top Q_ch$, with
\begin{equation}
Q_c=\gamma_c\left[\diag(d_c)+V_c^\top V_c\right].
\end{equation}
Because $d_c$ and $\gamma_c$ are unconstrained signed parameters, $Q_c$ is not necessarily positive semidefinite. On a ball $\|h\|_2\leq R$, its gradient satisfies
\begin{equation}\|\nabla f_c(h)\|_2\leq\|w_c\|_2+2\|Q_c\|_2R.\end{equation}
Thus a local bounded-domain Lipschitz statement is available. A global, parameter-independent Lipschitz guarantee for these quadratic heads is not.

For an exact increment $\delta$, the class-logit change is
\begin{equation}
f_c(h+\delta)-f_c(h)=w_c^\top\delta+2h^\top Q_c\delta+\delta^\top Q_c\delta.
\end{equation}
If the current predicted class $a$ has margin $m=f_a(h)-\max_{c\ne a}f_c(h)>0$, and all logit changes are bounded in absolute value by $e<m/2$, then its argmax remains unchanged. This is a conditional margin statement. The experiments measure actual changes rather than asserting that this sufficient condition holds for every corrupted sample. Softmax can suppress a substantial logit change in a saturated output or amplify its decision relevance near a boundary; latent movement alone is therefore not a performance metric.

\subsection{Whitened trust geometry}
Raw trust traces summarize scale but omit how each operator acts within the combined system. With the symmetric inverse square root of $H$, define
\begin{equation}
W_I=H^{-1/2}P_IH^{-1/2},\quad W_M=H^{-1/2}P_MH^{-1/2},\quad
W_0=\lambda H^{-1}.
\end{equation}
Then $W_I+W_M+W_0=I$. All three are positive definite, and each has eigenvalues strictly between zero and one. Their traces sum to $d$, which supplies an internal scale check. The matrices need not share eigenvectors; adding their individually sorted eigenvalues is not a valid directional decomposition. The analysis verifies the full operator closure and only then summarizes spectra and traces.

Effective rank is computed as $\exp[-\sum_j p_j\log p_j]$, where $p_j$ is an eigenvalue divided by the eigenvalue sum. A rank-32 correction to a positive diagonal matrix can produce effective rank near 256. Therefore high effective rank here is consistent with the parameterization. The inverse-square-root calculation is an analysis operation; inference uses the original solve.

\subsection{Exact latent and head decomposition}
At a frozen input, define five right-hand-side components $b_I=P_Iz_I$, $b_M=P_Mz_M$, $b_A=\Psi_{\mathrm{add}}$, $b_U=\Psi_{\mathrm{mult}}$, and $b_R=\Psi_{\mathrm{rel}}$. Setting $h_j=H^{-1}b_j$ gives the exact identity $h=\sum_jh_j$. For a quadratic class head,
\begin{equation}
f_c(h)=b_c+\sum_j w_c^\top h_j+\sum_j h_j^\top Q_ch_j
+2\sum_{j<k}h_j^\top Q_ch_k.
\end{equation}
The cross terms can be negative. This decomposition answers how a fixed prediction is algebraically assembled. It is not equivalent to retraining without a component, replacing an input field by a donor value, or removing a term while recomputing $H$. The appendix visualization atlas labels these distinct analyses accordingly.

\subsection{Computational scope}
Each trust operator uses a diagonal gate and a $d\times k$ factor rather than an unconstrained parameter for every matrix entry. The current implementation nevertheless assembles a dense $d\times d$ system and calls a dense differentiable solve. Its asymptotic factorization cost is cubic in $d$; a low-rank parameterization alone does not make the executed solver linear time. A Woodbury-based implementation could exploit the diagonal-plus-low-rank structure, but is not the implementation evaluated here. Runtime records are reported as measured experiment durations, not a controlled speed comparison against external methods.

\subsection{Relation to distributionally robust objectives}
Distributionally robust optimization can induce variance-based regularization under particular uncertainty sets \citep{namkoong2017variance}; the effect on classification depends on the loss and admissible distributional changes \citep{hu2018does}. These results motivate stating the executed surrogate, rather than inferring robustness from its name. Our groups are mini-batch state bins, not externally defined populations; their memberships, mean state inputs and radii are detached on this path. Consequently, the recorded experiment is neither a solution of every possible conditional adversarial distribution nor a certificate for arbitrary input corruption. Controlled perturbation tests assess its empirical response separately from the exact latent-system identities proved above.

\clearpage
\section{Data, Targets and Executed Training Protocol}\label{app:protocol}
\subsection{Prediction unit and partition unit}
The input case is one image and its associated metadata. Both eyes and repeated views from a patient can differ in appearance and labels while sharing systemic context. Photo-level targets therefore preserve information that patient-maximum targets would collapse. However, shared patient factors create dependence; putting different photographs of one patient into training and test partitions can create a leakage pathway. We use patient-disjoint partitions while retaining photo-level predictions and labels.

The original source contains 5,164 images from 1,291 patients. Eligibility requires at least one observed photo-level target, leaving 4,899 images and 1,287 patients across the three partitions. The remaining 265 photographs are excluded from disease-metric training/evaluation eligibility. They are not assigned artificial negative labels. Natural-quality diagnostics in the earlier analysis can have a different cohort from disease scoring; their individual figure captions specify the relevant denominator.

\begin{table}[!htbp]
\begin{center}\small
\setlength{\tabcolsep}{3pt}
\begin{tabular}{lrrrrr}
\toprule
Partition & Eligible photos & Patients & DR labels & DME labels & Ref labels \\
\midrule
Train & 3422 & 900 & 3410 & 3422 & 3413 \\
Validation & 734 & 193 & 732 & 734 & 733 \\
Test & 743 & 194 & 742 & 743 & 742 \\
\bottomrule\end{tabular}\end{center}
\caption{Frozen photo-target cohort with patient-disjoint partitions. Counts differ across tasks because unknown labels remain masked.}\label{tab:cohort}\end{table}

\begin{table}[!htbp]
\begin{center}\small
\setlength{\tabcolsep}{3pt}
\begin{tabular}{lrrrr}
\toprule
Task & Label & Train & Validation & Test \\
\midrule
DR & 0 & 2618 & 571 & 561 \\
DR & 1 & 185 & 45 & 42 \\
DR & 2 & 414 & 74 & 80 \\
DR & 3 & 67 & 9 & 6 \\
DR & 4 & 126 & 33 & 53 \\
DR & Unknown & 12 & 2 & 1 \\
DME & 0 & 3092 & 691 & 689 \\
DME & 1 & 330 & 43 & 54 \\
DME & Unknown & 0 & 0 & 0 \\
REF & 0 & 2776 & 611 & 602 \\
REF & 1 & 637 & 122 & 140 \\
REF & Unknown & 9 & 1 & 1 \\
\bottomrule\end{tabular}\end{center}
\caption{Observed label counts and task-specific unknown values.}\label{tab:labels}\end{table}

The patient split uses seed 2025 and stratifies by whether a patient has any observed referable-positive photograph. No patient appears in two partitions. The protocol was frozen before the reported retraining.

\subsection{Targets and missingness}
DR retains observed integer grades 0--4. DME retains observed binary labels. For Ref, a positive observed DME or an observed DR grade at least two establishes a positive target. A negative target requires enough observed information to establish that both terms of the OR are false. Otherwise the target remains unknown. This partial-observation logic avoids treating an unobserved disease component as negative. Each task has its own valid-label mask in both loss and metrics.

Ref is a derived target, not a separately adjudicated independent disease. It encourages the model to summarize a clinically motivated combination of DR and DME evidence, but three-task evaluation must not be described as three independent external validations. In the test partition, there are 742 valid DR labels, 743 DME labels and 742 Ref labels. The six grade-3 test photographs make that class's estimates especially uncertain.

\subsection{Metadata and train-fitted preprocessing}
Continuous inputs are age and diabetes duration. Categorical inputs are sex, educational level, systemic hypertension, insurance, insulin, oral diabetes treatment, obesity, smoking, alcohol consumption, vascular disease, acute myocardial infarction, nephropathy, neuropathy and diabetic foot. No disease target, patient identifier or image filename is an input feature.

Continuous imputation/scaling statistics and category vocabularies are fitted using training rows. Missing/unknown categories use the frozen encoding policy. Raw completeness is calculated before imputation and is retained as an explicit state input. The metadata encoder embeds categorical fields and joins them with continuous values before an MLP. Repeated photos can contain the same patient metadata; attribution summaries report both photo weighting and patient-balanced weighting to expose this distinction.

\subsection{Main architecture and hyperparameters}
\begin{table}[!htbp]
\begin{center}\small
\setlength{\tabcolsep}{3pt}
\begin{tabular}{ll}
\toprule
Setting & Executed value \\
\midrule
Image input & 512 by 512 RGB \\
Backbone & ResNet-50, ImageNet V2 \\
Latent width / image tokens & 256 / 256 \\
Metadata & 2 continuous + 14 categorical \\
Categorical embedding width & 8 \\
Trust rank / epsilon & 32 / 0.001 \\
Multiplicative factors & 4 components, each width 64 \\
Relational attention & 4 heads \\
Quadratic-head rank & 8 \\
Main dropout & 0.2 \\
Variational lambda & 0.1 \\
Train / evaluation batch & 16 / 32 \\
Seed / split seed & 2028 / 2025 \\
Maximum epochs / patience & 200 / 10 \\
Selected / final epoch & 13 / 23 \\
Backbone / other learning rate & $10^{-4}$ / $10^{-3}$ \\
AdamW weight decay & $10^{-4}$ \\
Frozen-gradient epochs / clip & 3 / 2.0 \\
\bottomrule\end{tabular}\end{center}
\caption{Main-model architecture and optimization settings.}\label{tab:hyper}\end{table}

The image encoder is initialized with the torchvision ResNet-50 ImageNet V2 weights. A failed pretrained-weight load is an error rather than a silent random initialization. Global pooled features and projected spatial features produce the image vector and tokens, respectively. The final convolutional map is $16\times16$ at 512-pixel input resolution, yielding 256 image tokens. The metadata branch produces the same latent dimension. Dropout is disabled for all frozen-checkpoint analyses.

The state descriptor receives detached encoder features during training, but its own MLP parameters receive gradients through the fusion task loss. The inference values are unchanged by this detachment. In post-hoc input attribution, the analysis explicitly differentiates the evaluation function through descriptor dependencies where that is the chosen explanation target; this must not be confused with reconstructing historical optimization gradients.

\subsection{Weighted loss and robust surrogate}
For each task, let $n_c$ be observed training class counts, replacing any zero count by one. Inverse class proportions are normalized by their mean across classes, then clipped to $[0.2,10]$, with no renormalization after clipping. The resulting weights are
\begin{align*}
w_{\mathrm{DR}}&=(0.2000,0.8700,0.3888,2.4023,1.2774),\\
w_{\mathrm{DME}}&=(0.2000,1.8071),\qquad
w_{\mathrm{Ref}}=(0.3733,1.6267).
\end{align*}
The empirical reduction is
\begin{equation}
\mathcal L_{\mathrm{emp}}=\frac1B\sum_{i=1}^B\sum_t a_t M_{it}
w_{t,y_{it}}\big[-\log p_t(y_{it}\mid x_i)\big],\quad(a_{\mathrm{DR}},a_{\mathrm{DME}},a_{\mathrm{Ref}})=(1,1.2,1).
\end{equation}
$B$ includes every batch row. This is not the same reduction as independently averaging each task by the sum of its observed class weights.

\begin{center}\small
\setlength{\tabcolsep}{3pt}
\begin{tabular}{ll}
\toprule
Robust-training setting & Value \\
\midrule
State bins & 4 \\
Temperature bounds & 0.05 to 2.0 \\
Radius coefficients & 0.05, 0.20, 0.30, 0.20 \\
Task Laplacian & $3I-\mathbf{1}\mathbf{1}^{\top}$ \\
\bottomrule\end{tabular}\end{center}

State risk bins are formed within each mini-batch. Empty bins contribute nothing. Nonempty groups use the learned bounded temperature and detached mean radius in Eq.~\ref{eq:robust}, weighted by their fraction of the batch. The objective approximates a conditional entropic robust construction using coarse groups; it neither solves an independent worst-case distribution for every sample nor estimates a clinical uncertainty distribution.

The final training objective at epoch $e$ is
\begin{equation}
\mathcal L_e=(1-\beta_e)\mathcal L_{\mathrm{emp}}+\beta_e\mathcal L_{\mathrm{rob}}
+2\!\times\!10^{-4}\mathcal R_{\mathrm{task}}+5\!\times\!10^{-5}\mathcal R_{\mathrm{factor}}
+10^{-3}c_e\mathcal L_{\mathrm{cons}}.
\end{equation}
The robust blend has eight warmup epochs followed by a 20-epoch ramp to 0.35. Consistency has five warmup epochs followed by a 15-epoch ramp to one. Inactive zero-weight branches are skipped in this main implementation. Consequently, auxiliary-forward counts and BatchNorm/random-number trajectories during warmup are part of the recipe, not merely a runtime detail.

Task coupling uses the graph Laplacian and the class-averaged linear coefficient of each task head, normalized by the number of entries in the stacked weight matrix. It does not couple every quadratic coefficient. Factor regularization averages squared weights of the three projections in each multiplicative component; it is not a nuclear norm or an adaptive rank-selection procedure. Consistency uses symmetric KL divergences between clean and separately perturbed image/metadata predictions. The perturbations used by that training loss are distinct from the controlled evaluation conditions.

\subsection{Optimization, stopping and thresholds}
AdamW uses learning rates $10^{-4}$ for backbone parameters and $10^{-3}$ for other parameters, with weight decay $10^{-4}$. All backbone parameters are present in the optimizer throughout. Their gradients are disabled for epochs 1--3 and enabled from epoch 4; optimizer state for already-trained parameters is preserved. The model remains in training mode, so backbone BatchNorm running statistics are updated during frozen-gradient epochs. Global gradient norm is clipped at 2.0. There is no main-run learning-rate scheduler.

Training images are resized to $512\times512$, followed by horizontal flip with probability 0.5, a random affine transform with probability 0.6 (up to 10 degrees, translation 0.03 and scale 0.95--1.05), and ColorJitter with brightness/contrast/saturation/hue magnitudes 0.12/0.12/0.08/0.02. Tensor conversion and ImageNet normalization follow. Evaluation uses deterministic resize and normalization. The batch sizes are 16 for training and 32 for evaluation; incomplete training batches are dropped.

The actual selected run stops after epoch 23 with patience 10 and best epoch 13. A new best requires an improvement larger than $10^{-6}$ in the recorded composite validation score. Neither test labels nor test thresholds participate in the epoch-selection calculation.

Each binary task selects a validation Youden threshold on a 1,001-point grid. For the main model these are 0.568 for DME and 0.009 for Ref. Fixed 0.5 results are retained alongside the selected operating points. No threshold is retuned by test subgroup or corruption severity. AUROC and AP are threshold-free, whereas sensitivity, specificity and accuracy are operating-point dependent.

\subsection{Metrics and resampling}
DR accuracy and macro-F1 use the argmax class, and QWK retains the ordered label set 0--4 even when a subgroup lacks some grades. Macro-F1 includes all five classes with zero division handled as zero. Binary AUROC and AP use positive-class probabilities. Binary Brier is the mean squared error of the positive probability. DR ECE uses top-label confidence/correctness; binary ECE uses positive probability/label, with ten equal-width bins. These ECE definitions are not interchangeable.

For paired uncertainty, sample test patients with replacement, include all their eligible photographs with the corresponding multiplicity, and recompute both models' metrics on the identical resample. Differences are formed within each replicate, and percentile 95\% intervals are reported. Unknown labels remain masked per task. The 2,000 replicates quantify sampling variation across observed patients. Comparisons are not adjusted for multiplicity.

\subsection{Relation to adaptive multitask weighting}
Uncertainty-based weighting learns task-loss scales \citep{kendall2018multitask}, whereas GradNorm adapts task weights using gradient magnitudes and relative learning rates \citep{chen2018gradnorm}. Our recorded task coefficients $(1,1.2,1)$ are fixed; class-frequency weights and missing-label masks have separate roles. Neither the learned state coordinates nor the task-head Laplacian penalty should be interpreted as an implementation of those adaptive weighting algorithms. The DR-centered evaluation emphasis therefore does not change the executed loss or the joint validation checkpoint rule.

\clearpage
\section{Comparator Adaptations and Full Comparisons}\label{app:comparators}
\subsection{Two comparison groups}
The clean reference group preserves documented conventional baseline architectures and deterministic image preprocessing. It contains only the selected clean Metadata MLP, clean Image ResNet and four clean feature-fusion runs. Earlier augmented unimodal development runs are not substituted into this group. The matched512 group instead pairs the full model with declared HEALNet, MADDi and TIP adaptations under the shared supervised recipe. This grouping prevents different resolutions and optimization recipes from being presented as one controlled architecture ranking.

All groups share the patient-disjoint partitions, photo targets, missing-label masks, train-fitted metadata preprocessing, class-weight construction and recorded joint validation score. The clean image/fusion runs use 224 pixels, batches 64/256, no augmentation, no backbone freezing and no gradient clipping. Historical clean unimodal stopping uses patience 20; clean fusion uses 10. The main and external matched group use 512 pixels, batches 16/32, augmentation, three frozen-gradient epochs, clipping at 2 and supervised patience 10. These differences remain relevant even when a model has a lower test score.

\subsection{HEALNet-R50}
The adaptation uses the published heterogeneous shared-latent fusion architecture \citep{hemker2024healnet}, with ResNet spatial features and encoded metadata as the two available modalities. The output has nine logits, partitioned into five DR and two logits for each binary task. Upstream constructor defaults provide the fusion settings; this is not a reimplementation of TCGA-specific tuned survival settings. A code patch removes silent exception swallowing and handles missing modalities consistently with the upstream verbose path. Complete-input core behavior is preserved. The upstream source commit is \texttt{90459b65a3d4a4ef9fd405671c457b5a1163cc7d}.

\subsection{MADDi-R50}
The adaptation retains the published self-attention followed by bidirectional cross-attention topology \citep{golovanevsky2022maddi}, using two modalities rather than inventing a genetic input. The clinical MLP has widths 200/100/50 and dropout 0.5/0.3/0.2; the ResNet pooled representation is projected to 50 dimensions. The two cross outputs and original modality vectors form a 200-dimensional fused feature, followed by nine task logits. Keras-style attention initialization and BatchNorm behavior are explicitly ported.

The official whole-vector input is one token per modality. With one key, its attention softmax equals one, so these branches act as learned value/output transforms; query/key gradients are zero. This is an inherited property of the chosen published path, not a patch-level attention implementation. Its original imaging architecture and disease task differ from this ResNet50 two-modality adaptation. The pinned upstream commit is \texttt{9d12bdb6ea9e9ec8cc1743956e1d4171c5d00fda}.

\subsection{TIP-R50}
TIP retains four-layer, width-512 tabular and multimodal Transformers with eight heads, field/CLS embeddings and the contrastive, image--tabular matching and masked-reconstruction objectives \citep{du2024tip}. Half of the 16 fields are specially masked during self-supervised pretraining. Training patients alone update parameters; fixed-seed validation loss monitors stopping, and test patients are excluded. Original within-batch photograph negatives are retained, including the possibility that another photograph from the same patient is a negative.

The additional SSL stage uses Adam, learning rate $3\times10^{-4}$, weight decay $1.5\times10^{-6}$, ten warmup epochs and cosine decay with a maximum of 500 epochs. Patience 10 after warmup is an explicit adaptation of the original fixed-budget protocol. All encoders train from the start of SSL without clipping. The selected SSL epoch is 5, with stopping at epoch 20; the selected encoder then initializes supervised training with a fresh optimizer. The supervised best/final epochs are 14/24. Thus 500 is a configured maximum, not the number of completed pretraining epochs. The pinned upstream commit is \texttt{87c19085843314262ccc077a86ebb3368d903727}.

\subsection{Ablation implementation}
A01 projects the constructed trust matrices to their diagonal; A02 replaces each by its trace-normalized scalar identity. These preserve common parameter construction rather than replacing the entire descriptor architecture. A03 removes the robust blend. A04/A05/A07 zero the multiplicative/relational/additive branch output while retaining the scheduled computation used to control random streams. A06 uses linear task heads. A08 zeros the consistency penalty while preserving its scheduled auxiliary-forward execution; A09 zeros task coupling. A10 learns a single global three-coordinate state rather than using sample-specific descriptors.

The displayed outcomes use the frozen experimental records. Their differences quantify the specified intervention under the chosen optimizer and stopping rule, not a guarantee that the same ordering holds after separate hyperparameter optimization or across repeated seeds.

\begin{table}[!htbp]
\begin{center}\small
\setlength{\tabcolsep}{3pt}
\begin{tabular}{lrrrrrrr}
\toprule
Method & DR Acc. & DR QWK & DR F1 & DME AUC & DME AP & Ref AUC & Ref AP \\
\midrule
RiVaT-Fuse & 0.8693 & 0.8472 & 0.5737 & 0.9628 & 0.8182 & 0.9315 & 0.8892 \\
Diagonal trust & 0.8315 & 0.8170 & 0.5280 & 0.9711 & 0.8450 & 0.9521 & 0.8995 \\
Scalar trust & 0.8450 & 0.8228 & 0.5319 & 0.9813 & 0.8116 & 0.9559 & 0.9164 \\
No robust loss & 0.8437 & 0.8502 & 0.5258 & 0.9804 & 0.8502 & 0.9508 & 0.9128 \\
No multiplicative & 0.8585 & 0.8356 & 0.5383 & 0.9720 & 0.8356 & 0.9504 & 0.9002 \\
No relational & 0.8464 & 0.8201 & 0.5293 & 0.9561 & 0.8176 & 0.9405 & 0.8968 \\
Linear heads & 0.8464 & 0.8489 & 0.5254 & 0.9653 & 0.8400 & 0.9471 & 0.8974 \\
No additive & 0.8423 & 0.8104 & 0.5013 & 0.9699 & 0.8348 & 0.9634 & 0.9168 \\
No consistency & 0.8369 & 0.8343 & 0.5246 & 0.9666 & 0.7879 & 0.9443 & 0.9020 \\
No task coupling & 0.8760 & 0.8422 & 0.5510 & 0.9497 & 0.7856 & 0.9366 & 0.8930 \\
Global learned state & 0.8464 & 0.8507 & 0.5734 & 0.9676 & 0.7826 & 0.9564 & 0.9091 \\
\bottomrule\end{tabular}\end{center}
\caption{Component comparisons. All values use the same font weight.}\label{tab:allablations}\end{table}
\clearpage\subsection{Paired external comparisons}
\begin{table}[!htbp]
\begin{center}\footnotesize
\setlength{\tabcolsep}{3pt}
\begin{tabular}{lrrr}
\toprule
Comparator & Metric & Full minus comparator & 95\% interval \\
\midrule
HEALNet-R50 & dr\_accuracy & 0.0714 & [0.0402, 0.1060] \\
HEALNet-R50 & dr\_qwk & 0.0681 & [0.0070, 0.1248] \\
HEALNet-R50 & dr\_macro\_f1 & 0.1766 & [0.0641, 0.2844] \\
HEALNet-R50 & dme\_auc & -0.0091 & [-0.0394, 0.0134] \\
HEALNet-R50 & dme\_average\_precision & 0.0981 & [-0.0502, 0.2249] \\
HEALNet-R50 & ref\_auc & 0.0056 & [-0.0154, 0.0279] \\
HEALNet-R50 & ref\_average\_precision & 0.0286 & [-0.0059, 0.0703] \\
MADDi-R50 & dr\_accuracy & 0.0270 & [0.0040, 0.0532] \\
MADDi-R50 & dr\_qwk & 0.0355 & [-0.0257, 0.1009] \\
MADDi-R50 & dr\_macro\_f1 & 0.0620 & [-0.0440, 0.1653] \\
MADDi-R50 & dme\_auc & 0.0039 & [-0.0415, 0.0511] \\
MADDi-R50 & dme\_average\_precision & 0.0205 & [-0.0538, 0.1011] \\
MADDi-R50 & ref\_auc & -0.0196 & [-0.0407, -0.0017] \\
MADDi-R50 & ref\_average\_precision & -0.0133 & [-0.0414, 0.0174] \\
TIP-R50 & dr\_accuracy & 0.0539 & [0.0230, 0.0876] \\
TIP-R50 & dr\_qwk & 0.1274 & [0.0663, 0.2010] \\
TIP-R50 & dr\_macro\_f1 & 0.1296 & [-0.0258, 0.2924] \\
TIP-R50 & dme\_auc & -0.0035 & [-0.0412, 0.0231] \\
TIP-R50 & dme\_average\_precision & 0.1811 & [0.0402, 0.2945] \\
TIP-R50 & ref\_auc & 0.0276 & [-0.0018, 0.0549] \\
TIP-R50 & ref\_average\_precision & 0.0713 & [0.0296, 0.1221] \\
\bottomrule\end{tabular}\end{center}
\caption{Paired patient-cluster percentile intervals from 2,000 resamples. These unadjusted intervals condition on the fitted models and do not measure retraining variation.}\label{tab:paired}\end{table}

\clearpage
\section{Reliability, Controlled Perturbations and Mechanism Analysis}\label{app:analysis}
\subsection{Three complementary analysis packages}
The first package provides a broad analysis of the frozen main model: cohort and feature summaries, training trajectories, prediction metrics, calibration, reliability-state diagnostics, latent/head decomposition and initial attributions. A second package increases attribution depth through two metadata backgrounds, a larger image-case inventory, numerical completeness checks and sensitivity comparisons. A final package deepens state and matrix interventions and validates the differentiated linear system. All three use the same selected main checkpoint.

Analysis output is preserved at the level needed to distinguish a numerical check from a scientific claim. A small solve residual checks the computation of $Hh=b$; it does not imply a correct disease label. Attribution completeness checks an accounting identity; it does not imply convergence of each feature's allocation. A nonzero internal response shows that a mechanism affects predictions; it does not show that this response improves prediction quality.

\subsection{Natural quality and error ranking}
The image state q can be evaluated against the dataset's natural quality annotation without using quality labels as training supervision. The recorded quality-ranking AUROC is approximately 0.668 with a patient-cluster interval of $[0.531,0.787]$. This is evidence of a quality association in the evaluated cohort, not a calibrated probability of image acceptability. Natural quality checks and synthetic corruption checks answer different questions and use their recorded cohort definitions.

For error detection, the frozen u score is compared with predictive uncertainty scores such as $1-\max_c p_c$ and entropy. Overall u error-ranking AUROC is approximately 0.660 for DR, 0.910 for DME and 0.597 for Ref, whereas the corresponding maximum-probability score is approximately 0.800, 0.850 and 0.771. The overall DME result deserves further stratification: among 689 negatives there are 25 fixed-0.5 errors and u AUROC is approximately 0.988; among 54 positives there are five errors and the empirical u ranking AUROC is zero. The small positive error count prevents a precise population claim, but shows that an aggregate ranking advantage does not establish equally useful within-class behavior.

Risk--coverage curves retain low-risk examples first. Exact score ties are handled by expected random ordering within each tied group, rather than input-row ordering. This matters for nearly saturated state values. The error definition is task-specific: DR argmax and fixed 0.5 for binary tasks. It must not be silently mixed with the validation-Youden operating-point errors used in a separate results table.

\subsection{Output calibration}
Raw probabilities remain the primary results. A supplementary temperature is fitted separately for each task by validation negative log-likelihood, with $\log T$ constrained to $[-4,4]$, following the general temperature-scaling construction \citep{guo2017calibration}. It is applied to test logits without refitting. The method changes output confidence and can change binary threshold behavior if an unchanged probability threshold is used. It does not calibrate q, u or r, and is not part of the frozen raw-model ranking in the main tables.

Reliability diagrams include sparse bins, whose estimates carry greater sampling uncertainty. Binary Brier scores and ECE quantify different properties from discrimination AUROC. In particular, a low Ref Youden threshold can produce high sensitivity with lower specificity despite reasonable threshold-free discrimination. These are operating-point trade-offs, not contradictions in the metric computation.

\subsection{Controlled input conditions}
\begin{table}[!htbp]
\begin{center}\small
\setlength{\tabcolsep}{3pt}
\begin{tabular}{llrr}
\toprule
Family & Severity & Seeds & Count \\
\midrule
Clean & None & 1 & 1 \\
Gaussian blur & Radius 1, 2, 4 & 1 each & 3 \\
Gaussian noise & SD 0.025, 0.05, 0.10 & 3 each & 9 \\
Metadata masking & 0.25, 0.50, 0.75 & 3 each & 9 \\
Metadata replacement & 0.25, 0.50, 0.75 & 3 each & 9 \\
Joint & Blur 2 + masking 0.50 & 3 & 3 \\
\bottomrule\end{tabular}\end{center}
\caption{The 34-condition frozen-checkpoint design, repeated separately on validation and test.}\label{tab:conditions}\end{table}

Each split contains one clean condition, three blur levels, nine noise conditions, nine metadata-masking conditions, nine metadata-replacement conditions and three joint conditions, totaling 34. Gaussian-noise and metadata conditions use seeds 61001, 61002 and 61003. Blur is deterministic at its recorded radius. The joint condition uses blur radius two and nominal 50\% metadata masking; it is one joint stress setting with three seeds, not a factorial response surface.

Image processing is PIL RGB conversion, bilinear resize to $512\times512$, the specified corruption, tensor conversion and ImageNet normalization. Noise draws are keyed to the image; metadata masks and donor choices are keyed to the patient. Nested random draws support matched severity comparisons. Replacement donors come from training patients. Actual changed-field masks are saved, because a nominal replacement can preserve a value by chance and a selected field may already be missing.

Every disease score is computed on the original eligible photograph set with the original task masks. The original clean-validation thresholds are held fixed across severity. The analysis never tunes a new threshold to a corrupted test condition. Curves show the sample standard deviation of three corruption-seed metric values where applicable, not a confidence interval over patients and not uncertainty over retraining.

\subsection{Clean-state interventions}
For each corrupted photograph, the adaptive branch uses its newly computed state. Four comparison branches replace q alone, u alone, r alone, or all three with the corresponding photograph's clean state while retaining current corrupted embeddings and tokens. These are privileged-reference diagnostics. At deployment, an unavailable clean version cannot be assumed, and the fixed-state branch is not an external baseline.

Individual coordinate effects need not sum to the all-coordinate effect: trust gates, interactions and the solve depend jointly on the state. Metadata-only changes should preserve image-derived q and u; image-only changes should preserve metadata-derived r. These invariances were numerically checked. The radius was recomputed as a state diagnostic, but no inference computation consumes it to modify the prediction.

Across noise levels, the adaptive branch can improve DR QWK relative to all-coordinate fixing while worsening DME AUROC. Under metadata missingness, its differences are smaller. Under incorrect but complete metadata, r changes little because its dominant explicit term measures availability rather than truthfulness. The appropriate conclusion is conditional state response with task-specific consequences, not a universal uncertainty correction.

\subsection{State scans and matrix substitutions}
One photograph per patient is chosen by first row order, then 64 patients are selected with a recorded random seed independently of labels or errors. The same selected patients are used across six conditions: clean, moderate blur, moderate noise, 50\% masking, 50\% replacement and the joint condition. This common sample allows paired geometric interpretation without changing the cohort for a more favorable outcome.

State scans hold all encoded features fixed. For each coordinate, they use five training-patient quantiles (5th, 25th, 50th, 75th and 95th), plus boundary stresses at zero and one. Pairwise scans use three quantiles per coordinate. Together with the native state this yields 49 settings. Training quantiles are calculated from the first photograph of each of the 900 training patients, not fitted to the test distribution. Boundary values can be outside typical empirical support and are identified as internal stress tests.

Matrix interventions retain the trace while replacing one or both trust matrices by their diagonal or a scaled identity. The same replacement is made in $H$ and in the matching evidence term of $b$; interactions stay at their native values. Altering only $H$ would be a different intervention. Clean probes cover every eligible validation/test photograph; corrupted probes cover the selected 64-patient samples. These fixed-weight interventions are reported separately from A01/A02 retraining.

\subsection{Checks against numerical artifacts}
The archive contains 963 hashed payload files and 267 verified completion records. Independent recomputation reproduces all 18,360 condition-metric rows, 9,828 probe-metric rows and 753,270 per-photo change rows. CPU replay from the frozen features covers 24,931 forward evaluations, 768 matrix-spectrum cases and all 96 derivative checks. Maximum DR probability disagreement with the original GPU analysis is approximately $1.02\times10^{-6}$; the corresponding DME and Ref maxima are below $4.5\times10^{-7}$ and $3.0\times10^{-7}$. All 441 model tensors remain unchanged.

The feature replay is not a claim that every corrupted-image CNN pass was rerun on CPU. It verifies the cached-feature-to-prediction mechanism path; clean full-image checks were also performed in the attribution package. Separating these scopes avoids presenting a file-integrity audit as an independent replication of the entire study.

\clearpage
\section{Metadata and Image Attribution Protocol}\label{app:attribution}
\subsection{Explanation targets}
Explanations depend on the output being explained. We use DR expected grade, $\sum_{c=0}^4c\,p_c$, as a continuous severity-oriented summary of the categorical model, together with individual DR probabilities and DME/Ref probability or log-odds outputs. Expected grade is an analysis target, not the training loss or the reported argmax decision. Binary log odds emphasize a different numerical scale from probabilities, especially near saturation. Attribution rankings from different target functions must not be pooled as though they were identical quantities.

\subsection{Metadata permutation attribution}
For each of 743 test photographs, the image representation is held fixed while subsets of the 16 raw metadata fields are replaced by values from training-patient background records. The altered record passes through the frozen preprocessing, metadata encoder, state descriptor and fusion/head path. This is an interventional donor-replacement game, rather than conditional sampling from the true joint distribution of clinical variables. Correlated or clinically constrained fields can therefore form combinations not typical of the observed cohort.

Two disjoint backgrounds contain 64 distinct training patients each. Each photograph/background calculation uses four antithetic permutation pairs: a permutation and its reverse share the donor. Marginal increments are averaged over the eight paths, and pair-level dispersion supplies a Monte Carlo standard error. The path telescoping property gives additive completeness even at a small path budget; it does not establish convergence of every individual field attribution. Numerical replay checks the endpoints and recomputes the allocation summaries independently.

Global importance is the average absolute attribution. Photo-weighted summaries count each photograph once, while patient-balanced summaries average within patient before averaging patients. Across the two backgrounds, DR expected-grade global rank correlation is 0.891; DME and Ref log-odds correlations are approximately 0.962 and 0.950. These results support stability of the leading global rankings under the tested donor choice. They do not establish local clinical causality or guarantee stability under other background populations.

\subsection{Image case inventory and methods}
The image inventory contains 32 distinct test patients. Eight deep cases are selected by declared class/error rules to span DR grades and binary failures; 24 additional cases provide broader coverage. This is a structured inspection set, not an unbiased sample for estimating explanation quality in the population. All selected cases, including failures and numerical caveats, are retained in the case atlas.

Integrated gradients \citep{sundararajan2017axiomatic} integrates the gradient along the path from a reference image to the observed image, while keeping metadata fixed. Region permutation attributions use SLIC regions as players, including a background region, and evaluate replacement with a blurred or mean image. These sampled Shapley-style estimates follow the additive-attribution framework \citep{lundberg2017unified}; they are not TreeSHAP or an exact sum over every possible ordering. Grad-CAM \citep{selvaraju2017gradcam} summarizes gradient-weighted convolutional activations and is nonnegative in the displayed version. In signed maps, red denotes positive and blue negative contributions to the stated output; bright Grad-CAM colors indicate larger nonnegative activation weights. These are different explanatory objects.

The eight deep cases use 256-node Gaussian-quadrature IG, roughly 96 SLIC regions and 16 antithetic permutation pairs for the primary image attribution. Additional checks vary blur versus mean replacement, segmentation granularity, permutation halves, IG node budgets, and layer3 versus layer4 Grad-CAM. The 24 broad cases use 32-node IG, roughly 48 regions and two permutation pairs. Actual region counts can differ from requested counts and are recorded per case.

\subsection{Completeness, agreement and faithfulness}
IG completeness compares the summed pixel attribution with the output contrast $F(x)-F(x_0)$. Both absolute residual and residual relative to the contrast are reported. A nearly zero contrast can make a relative error large even when the absolute residual is small, so the case tables retain both. For the deep blur/256-node setting, all 24 case--output combinations have relative residual below 2.6\%. In the broad blur/32-node setting, 24 of 72 combinations exceed 10\%; four mean-reference/128-node sensitivity combinations also exceed 10\%. The package retains these results instead of silently labeling every map converged.

Agreement uses rank correlations and overlap of high-attribution regions, with the ranking convention stated. Changing background, segmentation or explanation method can alter the map. For deep-case independent permutation halves, median spatial rank correlations are approximately 0.438 for DR, 0.616 for DME and 0.605 for Ref. Strong global metadata ranking stability therefore does not imply that all image-level explanations are equally stable.

Deletion/insertion analyses replace regions in attribution-ranked order and track the chosen target output. Random-order references provide context. Their interpretation depends on the replacement baseline and target direction: removing positive evidence, removing negative evidence and removing high absolute attribution answer different questions. A large area under a deletion curve is not intrinsically good or bad without those conventions. These perturbation tests assess the predictor's response to artificial interventions, not biological mechanisms or overlap with expert lesion masks.

Two predetermined image cases were repeated with identical settings to check strict repeatability. Repeated outputs agreed, and full-image clean replay for all 32 cases matched the frozen reference predictions within the recorded tolerance. Numerical repeatability, finite-budget stability and clinical validity are separate properties; the case atlas keeps them separate.

\subsection{Explanation evidence and its scope}
Saliency sanity checks show why visually plausible maps alone are insufficient evidence of dependence on learned parameters or training data \citep{adebayo2018sanity}. Attention weights can also disagree with other measures of explanatory importance, as demonstrated in language models \citep{jain2019attention}. These findings motivate keeping attention-based feature exchange distinct from the input-attribution analyses here. We report numerical completeness, configuration sensitivity and deletion/insertion responses for the trained predictor. We did not perform parameter- or label-randomization sanity checks, and these diagnostics are not substitutes for lesion annotations or clinical validation.

\clearpage\section{Data, Prediction and Representation Atlas}\label{app:atlas}
The following panels complement the main quantitative comparisons. Each figure addresses a different data or model question. The text states its population, analysis convention and interpretation; visual separation or attribution magnitude alone is not treated as predictive evidence.

\subsection{Cohort structure}
\assetfigure{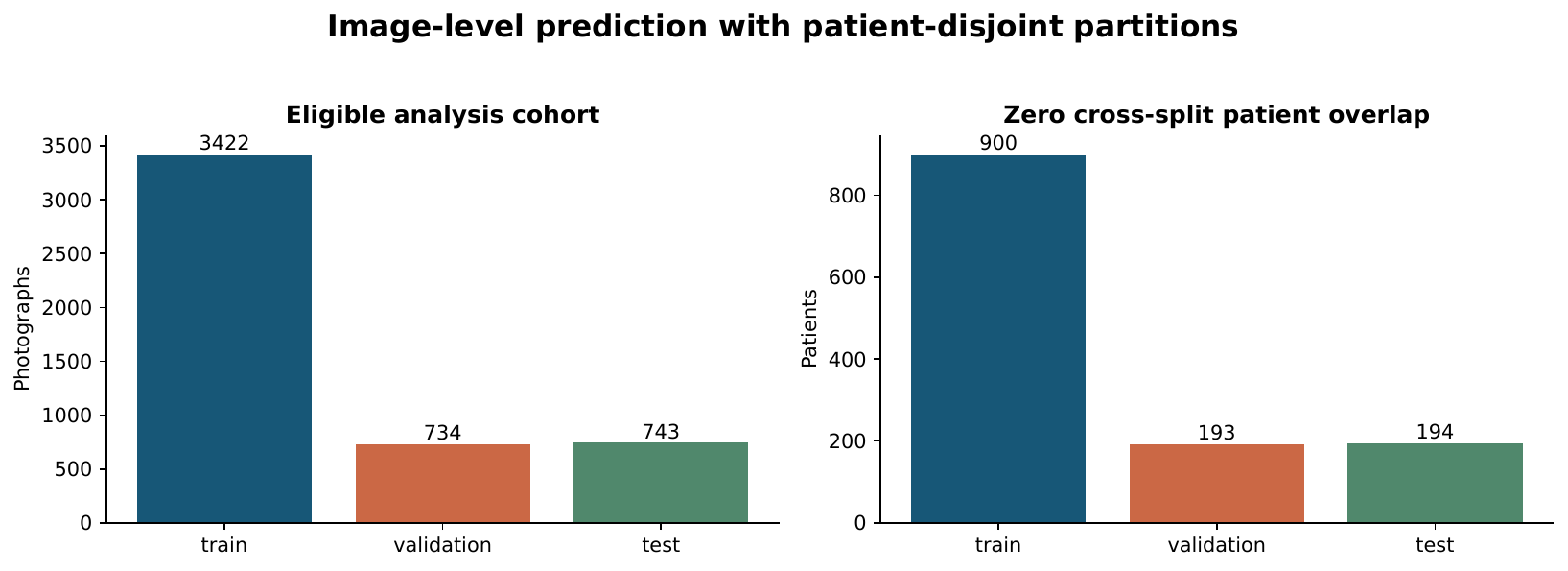}{Each prediction targets a photograph and its metadata. Patient-disjoint partitioning controls cross-split dependence; it does not change the prediction unit. Task-specific unknown labels remain masked. These counts describe the eligible cohort, not all 5,164 source photographs.}{fig:atlas0}

\subsection{Class support}
\assetfigure{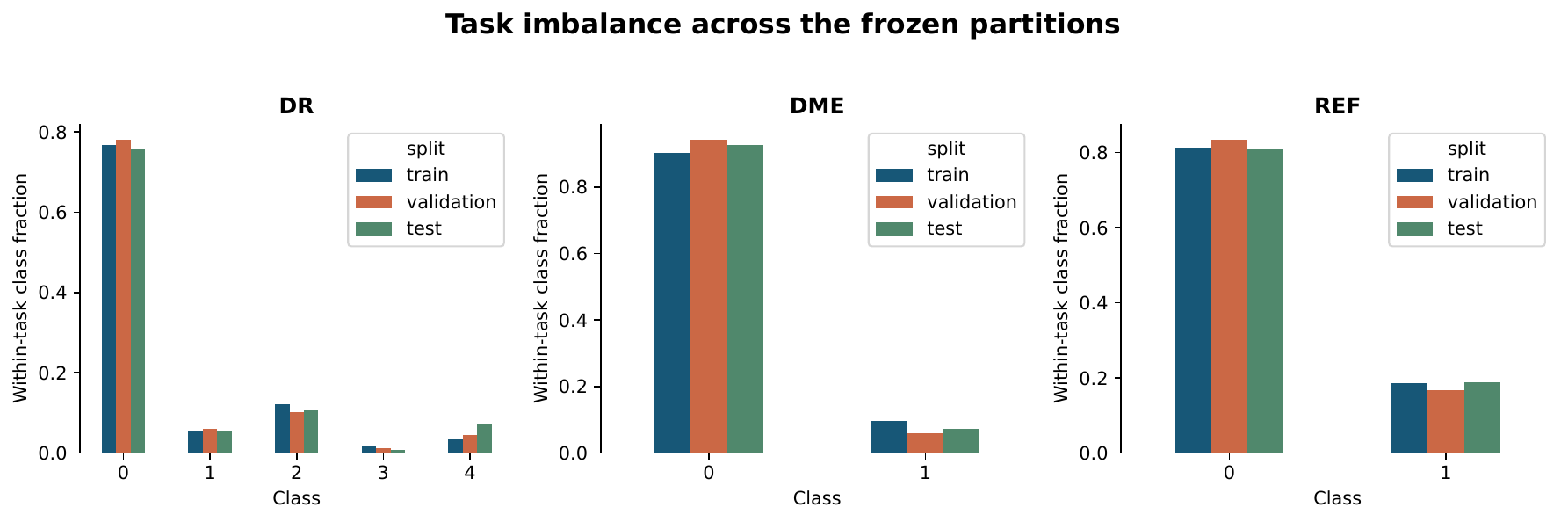}{Fractions use the valid labels of each task and split. REF is derived from DR/DME, so these are not three independent disease outcomes. All counts, including the six test images in DR grade 3, are retained.}{fig:atlas1}

\clearpage\subsection{Metadata availability}
\assetfigure{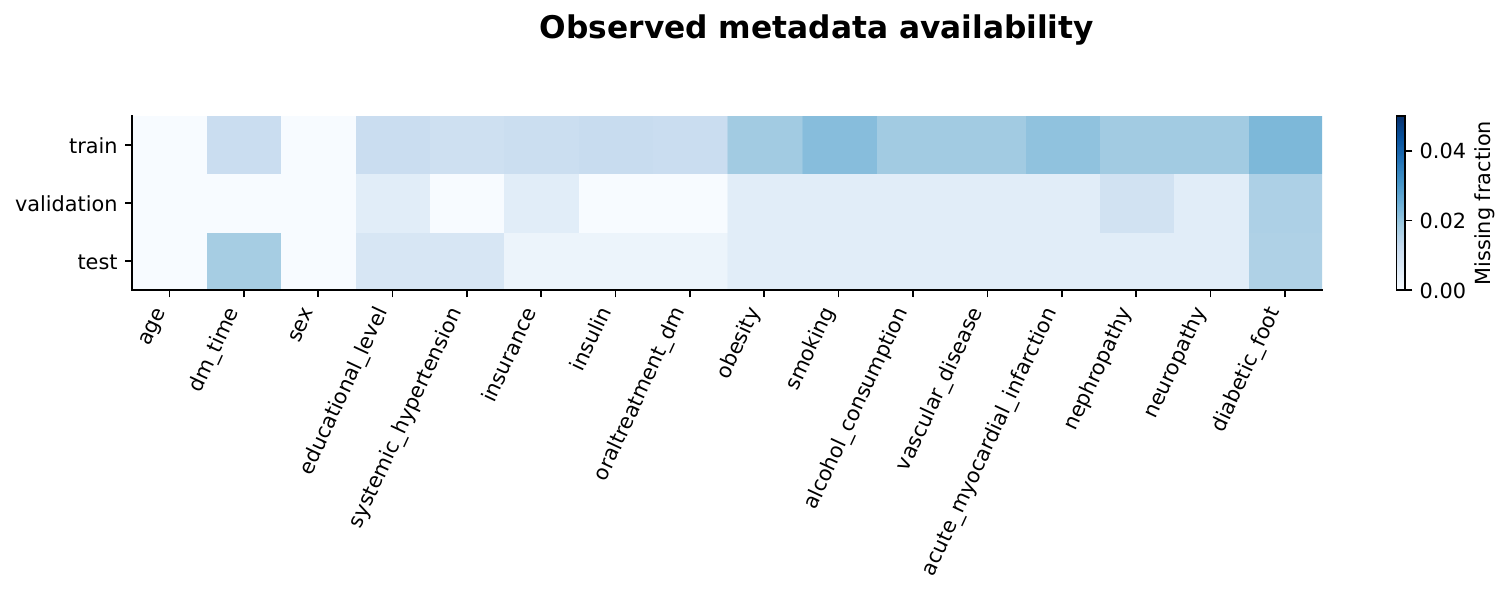}{Photo-weighted missing fractions before train-fitted imputation. These are observed missing values, not synthetic corruptions. Repeated photographs share patient metadata.}{fig:atlas2}

\subsection{Main training dynamics}
\assetfigure{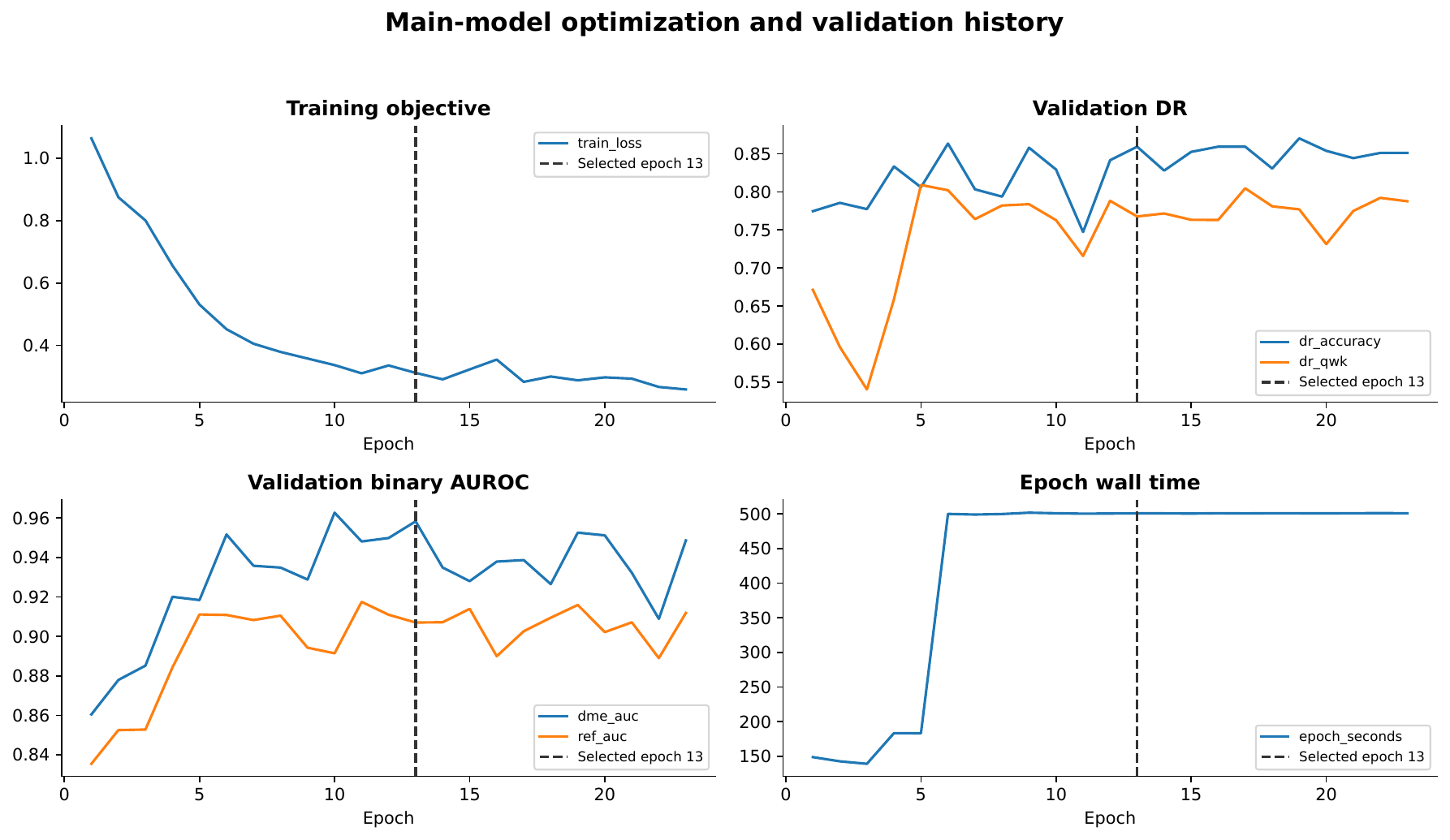}{Original unsmoothed logs: best epoch 13, last epoch 23, actual patience 10. Loss schedules change during warmup, so total loss magnitudes across epochs are not a fixed-objective comparison. Epoch times include the recorded training and validation workload.}{fig:atlas3}

\clearpage\subsection{DR confusion structure}
\assetfigure{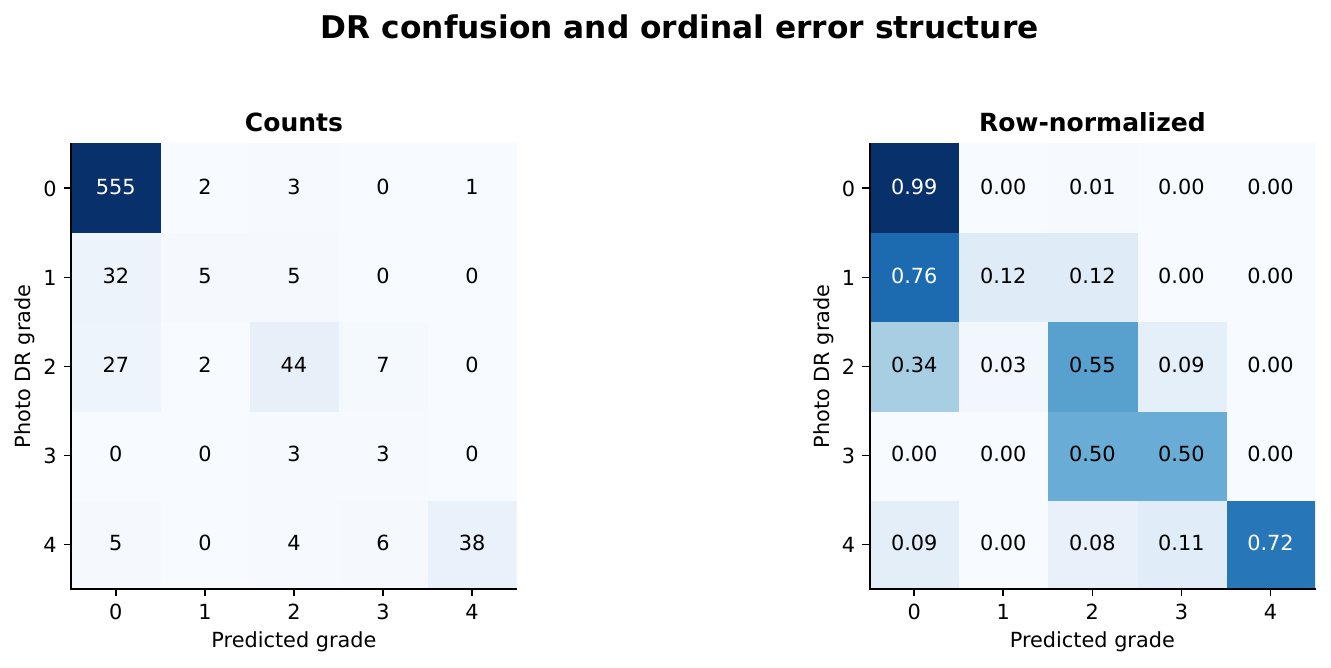}{All 742 valid test photographs. Minority-grade counts must accompany normalized values. Predictions are class argmax, with no threshold adjustment.}{fig:atlas4}

\subsection{DR grade-level performance}
\assetfigure{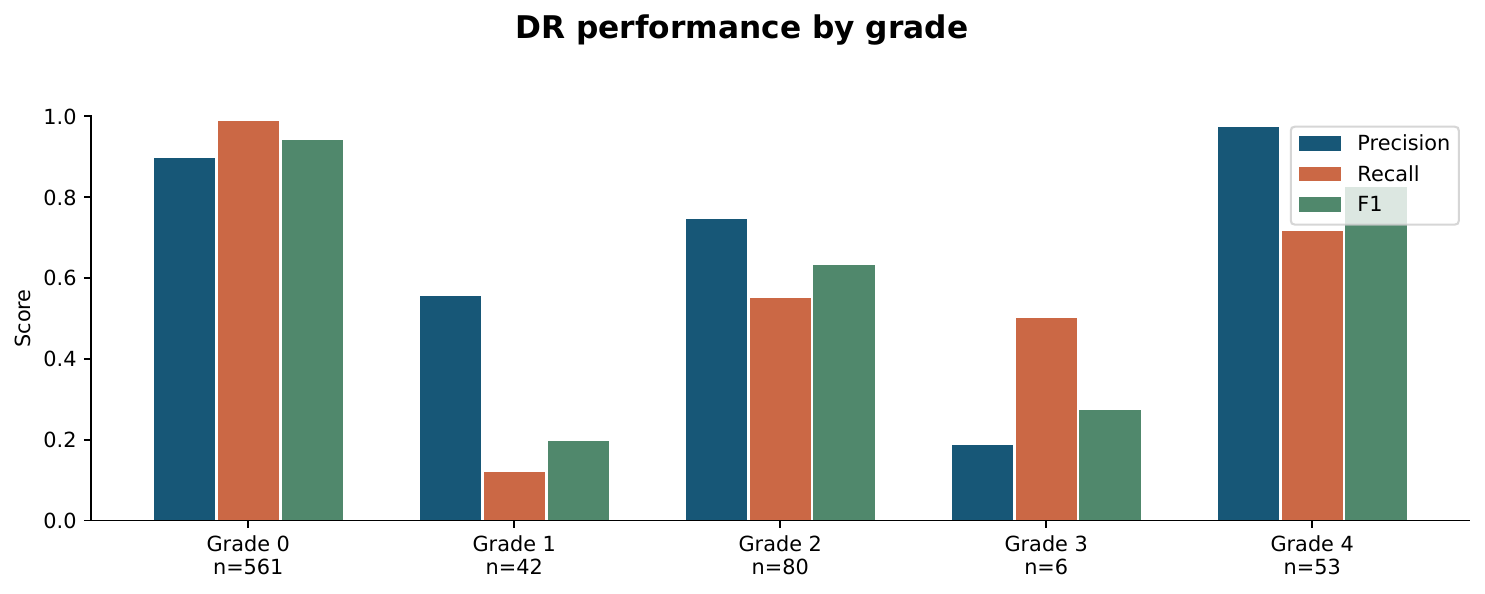}{Macro-F1 gives equal weight to each grade. Grade 1 recall and grade 3 support show limitations obscured by overall accuracy. No rare-class significance claim is made.}{fig:atlas5}

\clearpage\subsection{Binary operating-point analysis}
\assetfigure{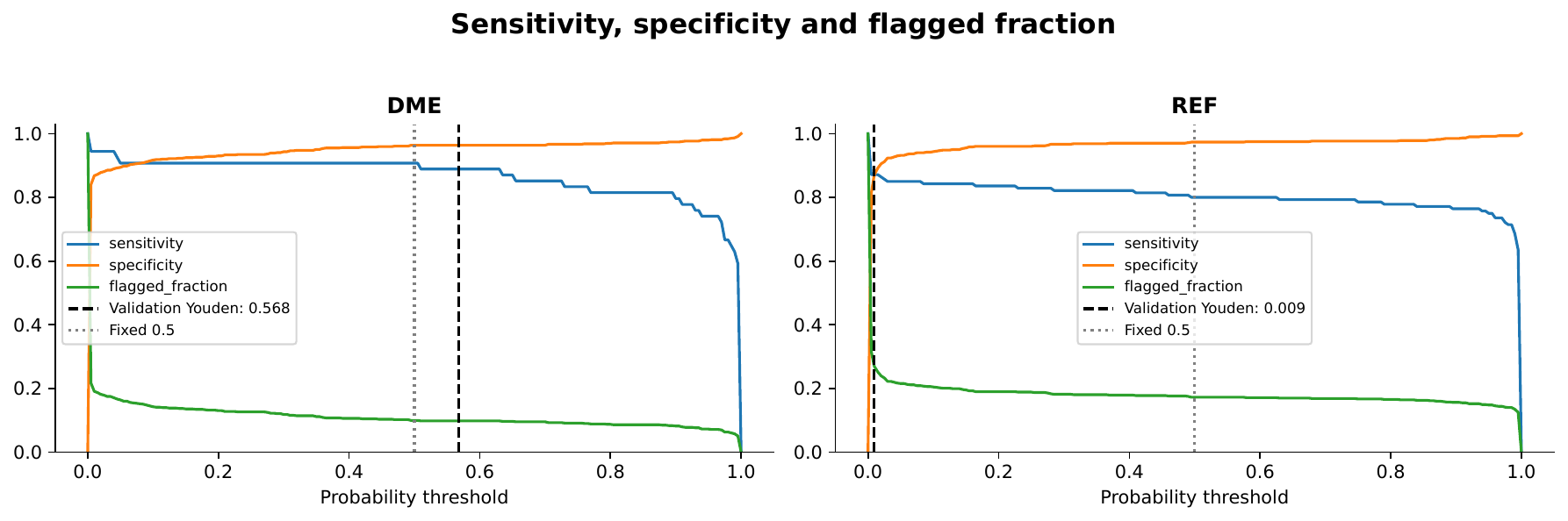}{Test operating curves with the two fixed reporting rules: validation-selected Youden and probability 0.5. No test threshold optimization is performed. The Ref Youden threshold of 0.009 gives a different sensitivity/specificity trade-off from 0.5.}{fig:atlas6}

\subsection{Probability calibration}
\assetfigure{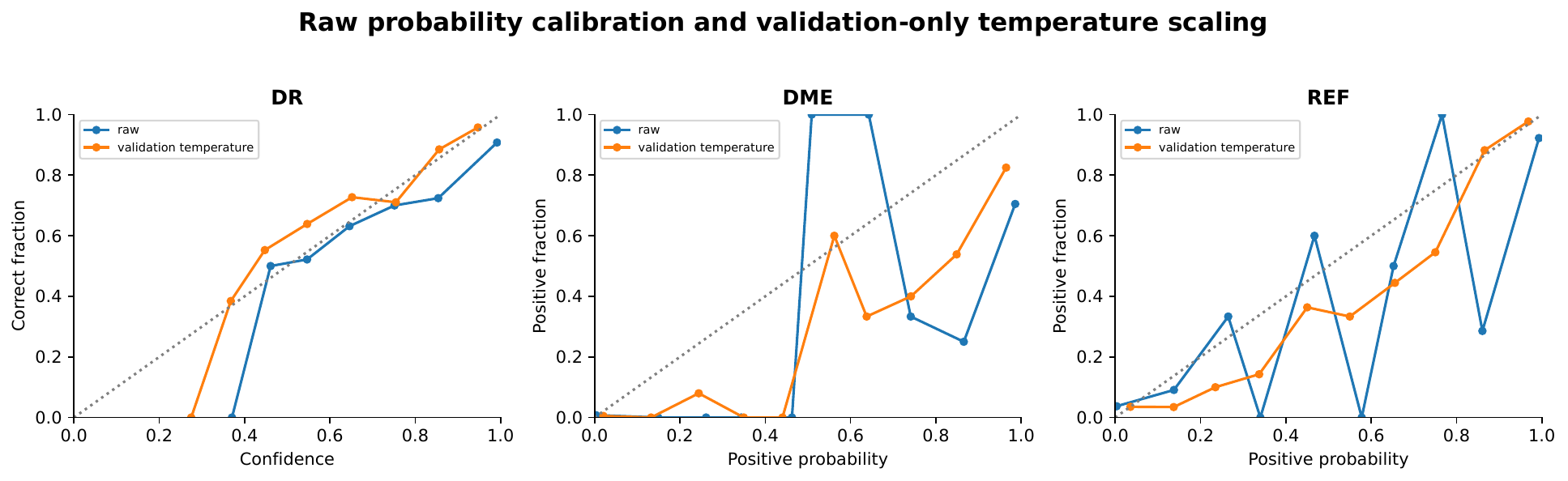}{Temperature is fitted separately per task using validation NLL only, with log(T) bounded to [-4,4]. Raw results remain primary. DR uses top-label ECE and sum-of-classes Brier; binary Brier uses positive probability. This analysis does not establish calibration of q/u/r. Sparse bins are retained in the table.}{fig:atlas7}

\clearpage\subsection{Selective prediction}
\assetfigure{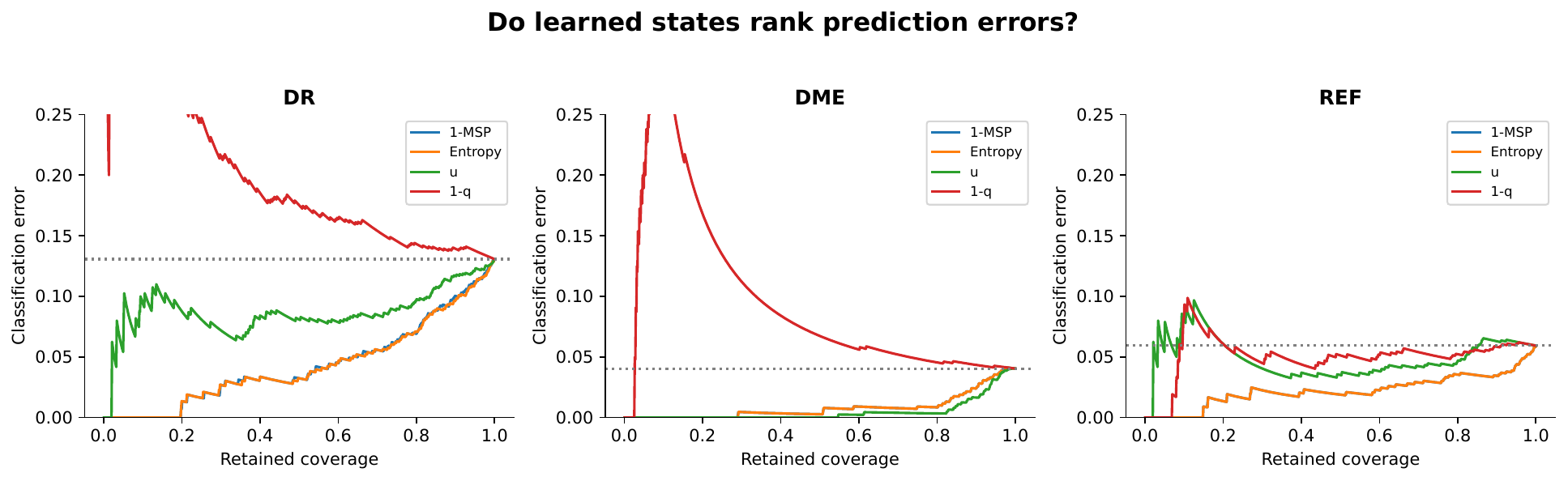}{Classification uses argmax for DR and fixed 0.5 for binary tasks. Lower scores are retained first. Exact score ties use expected random ordering within each tie group, avoiding row-order artifacts from saturated states. MSP/entropy are predictive scores; u is a learned state, not measured epistemic uncertainty.}{fig:atlas8}

\subsection{State distributions}
\assetfigure{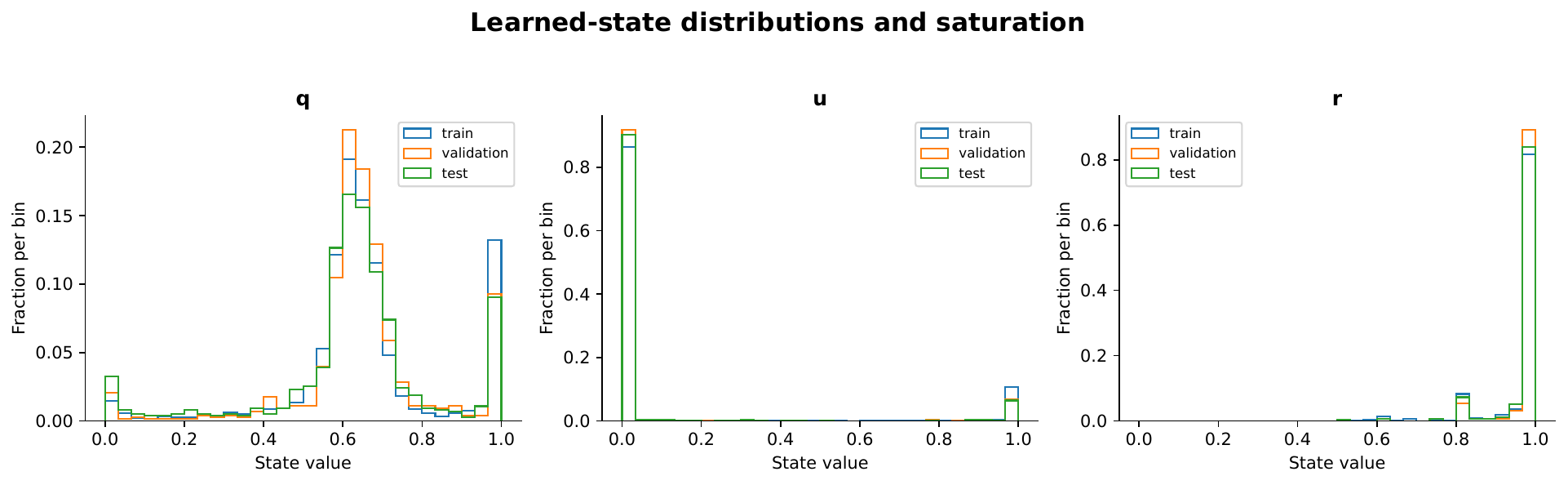}{State distributions of the selected model. The image-derived q and u coordinates are sigmoid outputs; r includes metadata completeness. Histograms summarize the learned state, while numerical saturation and reliability checks are reported in the analysis appendix.}{fig:atlas9}

\clearpage\subsection{Interaction magnitude and alignment}
\assetfigure{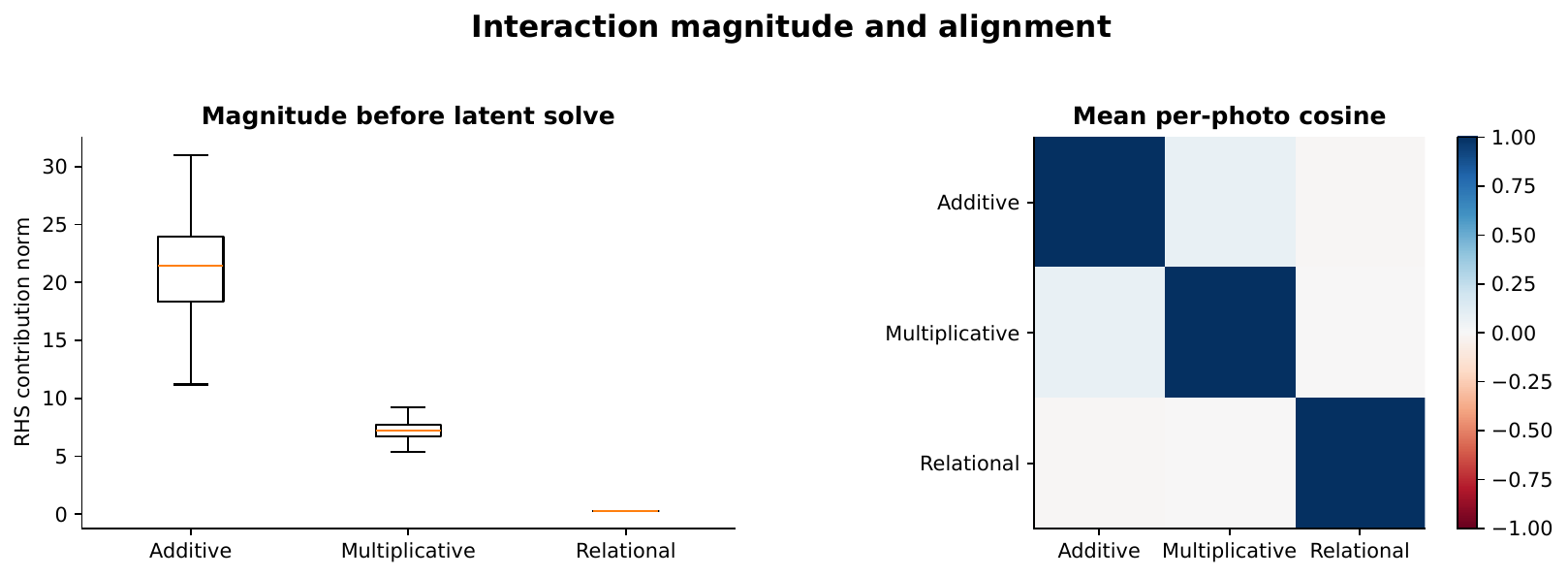}{Norms and alignment describe the frozen right-hand-side contributions. They are not causal importance and do not imply that removing a larger term causes a larger accuracy loss. Box plots omit drawn outliers only; all rows are exported.}{fig:atlas10}

\subsection{Representation geometry}
\assetfigure{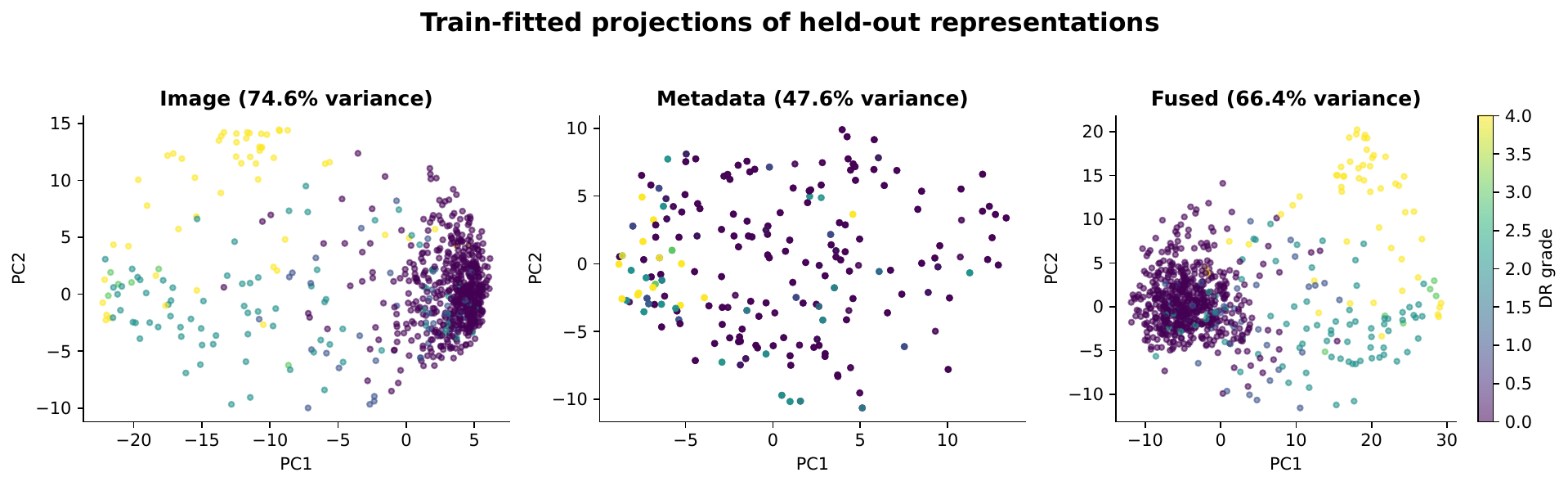}{Each PCA is fitted only to training features and then applied to test features. Axes belong to different spaces and cannot be compared coordinate-wise. Apparent separation is descriptive, not evidence of superior prediction or causal disentanglement. No test-driven projection tuning is performed.}{fig:atlas11}

\clearpage\subsection{Metadata associations}
\assetfigure{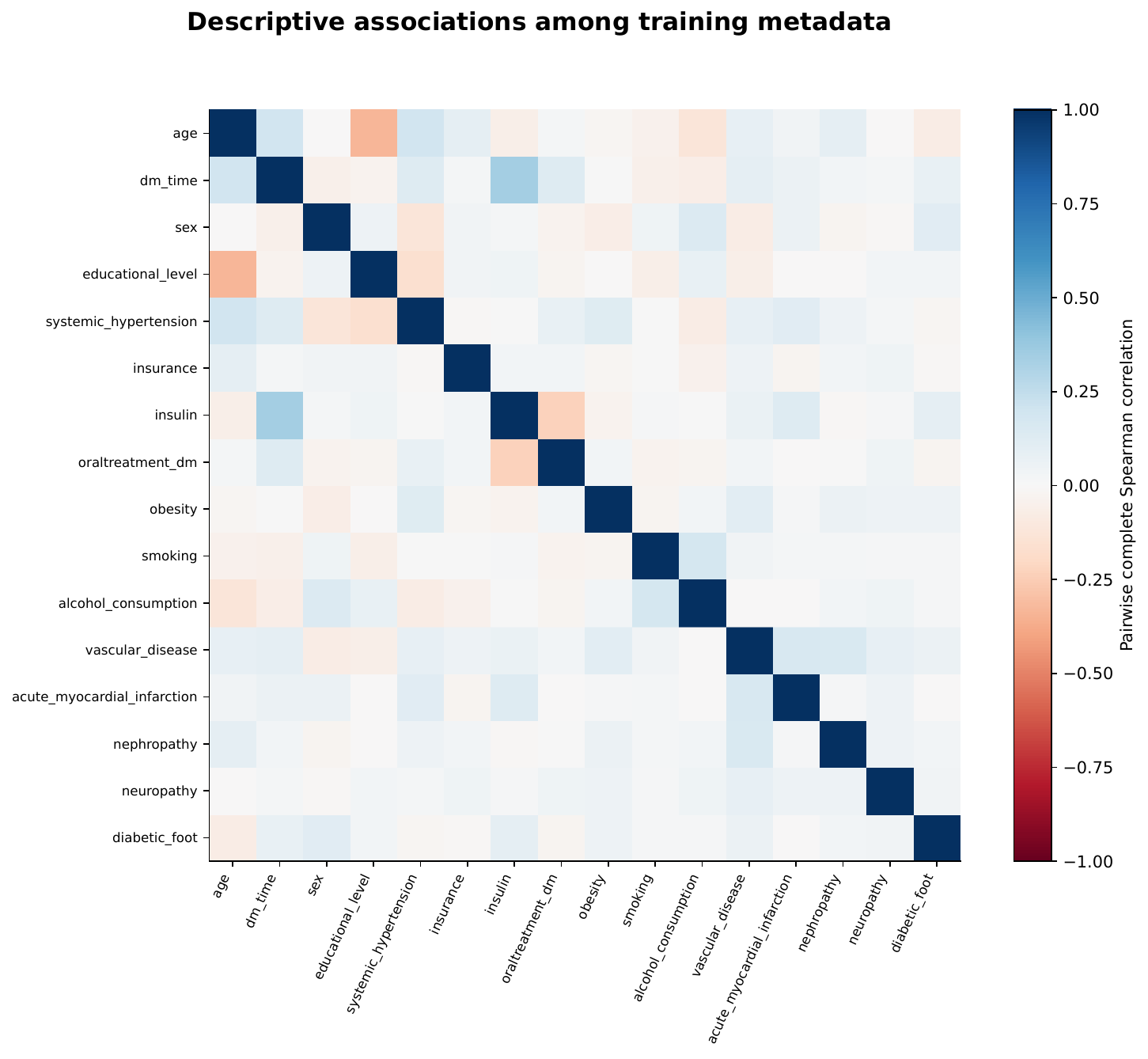}{One record per training patient. Age, diabetes duration, binary-coded fields and ordered education levels are summarized by pairwise Spearman correlation, excluding missing pairs. These are descriptive associations rather than effect estimates or predictive feature importance.}{fig:atlas12}

\subsection{Exploratory subgroup results}
\assetfigure{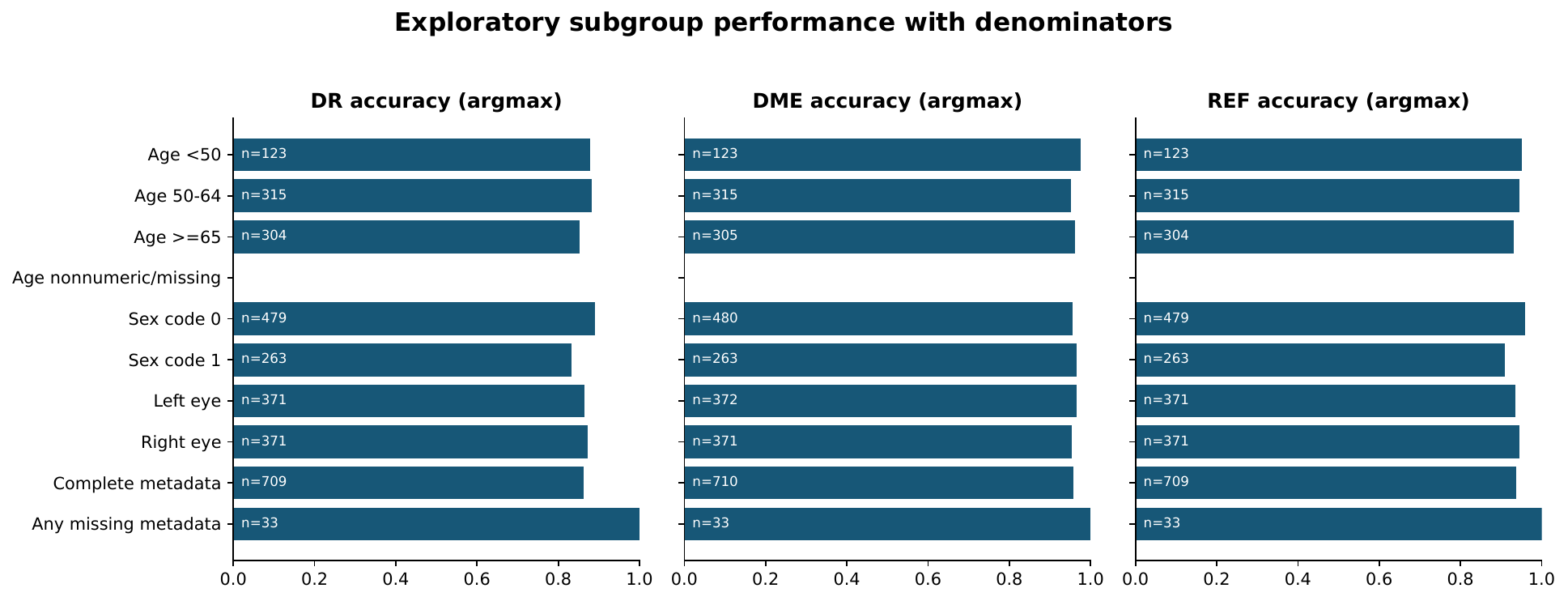}{Fixed descriptive age ranges, source sex codes, eye laterality and metadata availability. Accuracy uses argmax. Groups overlap and sparse positives limit binary conclusions. No fairness, clinical utility or subgroup significance claim is supported by this plot alone.}{fig:atlas13}

\clearpage\subsection{Trust spectra}
\assetfigure{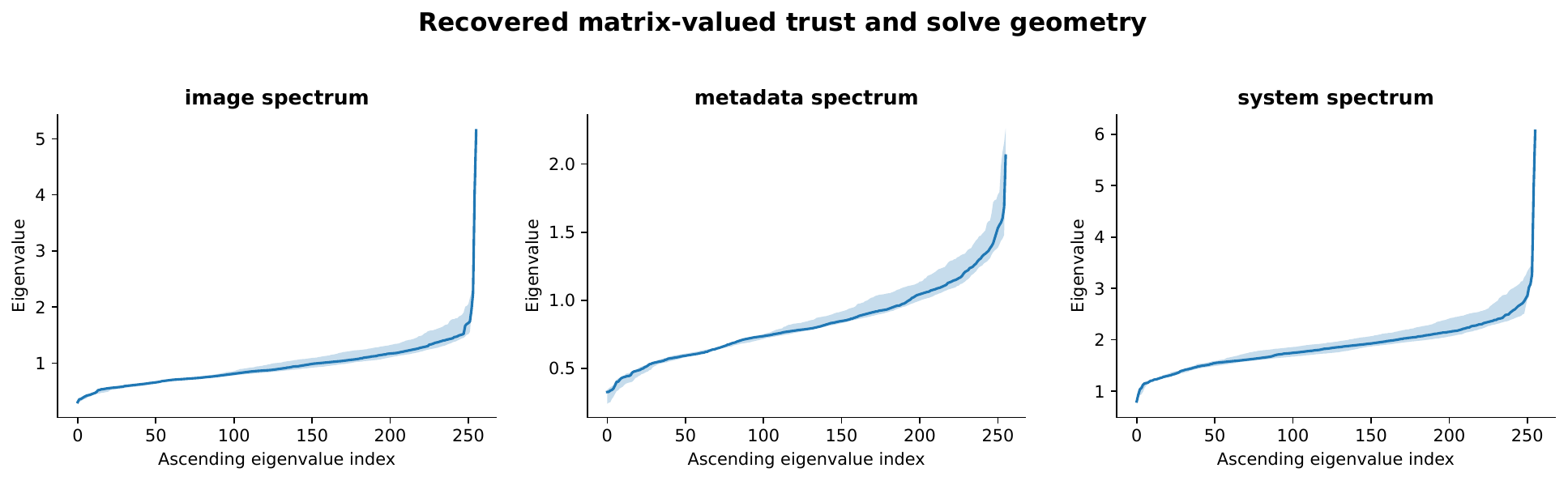}{All 743 test states are passed through the frozen trust operators. Curves show median ordered eigenvalues and the 5th--95th percentiles across photographs, rather than confidence intervals. Spectra use the complete matrices, including off-diagonal terms.}{fig:atlas14}

\subsection{Matrix structure}
\assetfigure{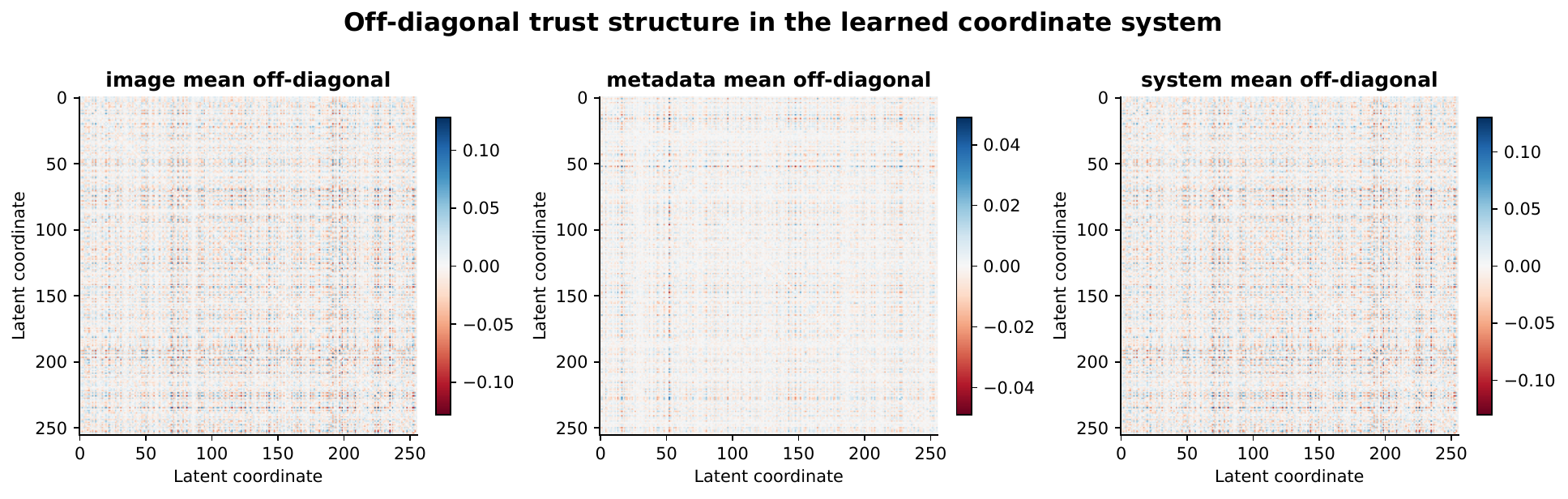}{Diagonal entries are removed only for display to reveal learned coupling. Latent coordinates have no direct clinical meaning. Separate color scales prevent small structure being visually erased; scale bars must remain. These maps are descriptive and are not attention or feature importance.}{fig:atlas15}

\clearpage\subsection{Latent decomposition}
\assetfigure{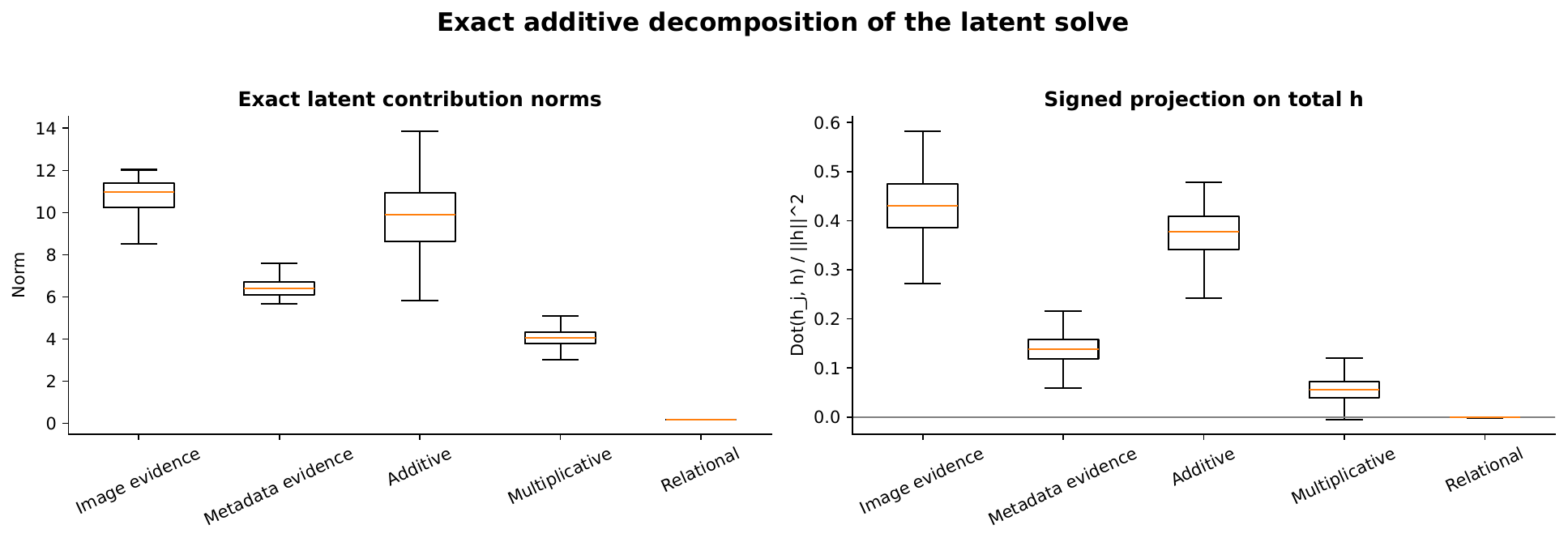}{At a fixed input, five latent components are obtained by solving the same system separately for the image, metadata, additive, multiplicative and relational evidence terms. Their sum equals the full latent vector. Signed directional shares sum to one but can be negative or exceed one. This is an exact algebraic decomposition, distinct from retrained ablation or causal attribution.}{fig:atlas16}

\subsection{Quadratic head decomposition}
\assetfigure{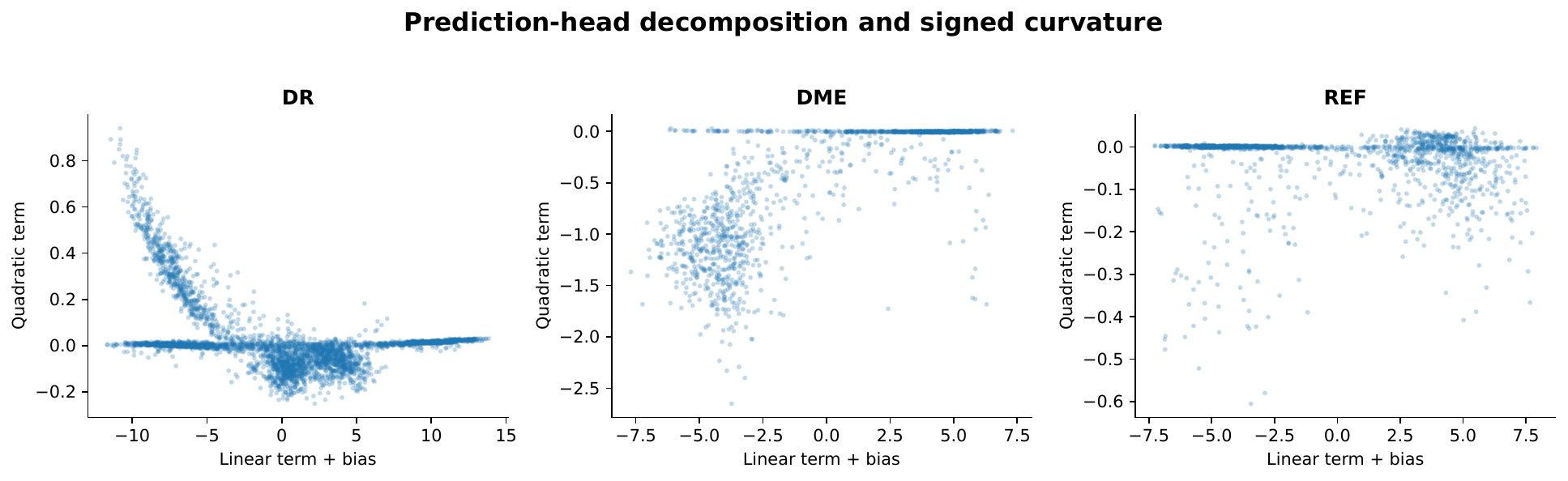}{Linear and quadratic contributions include every sample and output class. The complete cross-term expansion is numerically verified against the frozen logits. Signed diagonal coefficients and signed quadratic scales permit indefinite effective head curvature.}{fig:atlas17}

\clearpage\subsection{DR ordinal response to controlled perturbations}
\assetfigure{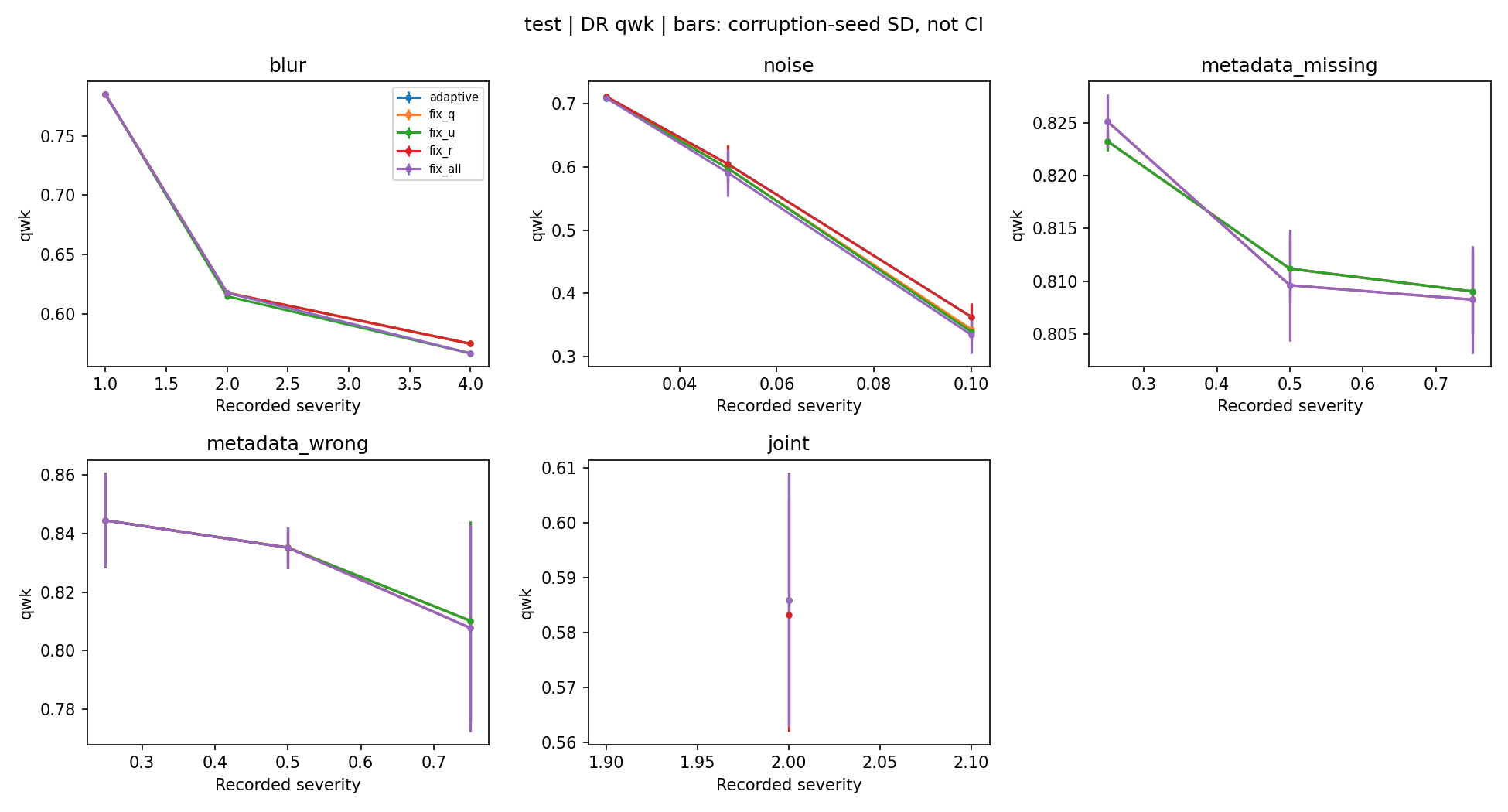}{Test results at the recorded severities. Random conditions use three corruption seeds, and bars show their sample SD. Clean-state coordinate fixing uses privileged clean references. The joint setting is one condition, not a severity sweep; no external model is compared in these panels.}{fig:test_dr_qwk_severity}

\subsection{DR accuracy under controlled perturbations}
\assetfigure{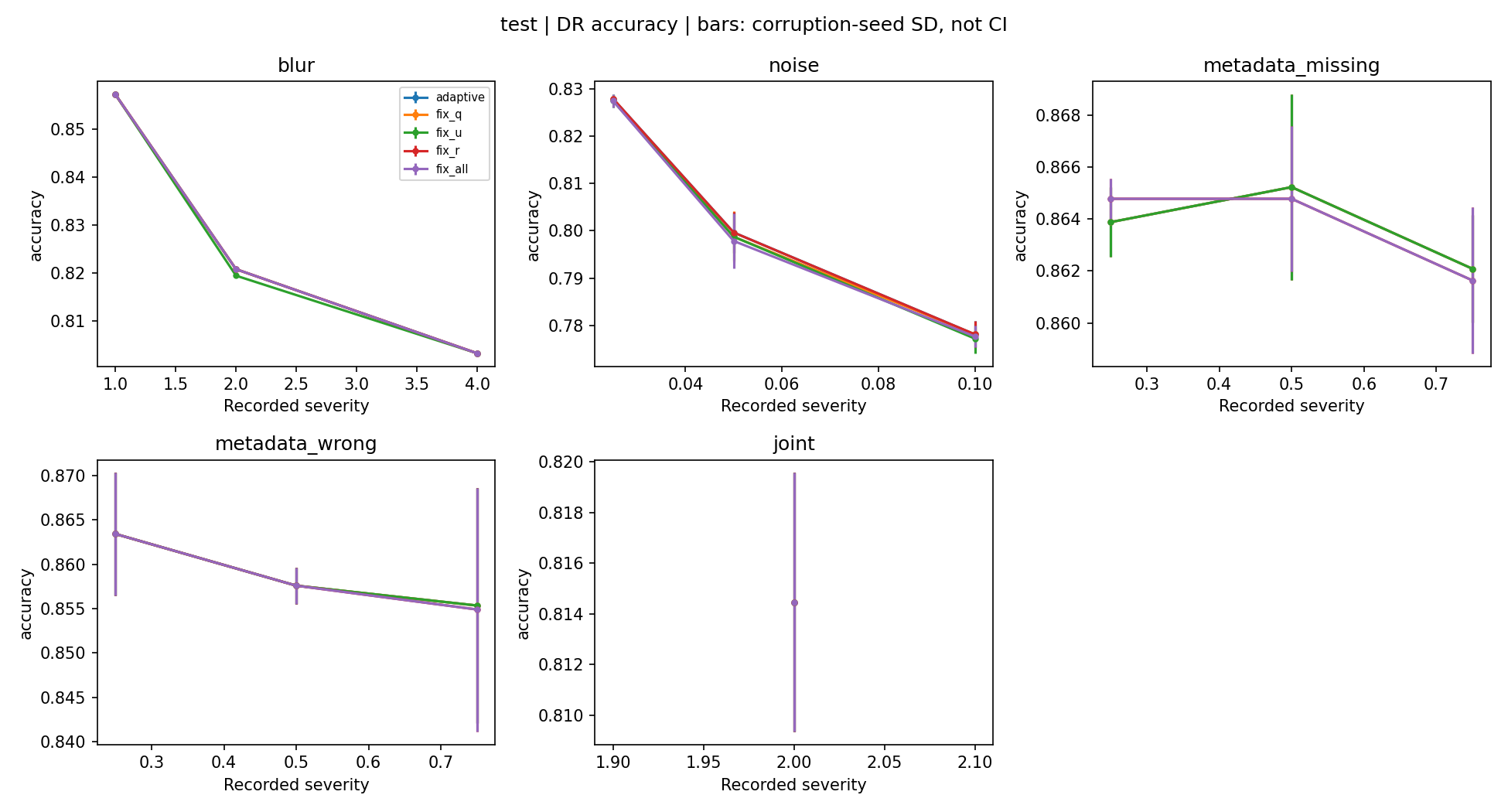}{Test results at the recorded severities. Random conditions use three corruption seeds, and bars show their sample SD. Clean-state coordinate fixing uses privileged clean references. The joint setting is one condition, not a severity sweep; no external model is compared in these panels.}{fig:test_dr_accuracy_severity}

\clearpage\subsection{DR macro-F1 under controlled perturbations}
\assetfigure{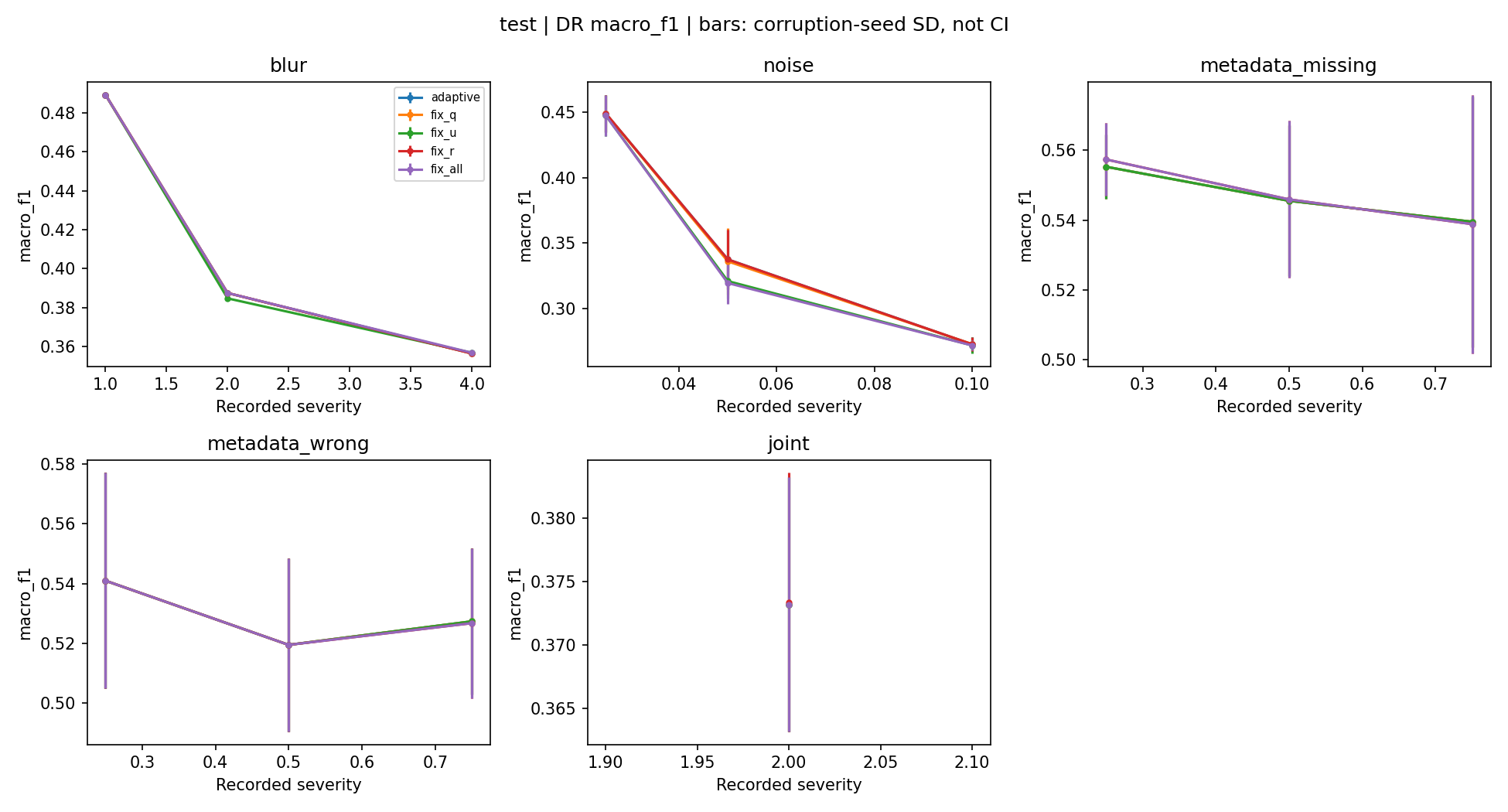}{Test results at the recorded severities. Random conditions use three corruption seeds, and bars show their sample SD. Clean-state coordinate fixing uses privileged clean references. The joint setting is one condition, not a severity sweep; no external model is compared in these panels.}{fig:test_dr_macro_f1_severity}

\subsection{DME discrimination under controlled perturbations}
\assetfigure{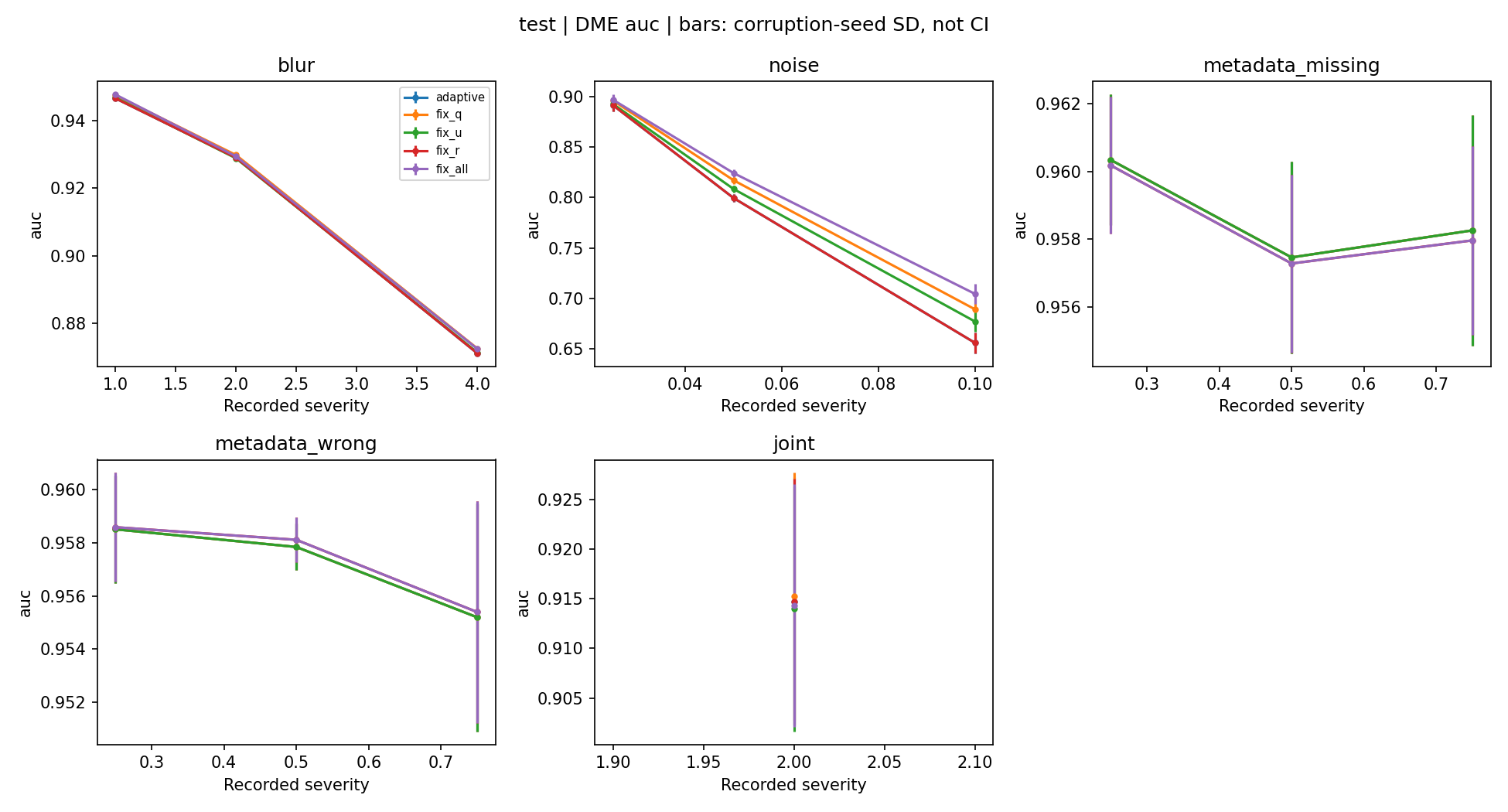}{Test results at the recorded severities. Random conditions use three corruption seeds, and bars show their sample SD. Clean-state coordinate fixing uses privileged clean references. The joint setting is one condition, not a severity sweep; no external model is compared in these panels.}{fig:test_dme_auc_severity}

\clearpage\subsection{Ref discrimination under controlled perturbations}
\assetfigure{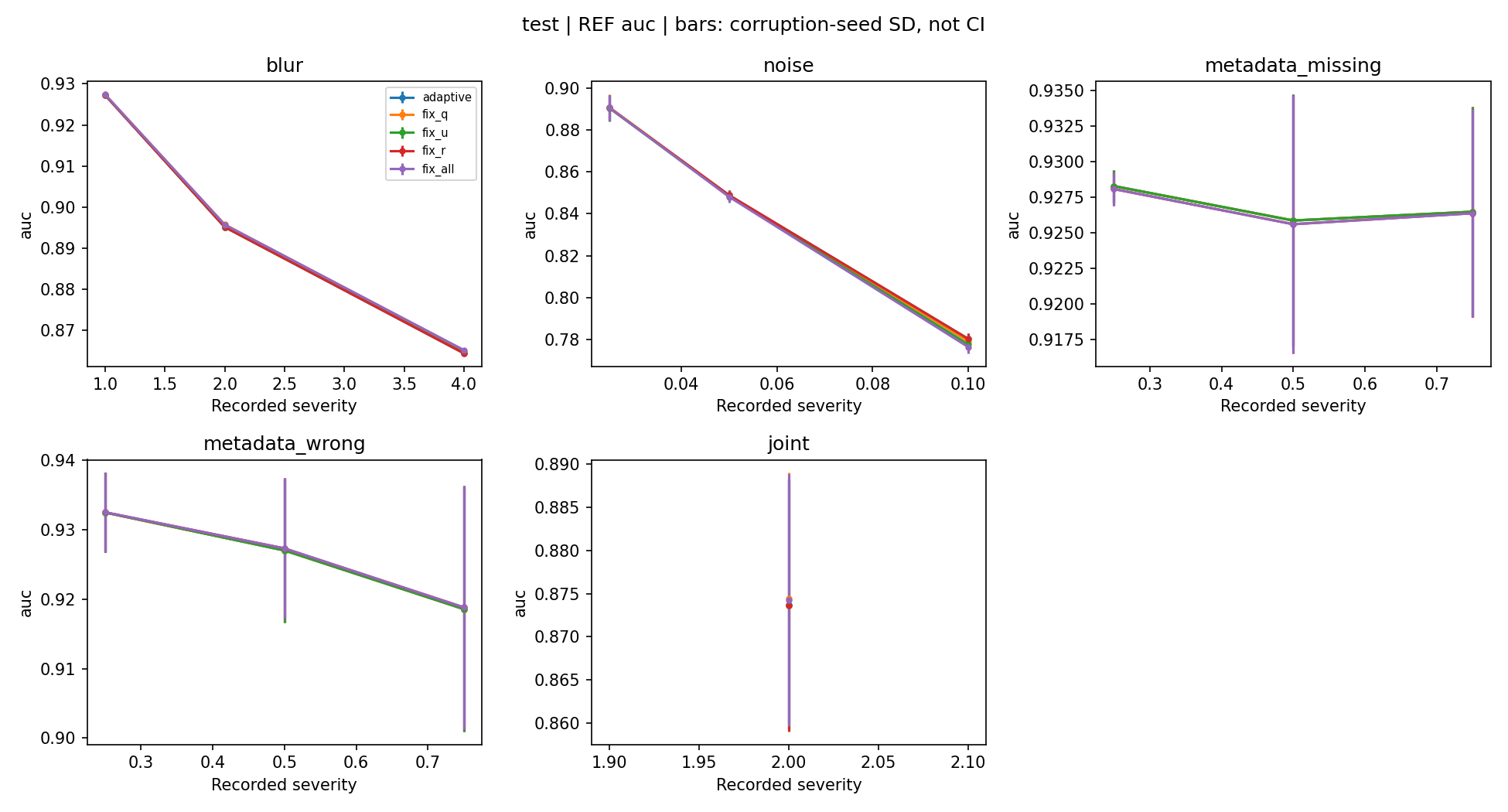}{Test results at the recorded severities. Random conditions use three corruption seeds, and bars show their sample SD. Clean-state coordinate fixing uses privileged clean references. The joint setting is one condition, not a severity sweep; no external model is compared in these panels.}{fig:test_ref_auc_severity}

\subsection{Matrix geometry across conditions}
\assetfigure{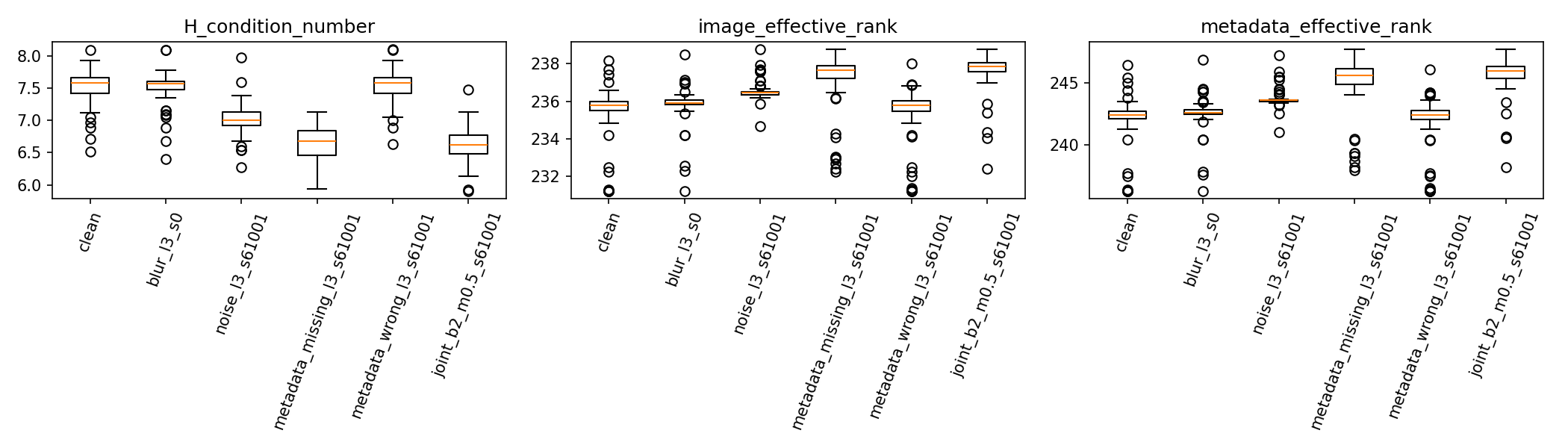}{The same 64 selected test patients are compared across the six recorded conditions; summaries describe internal geometry.}{fig:test_matrix_geometry}

\clearpage\subsection{Global DR metadata attribution}
\assetfigure{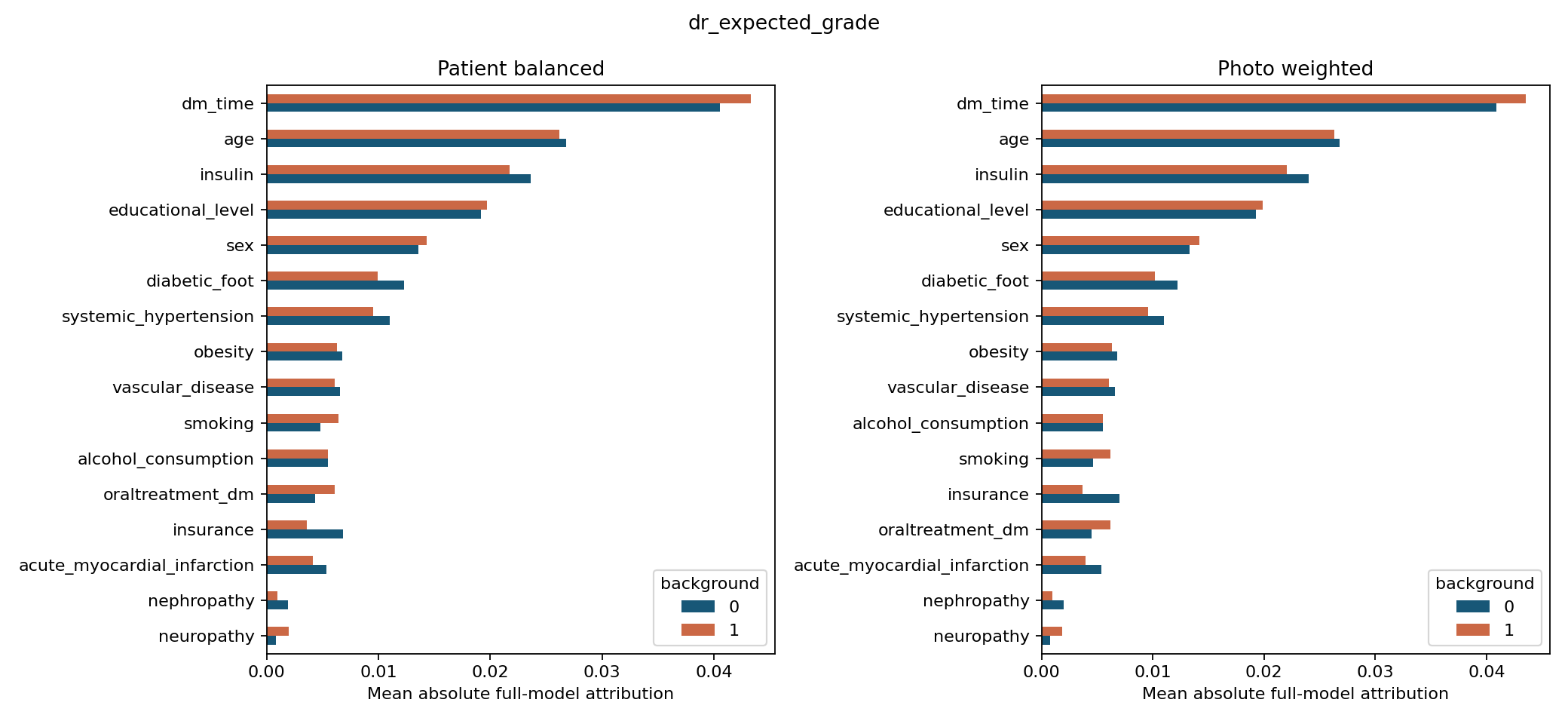}{Metadata permutation attribution for all 743 test photographs, with two disjoint 64-patient training backgrounds. Global absolute rankings describe model use under donor replacement, not causal effects.}{fig:dr_expected_grade_metadata}

\subsection{Global DME metadata attribution}
\assetfigure{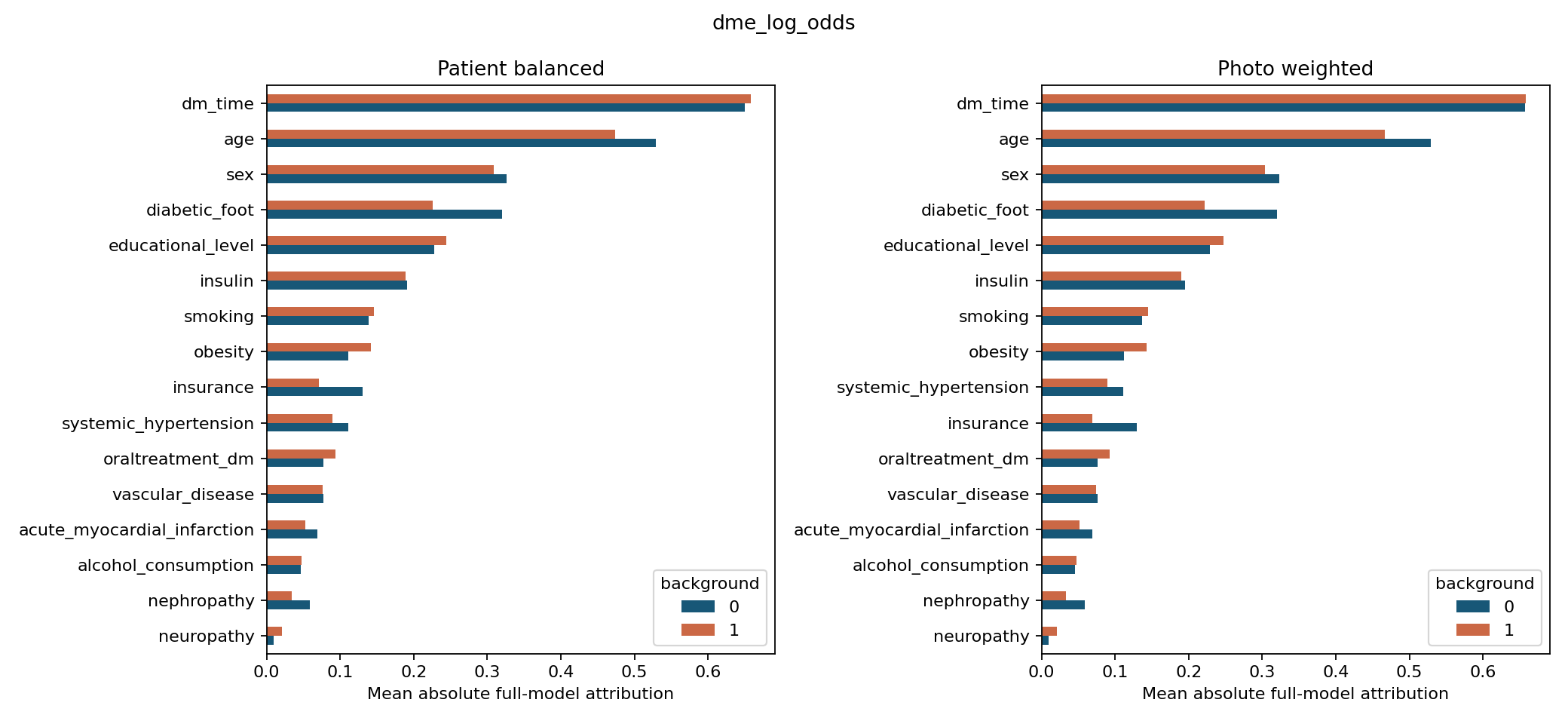}{Metadata permutation attribution for all 743 test photographs, with two disjoint 64-patient training backgrounds. Global absolute rankings describe model use under donor replacement, not causal effects.}{fig:dme_log_odds_metadata}

\clearpage\subsection{Global Ref metadata attribution}
\assetfigure{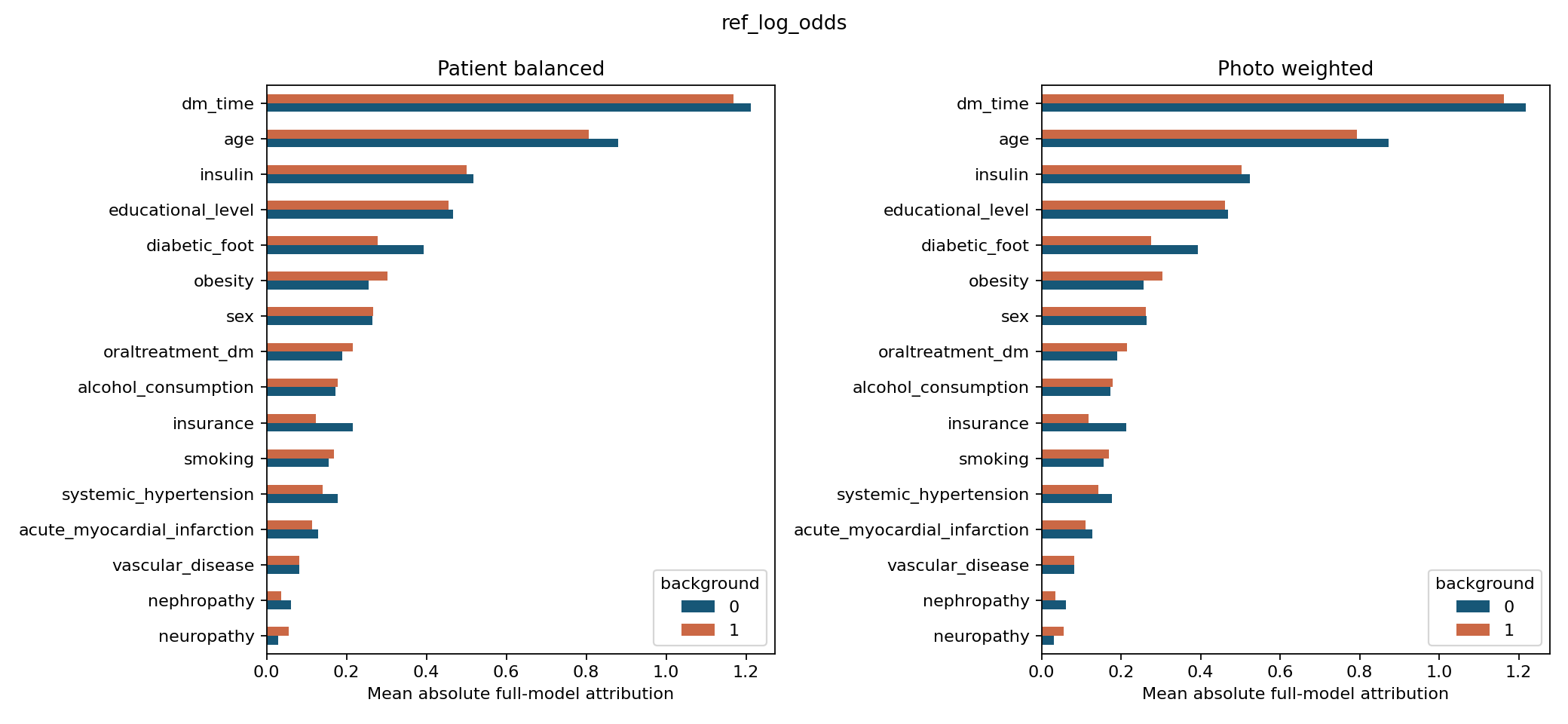}{Metadata permutation attribution for all 743 test photographs, with two disjoint 64-patient training backgrounds. Global absolute rankings describe model use under donor replacement, not causal effects.}{fig:ref_log_odds_metadata}

\subsection{Image explanation configuration sensitivity}
\assetfigure{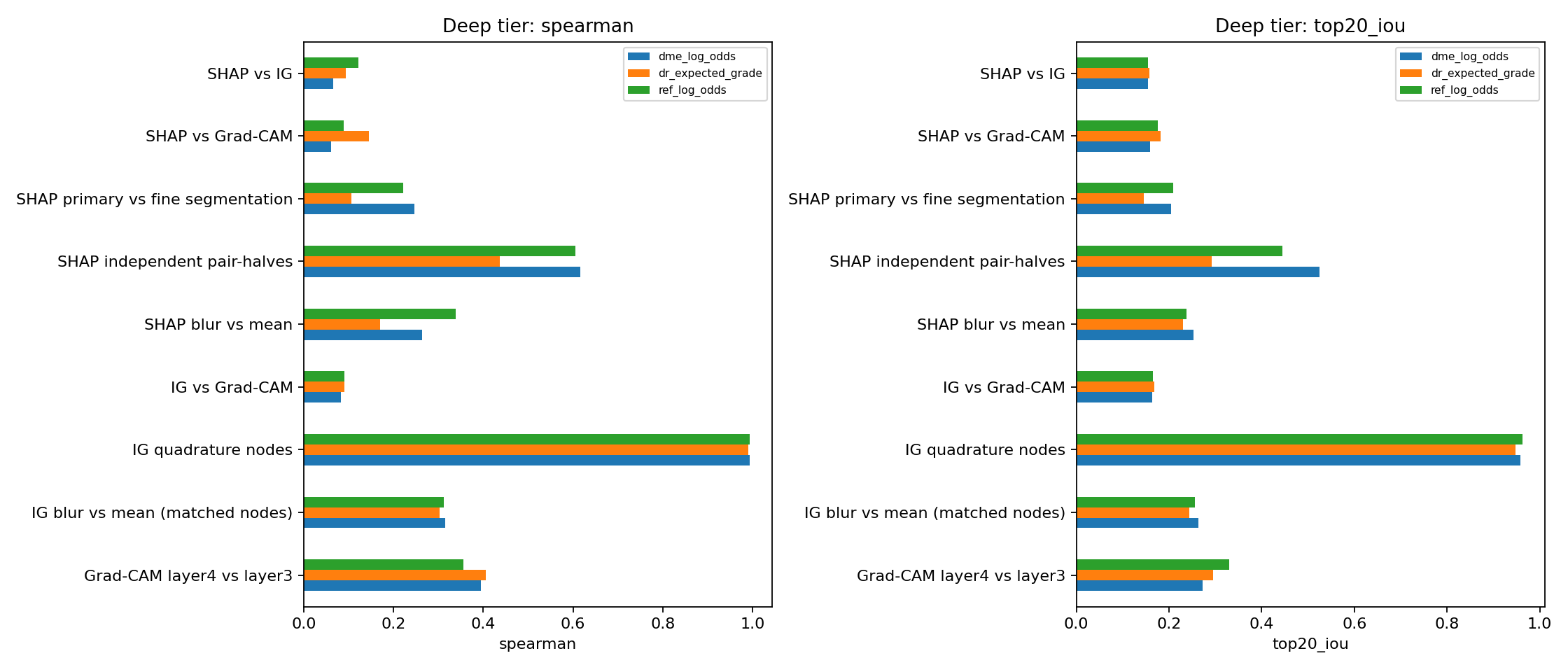}{Recorded image-explanation sensitivity or deletion/insertion results. Methods, replacement baselines and output targets remain distinct; Appendix~\ref{app:attribution} specifies the finite budgets and numerical limitations.}{fig:image_sensitivity_summary}

\clearpage\subsection{Deep-case deletion and insertion}
\assetfigurelarge{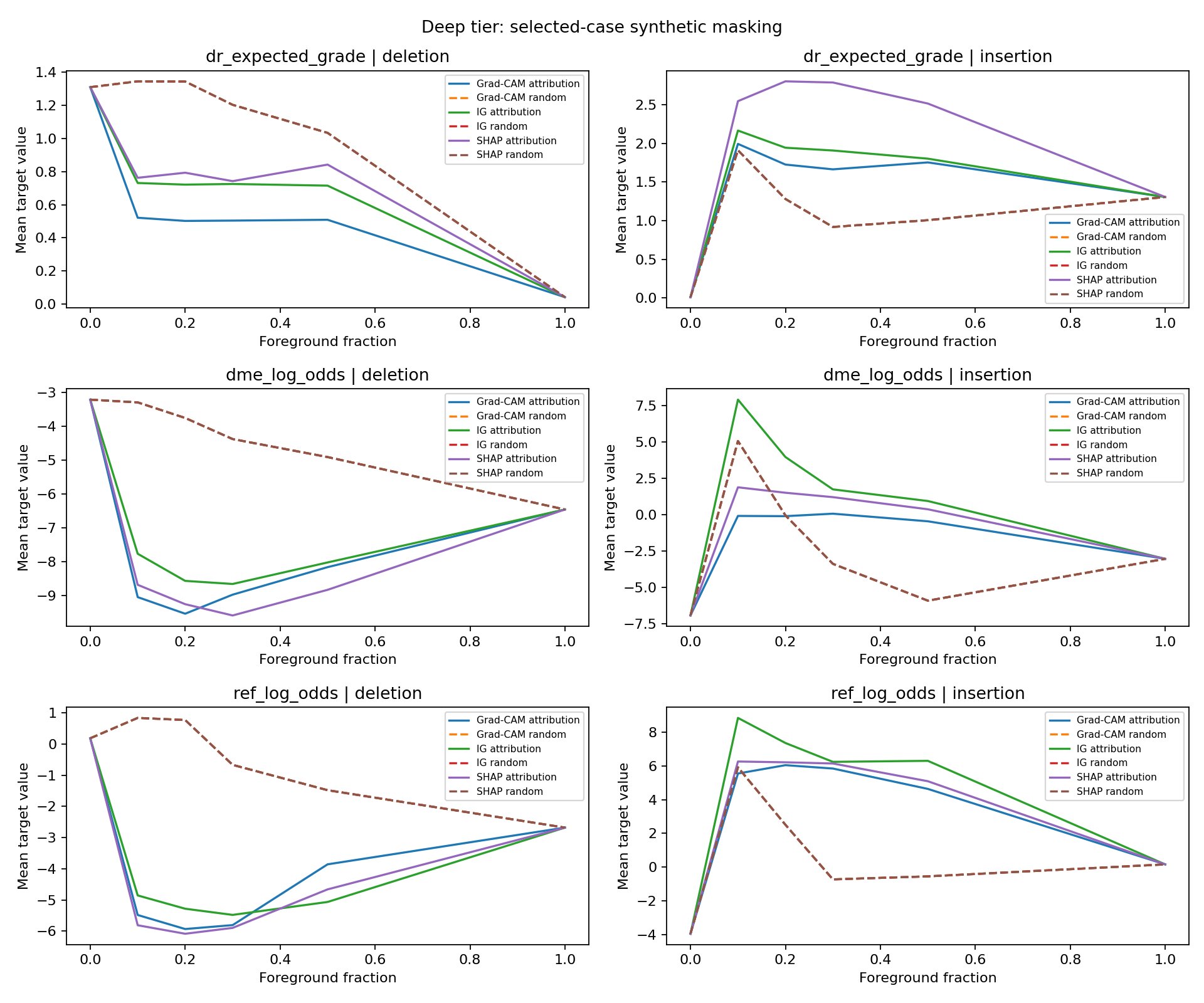}{Recorded image-explanation sensitivity or deletion/insertion results. Methods, replacement baselines and output targets remain distinct; Appendix~\ref{app:attribution} specifies the finite budgets and numerical limitations.}{fig:deep_faithfulness}

\clearpage\subsection{Broad-case deletion and insertion}
\assetfigurelarge{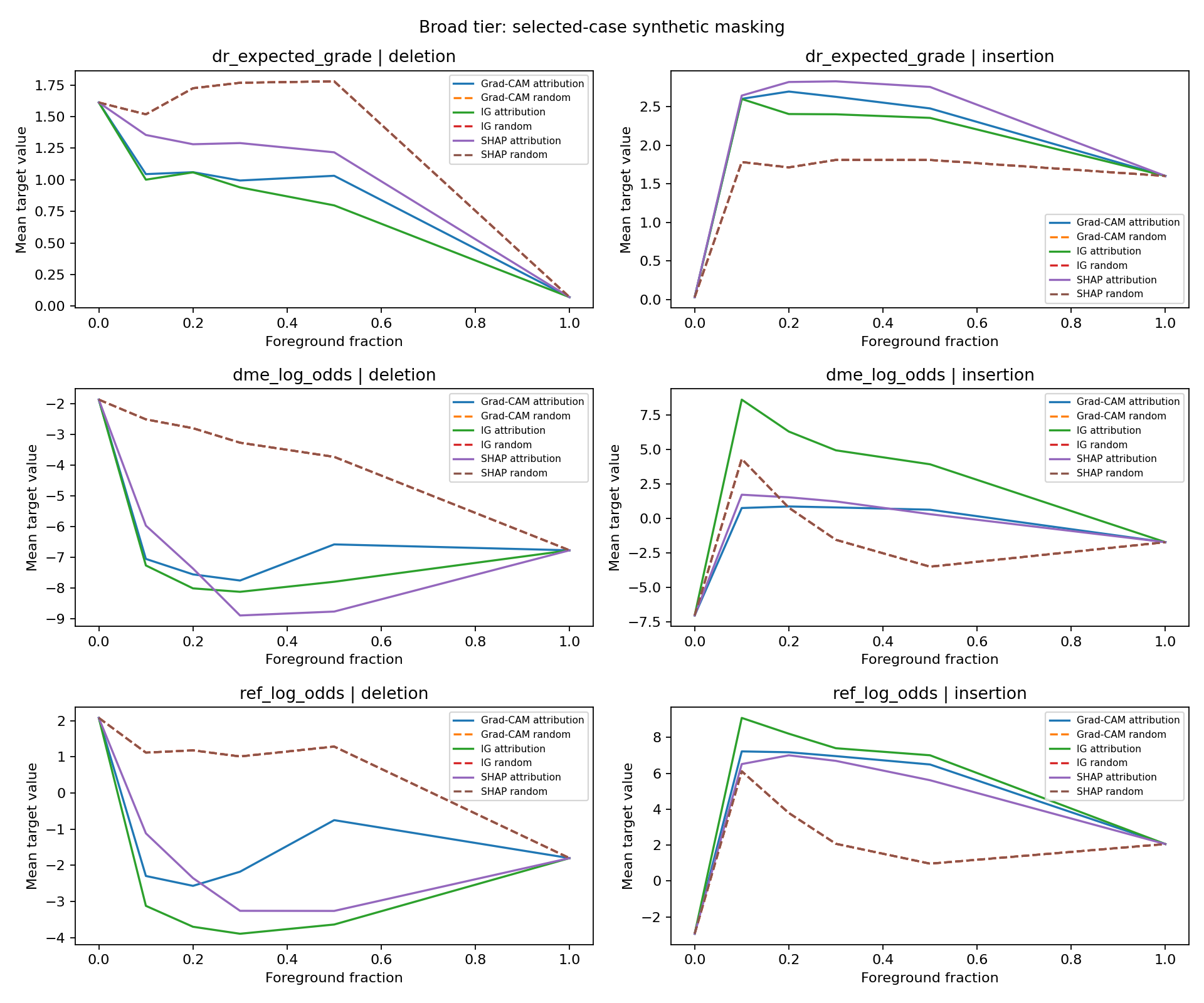}{Recorded image-explanation sensitivity or deletion/insertion results. Methods, replacement baselines and output targets remain distinct; Appendix~\ref{app:attribution} specifies the finite budgets and numerical limitations.}{fig:broad_faithfulness}

\clearpage\section{Complete Model and Ablation Result Cards}\label{app:cards}
Each card uses the same frozen partition definitions. AUC/AP are threshold-free; the two binary operating rules are validation-selected Youden (Y) and fixed 0.5 (F). Curves show raw recorded supervised histories, not smoothed or reconstructed training. These 20 detailed configurations are not independent replications of one model.

\subsection{RiVaT-Fuse}
The complete model combines both trust matrices, three interaction components, the exact latent solve and quadratic heads. This card is the single frozen reference used throughout the paper.

Image size: 512; training/evaluation batch: 16/32; actual patience: 10; selected epoch: 13; final recorded epoch: 23. Training seed is 2028.
\begin{center}\footnotesize
\setlength{\tabcolsep}{3pt}
\begin{tabular}{lrrrrr}
\toprule
DR partition & n & Accuracy & QWK & Macro-F1 & ECE10 \\
\midrule
Train & 3410 & 0.9519 & 0.9610 & 0.8533 & 0.0128 \\
Validation & 732 & 0.8593 & 0.7677 & 0.5108 & 0.0973 \\
Test & 742 & 0.8693 & 0.8472 & 0.5737 & 0.0809 \\
\bottomrule\end{tabular}\end{center}
\noindent\textbf{DME.}\begin{center}\footnotesize
\setlength{\tabcolsep}{3pt}
\begin{tabular}{lrrrrrrrr}
\toprule
Partition/rule & Thr. & Acc. & AUC & AP & Sens. & Spec. & Brier & ECE \\
\midrule
Train Y & 0.568 & 0.961 & 0.998 & 0.981 & 0.994 & 0.958 & 0.032 & 0.042 \\
Train F & 0.500 & 0.958 & 0.998 & 0.981 & 0.994 & 0.954 & 0.032 & 0.042 \\
Validation Y & 0.568 & 0.952 & 0.958 & 0.752 & 0.884 & 0.957 & 0.047 & 0.051 \\
Validation F & 0.500 & 0.950 & 0.958 & 0.752 & 0.884 & 0.954 & 0.047 & 0.051 \\
Test Y & 0.568 & 0.958 & 0.963 & 0.818 & 0.889 & 0.964 & 0.040 & 0.047 \\
Test F & 0.500 & 0.960 & 0.963 & 0.818 & 0.907 & 0.964 & 0.040 & 0.047 \\
\bottomrule\end{tabular}\end{center}
\noindent\textbf{REF.}\begin{center}\footnotesize
\setlength{\tabcolsep}{3pt}
\begin{tabular}{lrrrrrrrr}
\toprule
Partition/rule & Thr. & Acc. & AUC & AP & Sens. & Spec. & Brier & ECE \\
\midrule
Train Y & 0.009 & 0.929 & 0.999 & 0.997 & 0.998 & 0.912 & 0.007 & 0.007 \\
Train F & 0.500 & 0.990 & 0.999 & 0.997 & 0.987 & 0.991 & 0.007 & 0.007 \\
Validation Y & 0.009 & 0.866 & 0.907 & 0.829 & 0.844 & 0.871 & 0.066 & 0.062 \\
Validation F & 0.500 & 0.926 & 0.907 & 0.829 & 0.689 & 0.974 & 0.066 & 0.062 \\
Test Y & 0.009 & 0.865 & 0.931 & 0.889 & 0.871 & 0.864 & 0.052 & 0.048 \\
Test F & 0.500 & 0.941 & 0.931 & 0.889 & 0.800 & 0.973 & 0.052 & 0.048 \\
\bottomrule\end{tabular}\end{center}
\begin{center}\includegraphics[width=\linewidth]{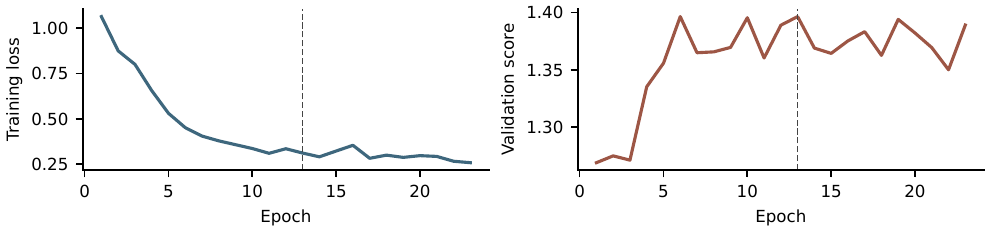}\end{center}
The dashed line marks the selected epoch. Training and validation values describe their respective partitions; only test values estimate held-out performance. Changes to the scheduled loss can alter its scale across epochs.

\clearpage\subsection{HEALNet-R50}
The complete external HEALNet-R50 adaptation uses shared latent attention for two heterogeneous input representations. Original survival heads are replaced by the shared three classification tasks.

Image size: 512; training/evaluation batch: 16/32; actual patience: 10; selected epoch: 11; final recorded epoch: 21. Training seed is 2028.
\begin{center}\footnotesize
\setlength{\tabcolsep}{3pt}
\begin{tabular}{lrrrrr}
\toprule
DR partition & n & Accuracy & QWK & Macro-F1 & ECE10 \\
\midrule
Train & 3410 & 0.7938 & 0.7594 & 0.4278 & 0.1640 \\
Validation & 732 & 0.7923 & 0.7114 & 0.4053 & 0.1726 \\
Test & 742 & 0.7978 & 0.7791 & 0.3971 & 0.1673 \\
\bottomrule\end{tabular}\end{center}
\noindent\textbf{DME.}\begin{center}\footnotesize
\setlength{\tabcolsep}{3pt}
\begin{tabular}{lrrrrrrrr}
\toprule
Partition/rule & Thr. & Acc. & AUC & AP & Sens. & Spec. & Brier & ECE \\
\midrule
Train Y & 0.236 & 0.868 & 0.975 & 0.827 & 0.976 & 0.857 & 0.064 & 0.119 \\
Train F & 0.500 & 0.916 & 0.975 & 0.827 & 0.936 & 0.914 & 0.064 & 0.119 \\
Validation Y & 0.236 & 0.839 & 0.943 & 0.599 & 0.953 & 0.832 & 0.079 & 0.136 \\
Validation F & 0.500 & 0.894 & 0.943 & 0.599 & 0.860 & 0.896 & 0.079 & 0.136 \\
Test Y & 0.236 & 0.845 & 0.972 & 0.720 & 0.963 & 0.836 & 0.066 & 0.127 \\
Test F & 0.500 & 0.919 & 0.972 & 0.720 & 0.907 & 0.920 & 0.066 & 0.127 \\
\bottomrule\end{tabular}\end{center}
\noindent\textbf{REF.}\begin{center}\footnotesize
\setlength{\tabcolsep}{3pt}
\begin{tabular}{lrrrrrrrr}
\toprule
Partition/rule & Thr. & Acc. & AUC & AP & Sens. & Spec. & Brier & ECE \\
\midrule
Train Y & 0.599 & 0.923 & 0.959 & 0.902 & 0.873 & 0.935 & 0.094 & 0.185 \\
Train F & 0.500 & 0.895 & 0.959 & 0.902 & 0.906 & 0.893 & 0.094 & 0.185 \\
Validation Y & 0.599 & 0.885 & 0.888 & 0.781 & 0.779 & 0.907 & 0.119 & 0.184 \\
Validation F & 0.500 & 0.859 & 0.888 & 0.781 & 0.779 & 0.876 & 0.119 & 0.184 \\
Test Y & 0.599 & 0.906 & 0.926 & 0.861 & 0.836 & 0.922 & 0.103 & 0.177 \\
Test F & 0.500 & 0.877 & 0.926 & 0.861 & 0.857 & 0.882 & 0.103 & 0.177 \\
\bottomrule\end{tabular}\end{center}
\begin{center}\includegraphics[width=\linewidth]{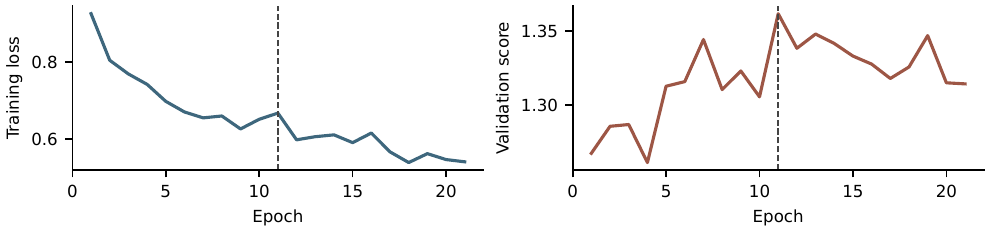}\end{center}
The dashed line marks the selected epoch. Training and validation values describe their respective partitions; only test values estimate held-out performance. Changes to the scheduled loss can alter its scale across epochs.

\clearpage\subsection{MADDi-R50}
The complete external MADDi-R50 adaptation retains the official one-token modality attention pathway, with a ResNet50 visual adapter and two available modalities. The resulting attention is not spatial patch selection.

Image size: 512; training/evaluation batch: 16/32; actual patience: 10; selected epoch: 9; final recorded epoch: 19. Training seed is 2028.
\begin{center}\footnotesize
\setlength{\tabcolsep}{3pt}
\begin{tabular}{lrrrrr}
\toprule
DR partition & n & Accuracy & QWK & Macro-F1 & ECE10 \\
\midrule
Train & 3410 & 0.9226 & 0.9241 & 0.7606 & 0.0488 \\
Validation & 732 & 0.8702 & 0.7752 & 0.6372 & 0.0194 \\
Test & 742 & 0.8423 & 0.8117 & 0.5118 & 0.0325 \\
\bottomrule\end{tabular}\end{center}
\noindent\textbf{DME.}\begin{center}\footnotesize
\setlength{\tabcolsep}{3pt}
\begin{tabular}{lrrrrrrrr}
\toprule
Partition/rule & Thr. & Acc. & AUC & AP & Sens. & Spec. & Brier & ECE \\
\midrule
Train Y & 0.166 & 0.929 & 0.992 & 0.930 & 0.982 & 0.924 & 0.028 & 0.033 \\
Train F & 0.500 & 0.957 & 0.992 & 0.930 & 0.933 & 0.960 & 0.028 & 0.033 \\
Validation Y & 0.166 & 0.925 & 0.967 & 0.827 & 0.930 & 0.925 & 0.033 & 0.037 \\
Validation F & 0.500 & 0.951 & 0.967 & 0.827 & 0.860 & 0.957 & 0.033 & 0.037 \\
Test Y & 0.166 & 0.918 & 0.959 & 0.798 & 0.907 & 0.919 & 0.034 & 0.025 \\
Test F & 0.500 & 0.957 & 0.959 & 0.798 & 0.741 & 0.974 & 0.034 & 0.025 \\
\bottomrule\end{tabular}\end{center}
\noindent\textbf{REF.}\begin{center}\footnotesize
\setlength{\tabcolsep}{3pt}
\begin{tabular}{lrrrrrrrr}
\toprule
Partition/rule & Thr. & Acc. & AUC & AP & Sens. & Spec. & Brier & ECE \\
\midrule
Train Y & 0.201 & 0.968 & 0.997 & 0.991 & 0.983 & 0.964 & 0.014 & 0.018 \\
Train F & 0.500 & 0.981 & 0.997 & 0.991 & 0.958 & 0.987 & 0.014 & 0.018 \\
Validation Y & 0.201 & 0.896 & 0.921 & 0.826 & 0.787 & 0.918 & 0.066 & 0.048 \\
Validation F & 0.500 & 0.922 & 0.921 & 0.826 & 0.730 & 0.961 & 0.066 & 0.048 \\
Test Y & 0.201 & 0.900 & 0.951 & 0.902 & 0.886 & 0.904 & 0.053 & 0.037 \\
Test F & 0.500 & 0.934 & 0.951 & 0.902 & 0.836 & 0.957 & 0.053 & 0.037 \\
\bottomrule\end{tabular}\end{center}
\begin{center}\includegraphics[width=\linewidth]{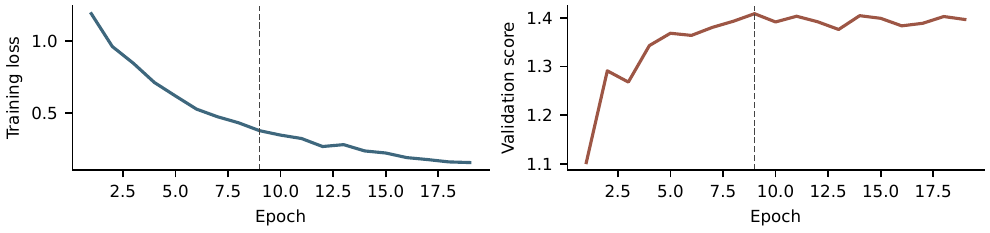}\end{center}
The dashed line marks the selected epoch. Training and validation values describe their respective partitions; only test values estimate held-out performance. Changes to the scheduled loss can alter its scale across epochs.

\clearpage\subsection{TIP-R50}
The complete external TIP-R50 adaptation includes three-objective self-supervised pretraining, then the matched supervised stage. The trajectory shown here is supervised; SSL selected epoch 5 and stopped at 20.

Image size: 512; training/evaluation batch: 16/32; actual patience: 10; selected epoch: 14; final recorded epoch: 24. Training seed is 2028.
\begin{center}\footnotesize
\setlength{\tabcolsep}{3pt}
\begin{tabular}{lrrrrr}
\toprule
DR partition & n & Accuracy & QWK & Macro-F1 & ECE10 \\
\midrule
Train & 3410 & 0.8176 & 0.7825 & 0.3576 & 0.1731 \\
Validation & 732 & 0.8210 & 0.7016 & 0.3677 & 0.1687 \\
Test & 742 & 0.8154 & 0.7198 & 0.4442 & 0.1712 \\
\bottomrule\end{tabular}\end{center}
\noindent\textbf{DME.}\begin{center}\footnotesize
\setlength{\tabcolsep}{3pt}
\begin{tabular}{lrrrrrrrr}
\toprule
Partition/rule & Thr. & Acc. & AUC & AP & Sens. & Spec. & Brier & ECE \\
\midrule
Train Y & 0.380 & 0.904 & 0.968 & 0.815 & 0.942 & 0.900 & 0.077 & 0.160 \\
Train F & 0.500 & 0.913 & 0.968 & 0.815 & 0.939 & 0.910 & 0.077 & 0.160 \\
Validation Y & 0.380 & 0.888 & 0.933 & 0.639 & 0.930 & 0.886 & 0.090 & 0.178 \\
Validation F & 0.500 & 0.898 & 0.933 & 0.639 & 0.907 & 0.897 & 0.090 & 0.178 \\
Test Y & 0.380 & 0.891 & 0.966 & 0.637 & 0.907 & 0.890 & 0.083 & 0.169 \\
Test F & 0.500 & 0.902 & 0.966 & 0.637 & 0.907 & 0.901 & 0.083 & 0.169 \\
\bottomrule\end{tabular}\end{center}
\noindent\textbf{REF.}\begin{center}\footnotesize
\setlength{\tabcolsep}{3pt}
\begin{tabular}{lrrrrrrrr}
\toprule
Partition/rule & Thr. & Acc. & AUC & AP & Sens. & Spec. & Brier & ECE \\
\midrule
Train Y & 0.655 & 0.942 & 0.939 & 0.898 & 0.812 & 0.972 & 0.111 & 0.239 \\
Train F & 0.500 & 0.936 & 0.939 & 0.898 & 0.838 & 0.959 & 0.111 & 0.239 \\
Validation Y & 0.655 & 0.920 & 0.874 & 0.746 & 0.713 & 0.961 & 0.132 & 0.245 \\
Validation F & 0.500 & 0.899 & 0.874 & 0.746 & 0.713 & 0.936 & 0.132 & 0.245 \\
Test Y & 0.655 & 0.914 & 0.904 & 0.818 & 0.686 & 0.967 & 0.125 & 0.228 \\
Test F & 0.500 & 0.912 & 0.904 & 0.818 & 0.743 & 0.952 & 0.125 & 0.228 \\
\bottomrule\end{tabular}\end{center}
\begin{center}\includegraphics[width=\linewidth]{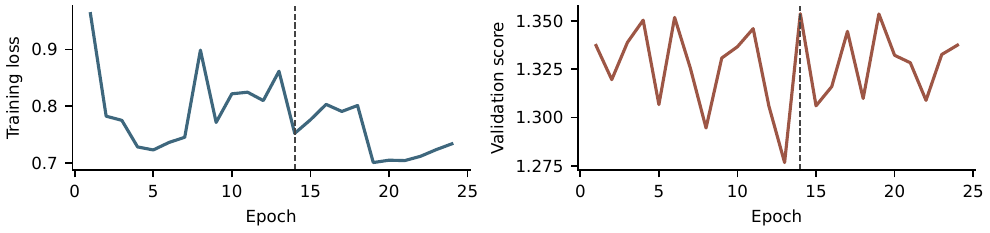}\end{center}
The dashed line marks the selected epoch. Training and validation values describe their respective partitions; only test values estimate held-out performance. Changes to the scheduled loss can alter its scale across epochs.

\clearpage\subsection{Metadata MLP}
The clean metadata-only reference uses the selected weighted historical architecture and train-fitted metadata fields. It has no image encoder. Its completed stopping patience is 20.

Image size: not applicable; training/evaluation batch: 64/256; actual patience: 20; selected epoch: 16; final recorded epoch: 36. Training seed is 2028.
\begin{center}\footnotesize
\setlength{\tabcolsep}{3pt}
\begin{tabular}{lrrrrr}
\toprule
DR partition & n & Accuracy & QWK & Macro-F1 & ECE10 \\
\midrule
Train & 3410 & 0.7384 & 0.4759 & 0.3766 & 0.1401 \\
Validation & 732 & 0.7473 & 0.4096 & 0.2860 & 0.1564 \\
Test & 742 & 0.7345 & 0.6036 & 0.3060 & 0.1464 \\
\bottomrule\end{tabular}\end{center}
\noindent\textbf{DME.}\begin{center}\footnotesize
\setlength{\tabcolsep}{3pt}
\begin{tabular}{lrrrrrrrr}
\toprule
Partition/rule & Thr. & Acc. & AUC & AP & Sens. & Spec. & Brier & ECE \\
\midrule
Train Y & 0.691 & 0.854 & 0.867 & 0.443 & 0.630 & 0.878 & 0.163 & 0.256 \\
Train F & 0.500 & 0.742 & 0.867 & 0.443 & 0.858 & 0.729 & 0.163 & 0.256 \\
Validation Y & 0.691 & 0.849 & 0.728 & 0.154 & 0.535 & 0.868 & 0.181 & 0.283 \\
Validation F & 0.500 & 0.706 & 0.728 & 0.154 & 0.535 & 0.716 & 0.181 & 0.283 \\
Test Y & 0.691 & 0.856 & 0.802 & 0.212 & 0.519 & 0.882 & 0.185 & 0.292 \\
Test F & 0.500 & 0.701 & 0.802 & 0.212 & 0.815 & 0.692 & 0.185 & 0.292 \\
\bottomrule\end{tabular}\end{center}
\noindent\textbf{REF.}\begin{center}\footnotesize
\setlength{\tabcolsep}{3pt}
\begin{tabular}{lrrrrrrrr}
\toprule
Partition/rule & Thr. & Acc. & AUC & AP & Sens. & Spec. & Brier & ECE \\
\midrule
Train Y & 0.617 & 0.823 & 0.866 & 0.596 & 0.752 & 0.839 & 0.159 & 0.201 \\
Train F & 0.500 & 0.766 & 0.866 & 0.596 & 0.830 & 0.751 & 0.159 & 0.201 \\
Validation Y & 0.617 & 0.817 & 0.754 & 0.494 & 0.590 & 0.863 & 0.179 & 0.206 \\
Validation F & 0.500 & 0.735 & 0.754 & 0.494 & 0.615 & 0.759 & 0.179 & 0.206 \\
Test Y & 0.617 & 0.794 & 0.824 & 0.532 & 0.643 & 0.829 & 0.173 & 0.203 \\
Test F & 0.500 & 0.749 & 0.824 & 0.532 & 0.779 & 0.743 & 0.173 & 0.203 \\
\bottomrule\end{tabular}\end{center}
\begin{center}\includegraphics[width=\linewidth]{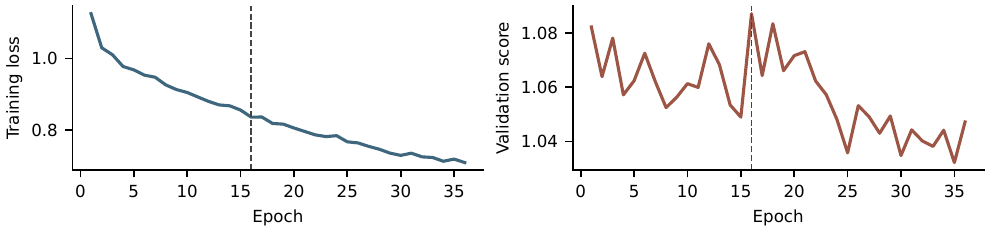}\end{center}
The dashed line marks the selected epoch. Training and validation values describe their respective partitions; only test values estimate held-out performance. Changes to the scheduled loss can alter its scale across epochs.

\clearpage\subsection{Image ResNet}
The clean image-only reference uses ResNet-50 pooled features, dropout and linear task heads. All backbone parameters train from epoch one at 224 pixels without augmentation or clipping. Its completed stopping patience is 20.

Image size: 224; training/evaluation batch: 64/256; actual patience: 20; selected epoch: 30; final recorded epoch: 50. Training seed is 2028.
\begin{center}\footnotesize
\setlength{\tabcolsep}{3pt}
\begin{tabular}{lrrrrr}
\toprule
DR partition & n & Accuracy & QWK & Macro-F1 & ECE10 \\
\midrule
Train & 3410 & 0.9997 & 0.9994 & 0.9997 & 0.0009 \\
Validation & 732 & 0.8169 & 0.6434 & 0.4801 & 0.1424 \\
Test & 742 & 0.8342 & 0.7931 & 0.4748 & 0.1269 \\
\bottomrule\end{tabular}\end{center}
\noindent\textbf{DME.}\begin{center}\footnotesize
\setlength{\tabcolsep}{3pt}
\begin{tabular}{lrrrrrrrr}
\toprule
Partition/rule & Thr. & Acc. & AUC & AP & Sens. & Spec. & Brier & ECE \\
\midrule
Train Y & 0.015 & 0.999 & 1.000 & 1.000 & 1.000 & 0.999 & 0.000 & 0.000 \\
Train F & 0.500 & 1.000 & 1.000 & 1.000 & 1.000 & 1.000 & 0.000 & 0.000 \\
Validation Y & 0.015 & 0.937 & 0.950 & 0.758 & 0.884 & 0.941 & 0.043 & 0.044 \\
Validation F & 0.500 & 0.947 & 0.950 & 0.758 & 0.744 & 0.959 & 0.043 & 0.044 \\
Test Y & 0.015 & 0.927 & 0.939 & 0.686 & 0.778 & 0.939 & 0.046 & 0.047 \\
Test F & 0.500 & 0.952 & 0.939 & 0.686 & 0.648 & 0.975 & 0.046 & 0.047 \\
\bottomrule\end{tabular}\end{center}
\noindent\textbf{REF.}\begin{center}\footnotesize
\setlength{\tabcolsep}{3pt}
\begin{tabular}{lrrrrrrrr}
\toprule
Partition/rule & Thr. & Acc. & AUC & AP & Sens. & Spec. & Brier & ECE \\
\midrule
Train Y & 0.010 & 0.993 & 1.000 & 1.000 & 1.000 & 0.992 & 0.000 & 0.001 \\
Train F & 0.500 & 1.000 & 1.000 & 1.000 & 1.000 & 1.000 & 0.000 & 0.001 \\
Validation Y & 0.010 & 0.819 & 0.848 & 0.696 & 0.746 & 0.833 & 0.103 & 0.101 \\
Validation F & 0.500 & 0.891 & 0.848 & 0.696 & 0.615 & 0.946 & 0.103 & 0.101 \\
Test Y & 0.010 & 0.840 & 0.922 & 0.859 & 0.836 & 0.841 & 0.065 & 0.061 \\
Test F & 0.500 & 0.929 & 0.922 & 0.859 & 0.764 & 0.967 & 0.065 & 0.061 \\
\bottomrule\end{tabular}\end{center}
\begin{center}\includegraphics[width=\linewidth]{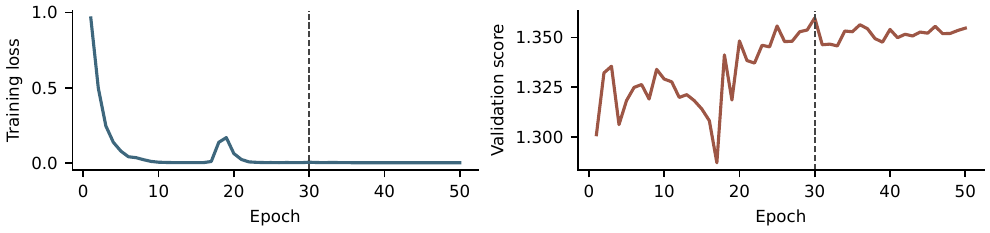}\end{center}
The dashed line marks the selected epoch. Training and validation values describe their respective partitions; only test values estimate held-out performance. Changes to the scheduled loss can alter its scale across epochs.

\clearpage\subsection{Concat Fusion}
The clean concatenation reference joins image and metadata features before a learned prediction pathway. It shares the clean fusion recipe and the frozen targets and preprocessing.

Image size: 224; training/evaluation batch: 64/256; actual patience: 10; selected epoch: 18; final recorded epoch: 28. Training seed is 2028.
\begin{center}\footnotesize
\setlength{\tabcolsep}{3pt}
\begin{tabular}{lrrrrr}
\toprule
DR partition & n & Accuracy & QWK & Macro-F1 & ECE10 \\
\midrule
Train & 3410 & 0.9985 & 0.9972 & 0.9986 & 0.0026 \\
Validation & 732 & 0.8183 & 0.6740 & 0.4554 & 0.1320 \\
Test & 742 & 0.8275 & 0.7870 & 0.4422 & 0.1200 \\
\bottomrule\end{tabular}\end{center}
\noindent\textbf{DME.}\begin{center}\footnotesize
\setlength{\tabcolsep}{3pt}
\begin{tabular}{lrrrrrrrr}
\toprule
Partition/rule & Thr. & Acc. & AUC & AP & Sens. & Spec. & Brier & ECE \\
\midrule
Train Y & 0.028 & 0.990 & 1.000 & 1.000 & 1.000 & 0.989 & 0.001 & 0.002 \\
Train F & 0.500 & 0.998 & 1.000 & 1.000 & 1.000 & 0.998 & 0.001 & 0.002 \\
Validation Y & 0.028 & 0.890 & 0.940 & 0.727 & 0.930 & 0.887 & 0.067 & 0.067 \\
Validation F & 0.500 & 0.922 & 0.940 & 0.727 & 0.698 & 0.936 & 0.067 & 0.067 \\
Test Y & 0.028 & 0.900 & 0.939 & 0.759 & 0.870 & 0.903 & 0.049 & 0.049 \\
Test F & 0.500 & 0.943 & 0.939 & 0.759 & 0.815 & 0.954 & 0.049 & 0.049 \\
\bottomrule\end{tabular}\end{center}
\noindent\textbf{REF.}\begin{center}\footnotesize
\setlength{\tabcolsep}{3pt}
\begin{tabular}{lrrrrrrrr}
\toprule
Partition/rule & Thr. & Acc. & AUC & AP & Sens. & Spec. & Brier & ECE \\
\midrule
Train Y & 0.032 & 0.994 & 1.000 & 1.000 & 1.000 & 0.993 & 0.000 & 0.001 \\
Train F & 0.500 & 1.000 & 1.000 & 1.000 & 1.000 & 1.000 & 0.000 & 0.001 \\
Validation Y & 0.032 & 0.847 & 0.874 & 0.749 & 0.762 & 0.864 & 0.096 & 0.093 \\
Validation F & 0.500 & 0.894 & 0.874 & 0.749 & 0.631 & 0.946 & 0.096 & 0.093 \\
Test Y & 0.032 & 0.865 & 0.924 & 0.853 & 0.843 & 0.870 & 0.076 & 0.073 \\
Test F & 0.500 & 0.907 & 0.924 & 0.853 & 0.743 & 0.945 & 0.076 & 0.073 \\
\bottomrule\end{tabular}\end{center}
\begin{center}\includegraphics[width=\linewidth]{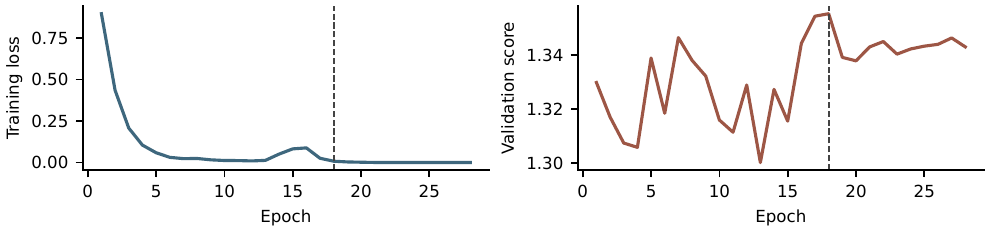}\end{center}
The dashed line marks the selected epoch. Training and validation values describe their respective partitions; only test values estimate held-out performance. Changes to the scheduled loss can alter its scale across epochs.

\clearpage\subsection{FiLM Fusion}
The clean FiLM reference conditions visual features using metadata-derived affine modulation before prediction. It is the declared feature-fusion implementation, not a replication of the original visual-reasoning benchmark.

Image size: 224; training/evaluation batch: 64/256; actual patience: 10; selected epoch: 15; final recorded epoch: 25. Training seed is 2028.
\begin{center}\footnotesize
\setlength{\tabcolsep}{3pt}
\begin{tabular}{lrrrrr}
\toprule
DR partition & n & Accuracy & QWK & Macro-F1 & ECE10 \\
\midrule
Train & 3410 & 0.9921 & 0.9941 & 0.9838 & 0.0040 \\
Validation & 732 & 0.7951 & 0.6965 & 0.4547 & 0.1703 \\
Test & 742 & 0.8154 & 0.7701 & 0.4353 & 0.1458 \\
\bottomrule\end{tabular}\end{center}
\noindent\textbf{DME.}\begin{center}\footnotesize
\setlength{\tabcolsep}{3pt}
\begin{tabular}{lrrrrrrrr}
\toprule
Partition/rule & Thr. & Acc. & AUC & AP & Sens. & Spec. & Brier & ECE \\
\midrule
Train Y & 0.060 & 0.988 & 1.000 & 0.999 & 0.997 & 0.987 & 0.004 & 0.005 \\
Train F & 0.500 & 0.996 & 1.000 & 0.999 & 0.997 & 0.996 & 0.004 & 0.005 \\
Validation Y & 0.060 & 0.932 & 0.921 & 0.661 & 0.767 & 0.942 & 0.052 & 0.053 \\
Validation F & 0.500 & 0.943 & 0.921 & 0.661 & 0.721 & 0.957 & 0.052 & 0.053 \\
Test Y & 0.060 & 0.922 & 0.947 & 0.663 & 0.833 & 0.929 & 0.057 & 0.057 \\
Test F & 0.500 & 0.938 & 0.947 & 0.663 & 0.778 & 0.951 & 0.057 & 0.057 \\
\bottomrule\end{tabular}\end{center}
\noindent\textbf{REF.}\begin{center}\footnotesize
\setlength{\tabcolsep}{3pt}
\begin{tabular}{lrrrrrrrr}
\toprule
Partition/rule & Thr. & Acc. & AUC & AP & Sens. & Spec. & Brier & ECE \\
\midrule
Train Y & 0.006 & 0.987 & 1.000 & 0.999 & 1.000 & 0.984 & 0.001 & 0.001 \\
Train F & 0.500 & 0.999 & 1.000 & 0.999 & 0.998 & 0.999 & 0.001 & 0.001 \\
Validation Y & 0.006 & 0.857 & 0.867 & 0.737 & 0.730 & 0.882 & 0.094 & 0.093 \\
Validation F & 0.500 & 0.898 & 0.867 & 0.737 & 0.615 & 0.954 & 0.094 & 0.093 \\
Test Y & 0.006 & 0.838 & 0.917 & 0.775 & 0.800 & 0.847 & 0.093 & 0.084 \\
Test F & 0.500 & 0.899 & 0.917 & 0.775 & 0.721 & 0.940 & 0.093 & 0.084 \\
\bottomrule\end{tabular}\end{center}
\begin{center}\includegraphics[width=\linewidth]{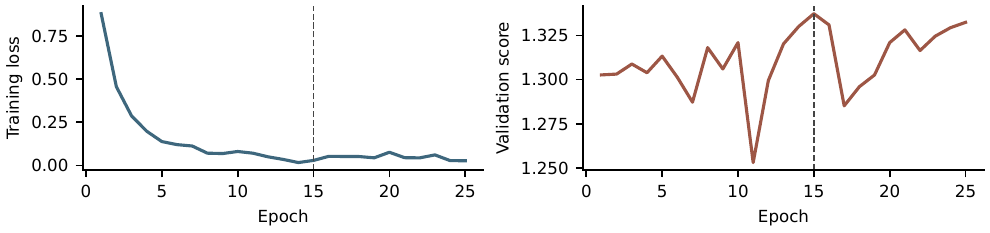}\end{center}
The dashed line marks the selected epoch. Training and validation values describe their respective partitions; only test values estimate held-out performance. Changes to the scheduled loss can alter its scale across epochs.

\clearpage\subsection{Gated Fusion}
The clean gated reference combines modality features using learned gates before task predictions. Its scores characterize the selected clean recipe; a gate alone does not provide a positive-definite latent estimation objective.

Image size: 224; training/evaluation batch: 64/256; actual patience: 10; selected epoch: 7; final recorded epoch: 17. Training seed is 2028.
\begin{center}\footnotesize
\setlength{\tabcolsep}{3pt}
\begin{tabular}{lrrrrr}
\toprule
DR partition & n & Accuracy & QWK & Macro-F1 & ECE10 \\
\midrule
Train & 3410 & 0.9680 & 0.9740 & 0.9079 & 0.0353 \\
Validation & 732 & 0.7883 & 0.6789 & 0.4943 & 0.0981 \\
Test & 742 & 0.8100 & 0.8033 & 0.5209 & 0.0870 \\
\bottomrule\end{tabular}\end{center}
\noindent\textbf{DME.}\begin{center}\footnotesize
\setlength{\tabcolsep}{3pt}
\begin{tabular}{lrrrrrrrr}
\toprule
Partition/rule & Thr. & Acc. & AUC & AP & Sens. & Spec. & Brier & ECE \\
\midrule
Train Y & 0.310 & 0.929 & 0.993 & 0.914 & 0.994 & 0.922 & 0.033 & 0.045 \\
Train F & 0.500 & 0.953 & 0.993 & 0.914 & 0.988 & 0.950 & 0.033 & 0.045 \\
Validation Y & 0.310 & 0.913 & 0.923 & 0.652 & 0.860 & 0.916 & 0.050 & 0.061 \\
Validation F & 0.500 & 0.937 & 0.923 & 0.652 & 0.791 & 0.946 & 0.050 & 0.061 \\
Test Y & 0.310 & 0.906 & 0.960 & 0.732 & 0.889 & 0.907 & 0.048 & 0.052 \\
Test F & 0.500 & 0.935 & 0.960 & 0.732 & 0.833 & 0.943 & 0.048 & 0.052 \\
\bottomrule\end{tabular}\end{center}
\noindent\textbf{REF.}\begin{center}\footnotesize
\setlength{\tabcolsep}{3pt}
\begin{tabular}{lrrrrrrrr}
\toprule
Partition/rule & Thr. & Acc. & AUC & AP & Sens. & Spec. & Brier & ECE \\
\midrule
Train Y & 0.012 & 0.953 & 1.000 & 0.999 & 1.000 & 0.942 & 0.006 & 0.009 \\
Train F & 0.500 & 0.993 & 1.000 & 0.999 & 0.998 & 0.992 & 0.006 & 0.009 \\
Validation Y & 0.012 & 0.820 & 0.866 & 0.732 & 0.803 & 0.823 & 0.104 & 0.100 \\
Validation F & 0.500 & 0.876 & 0.866 & 0.732 & 0.656 & 0.920 & 0.104 & 0.100 \\
Test Y & 0.012 & 0.875 & 0.946 & 0.864 & 0.879 & 0.874 & 0.079 & 0.074 \\
Test F & 0.500 & 0.908 & 0.946 & 0.864 & 0.743 & 0.947 & 0.079 & 0.074 \\
\bottomrule\end{tabular}\end{center}
\begin{center}\includegraphics[width=\linewidth]{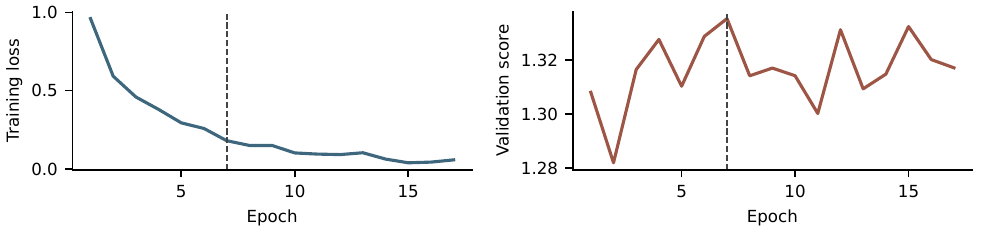}\end{center}
The dashed line marks the selected epoch. Training and validation values describe their respective partitions; only test values estimate held-out performance. Changes to the scheduled loss can alter its scale across epochs.

\clearpage\subsection{Cross-/Co-Attention}
The clean cross-/co-attention reference exchanges information at the representation level before task prediction. Its input resolution and optimization group differ from the matched512 comparison.

Image size: 224; training/evaluation batch: 64/256; actual patience: 10; selected epoch: 15; final recorded epoch: 25. Training seed is 2028.
\begin{center}\footnotesize
\setlength{\tabcolsep}{3pt}
\begin{tabular}{lrrrrr}
\toprule
DR partition & n & Accuracy & QWK & Macro-F1 & ECE10 \\
\midrule
Train & 3410 & 0.9232 & 0.9410 & 0.7372 & 0.0355 \\
Validation & 732 & 0.8005 & 0.6798 & 0.4898 & 0.0582 \\
Test & 742 & 0.8275 & 0.7901 & 0.4752 & 0.0341 \\
\bottomrule\end{tabular}\end{center}
\noindent\textbf{DME.}\begin{center}\footnotesize
\setlength{\tabcolsep}{3pt}
\begin{tabular}{lrrrrrrrr}
\toprule
Partition/rule & Thr. & Acc. & AUC & AP & Sens. & Spec. & Brier & ECE \\
\midrule
Train Y & 0.035 & 0.892 & 0.992 & 0.923 & 0.997 & 0.880 & 0.025 & 0.029 \\
Train F & 0.500 & 0.966 & 0.992 & 0.923 & 0.961 & 0.967 & 0.025 & 0.029 \\
Validation Y & 0.035 & 0.843 & 0.911 & 0.692 & 0.907 & 0.839 & 0.046 & 0.046 \\
Validation F & 0.500 & 0.937 & 0.911 & 0.692 & 0.698 & 0.952 & 0.046 & 0.046 \\
Test Y & 0.035 & 0.855 & 0.962 & 0.740 & 0.926 & 0.849 & 0.036 & 0.030 \\
Test F & 0.500 & 0.956 & 0.962 & 0.740 & 0.759 & 0.971 & 0.036 & 0.030 \\
\bottomrule\end{tabular}\end{center}
\noindent\textbf{REF.}\begin{center}\footnotesize
\setlength{\tabcolsep}{3pt}
\begin{tabular}{lrrrrrrrr}
\toprule
Partition/rule & Thr. & Acc. & AUC & AP & Sens. & Spec. & Brier & ECE \\
\midrule
Train Y & 0.331 & 0.982 & 0.999 & 0.993 & 0.992 & 0.980 & 0.011 & 0.018 \\
Train F & 0.500 & 0.987 & 0.999 & 0.993 & 0.986 & 0.987 & 0.011 & 0.018 \\
Validation Y & 0.331 & 0.865 & 0.850 & 0.700 & 0.680 & 0.902 & 0.110 & 0.100 \\
Validation F & 0.500 & 0.870 & 0.850 & 0.700 & 0.648 & 0.915 & 0.110 & 0.100 \\
Test Y & 0.331 & 0.904 & 0.929 & 0.845 & 0.793 & 0.930 & 0.075 & 0.058 \\
Test F & 0.500 & 0.907 & 0.929 & 0.845 & 0.750 & 0.944 & 0.075 & 0.058 \\
\bottomrule\end{tabular}\end{center}
\begin{center}\includegraphics[width=\linewidth]{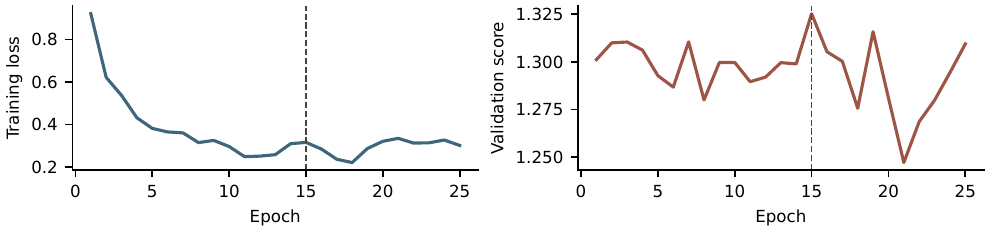}\end{center}
The dashed line marks the selected epoch. Training and validation values describe their respective partitions; only test values estimate held-out performance. Changes to the scheduled loss can alter its scale across epochs.

\clearpage\subsection{Diagonal trust}
Trust matrices are projected to their diagonal after common construction. This removes off-diagonal direction coupling while retaining the common parameter initialization and the specified training path.

Image size: 512; training/evaluation batch: 16/32; actual patience: 10; selected epoch: 12; final recorded epoch: 22. Training seed is 2028.
\begin{center}\footnotesize
\setlength{\tabcolsep}{3pt}
\begin{tabular}{lrrrrr}
\toprule
DR partition & n & Accuracy & QWK & Macro-F1 & ECE10 \\
\midrule
Train & 3410 & 0.9443 & 0.9526 & 0.8531 & 0.0486 \\
Validation & 732 & 0.8525 & 0.8306 & 0.5855 & 0.0236 \\
Test & 742 & 0.8315 & 0.8170 & 0.5280 & 0.0445 \\
\bottomrule\end{tabular}\end{center}
\noindent\textbf{DME.}\begin{center}\footnotesize
\setlength{\tabcolsep}{3pt}
\begin{tabular}{lrrrrrrrr}
\toprule
Partition/rule & Thr. & Acc. & AUC & AP & Sens. & Spec. & Brier & ECE \\
\midrule
Train Y & 0.244 & 0.957 & 0.997 & 0.966 & 0.997 & 0.952 & 0.023 & 0.031 \\
Train F & 0.500 & 0.970 & 0.997 & 0.966 & 0.988 & 0.968 & 0.023 & 0.031 \\
Validation Y & 0.244 & 0.940 & 0.962 & 0.829 & 0.930 & 0.941 & 0.044 & 0.047 \\
Validation F & 0.500 & 0.946 & 0.962 & 0.829 & 0.907 & 0.948 & 0.044 & 0.047 \\
Test Y & 0.244 & 0.933 & 0.971 & 0.845 & 0.907 & 0.935 & 0.038 & 0.038 \\
Test F & 0.500 & 0.950 & 0.971 & 0.845 & 0.833 & 0.959 & 0.038 & 0.038 \\
\bottomrule\end{tabular}\end{center}
\noindent\textbf{REF.}\begin{center}\footnotesize
\setlength{\tabcolsep}{3pt}
\begin{tabular}{lrrrrrrrr}
\toprule
Partition/rule & Thr. & Acc. & AUC & AP & Sens. & Spec. & Brier & ECE \\
\midrule
Train Y & 0.082 & 0.987 & 0.999 & 0.997 & 0.984 & 0.988 & 0.007 & 0.005 \\
Train F & 0.500 & 0.991 & 0.999 & 0.997 & 0.972 & 0.996 & 0.007 & 0.005 \\
Validation Y & 0.082 & 0.930 & 0.899 & 0.847 & 0.787 & 0.959 & 0.053 & 0.051 \\
Validation F & 0.500 & 0.943 & 0.899 & 0.847 & 0.738 & 0.984 & 0.053 & 0.051 \\
Test Y & 0.082 & 0.930 & 0.952 & 0.900 & 0.807 & 0.958 & 0.057 & 0.055 \\
Test F & 0.500 & 0.933 & 0.952 & 0.900 & 0.764 & 0.972 & 0.057 & 0.055 \\
\bottomrule\end{tabular}\end{center}
\begin{center}\includegraphics[width=\linewidth]{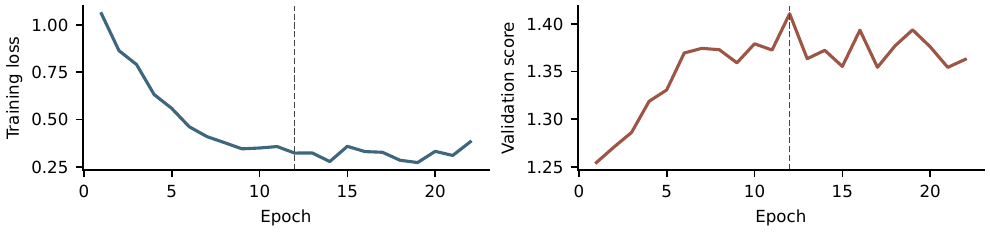}\end{center}
The dashed line marks the selected epoch. Training and validation values describe their respective partitions; only test values estimate held-out performance. Changes to the scheduled loss can alter its scale across epochs.

\clearpage\subsection{Scalar trust}
Each trust matrix is replaced by a trace-preserving scalar multiple of identity. This tests isotropic weighting within the controlled implementation rather than a separately optimized scalar-gating architecture.

Image size: 512; training/evaluation batch: 16/32; actual patience: 10; selected epoch: 24; final recorded epoch: 34. Training seed is 2028.
\begin{center}\footnotesize
\setlength{\tabcolsep}{3pt}
\begin{tabular}{lrrrrr}
\toprule
DR partition & n & Accuracy & QWK & Macro-F1 & ECE10 \\
\midrule
Train & 3410 & 0.9771 & 0.9867 & 0.9340 & 0.0053 \\
Validation & 732 & 0.8484 & 0.8231 & 0.6021 & 0.1033 \\
Test & 742 & 0.8450 & 0.8228 & 0.5319 & 0.1064 \\
\bottomrule\end{tabular}\end{center}
\noindent\textbf{DME.}\begin{center}\footnotesize
\setlength{\tabcolsep}{3pt}
\begin{tabular}{lrrrrrrrr}
\toprule
Partition/rule & Thr. & Acc. & AUC & AP & Sens. & Spec. & Brier & ECE \\
\midrule
Train Y & 0.850 & 0.990 & 1.000 & 0.996 & 1.000 & 0.989 & 0.017 & 0.025 \\
Train F & 0.500 & 0.977 & 1.000 & 0.996 & 1.000 & 0.974 & 0.017 & 0.025 \\
Validation Y & 0.850 & 0.955 & 0.952 & 0.755 & 0.884 & 0.959 & 0.050 & 0.054 \\
Validation F & 0.500 & 0.941 & 0.952 & 0.755 & 0.884 & 0.945 & 0.050 & 0.054 \\
Test Y & 0.850 & 0.953 & 0.981 & 0.812 & 0.833 & 0.962 & 0.045 & 0.046 \\
Test F & 0.500 & 0.946 & 0.981 & 0.812 & 0.889 & 0.951 & 0.045 & 0.046 \\
\bottomrule\end{tabular}\end{center}
\noindent\textbf{REF.}\begin{center}\footnotesize
\setlength{\tabcolsep}{3pt}
\begin{tabular}{lrrrrrrrr}
\toprule
Partition/rule & Thr. & Acc. & AUC & AP & Sens. & Spec. & Brier & ECE \\
\midrule
Train Y & 0.035 & 0.992 & 1.000 & 1.000 & 1.000 & 0.990 & 0.002 & 0.003 \\
Train F & 0.500 & 0.998 & 1.000 & 1.000 & 0.998 & 0.997 & 0.002 & 0.003 \\
Validation Y & 0.035 & 0.928 & 0.933 & 0.853 & 0.795 & 0.954 & 0.056 & 0.056 \\
Validation F & 0.500 & 0.941 & 0.933 & 0.853 & 0.762 & 0.977 & 0.056 & 0.056 \\
Test Y & 0.035 & 0.938 & 0.956 & 0.916 & 0.850 & 0.958 & 0.049 & 0.048 \\
Test F & 0.500 & 0.945 & 0.956 & 0.916 & 0.800 & 0.978 & 0.049 & 0.048 \\
\bottomrule\end{tabular}\end{center}
\begin{center}\includegraphics[width=\linewidth]{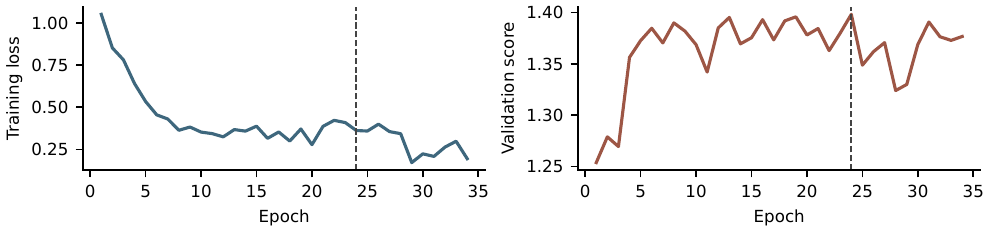}\end{center}
The dashed line marks the selected epoch. Training and validation values describe their respective partitions; only test values estimate held-out performance. Changes to the scheduled loss can alter its scale across epochs.

\clearpage\subsection{No robust loss}
The robust-loss blend is removed, retaining the other full-model components and regularizers. Differences reflect this training intervention with the shared checkpoint-selection rule.

Image size: 512; training/evaluation batch: 16/32; actual patience: 10; selected epoch: 12; final recorded epoch: 22. Training seed is 2028.
\begin{center}\footnotesize
\setlength{\tabcolsep}{3pt}
\begin{tabular}{lrrrrr}
\toprule
DR partition & n & Accuracy & QWK & Macro-F1 & ECE10 \\
\midrule
Train & 3410 & 0.9431 & 0.9616 & 0.8598 & 0.0262 \\
Validation & 732 & 0.8402 & 0.8000 & 0.5237 & 0.0730 \\
Test & 742 & 0.8437 & 0.8502 & 0.5258 & 0.0573 \\
\bottomrule\end{tabular}\end{center}
\noindent\textbf{DME.}\begin{center}\footnotesize
\setlength{\tabcolsep}{3pt}
\begin{tabular}{lrrrrrrrr}
\toprule
Partition/rule & Thr. & Acc. & AUC & AP & Sens. & Spec. & Brier & ECE \\
\midrule
Train Y & 0.170 & 0.962 & 0.997 & 0.970 & 0.997 & 0.958 & 0.024 & 0.028 \\
Train F & 0.500 & 0.969 & 0.997 & 0.970 & 0.988 & 0.967 & 0.024 & 0.028 \\
Validation Y & 0.170 & 0.941 & 0.965 & 0.807 & 0.907 & 0.944 & 0.041 & 0.044 \\
Validation F & 0.500 & 0.946 & 0.965 & 0.807 & 0.884 & 0.949 & 0.041 & 0.044 \\
Test Y & 0.170 & 0.929 & 0.980 & 0.850 & 0.889 & 0.932 & 0.054 & 0.056 \\
Test F & 0.500 & 0.937 & 0.980 & 0.850 & 0.870 & 0.942 & 0.054 & 0.056 \\
\bottomrule\end{tabular}\end{center}
\noindent\textbf{REF.}\begin{center}\footnotesize
\setlength{\tabcolsep}{3pt}
\begin{tabular}{lrrrrrrrr}
\toprule
Partition/rule & Thr. & Acc. & AUC & AP & Sens. & Spec. & Brier & ECE \\
\midrule
Train Y & 0.006 & 0.975 & 0.999 & 0.998 & 0.994 & 0.971 & 0.007 & 0.006 \\
Train F & 0.500 & 0.991 & 0.999 & 0.998 & 0.959 & 0.998 & 0.007 & 0.006 \\
Validation Y & 0.006 & 0.914 & 0.909 & 0.845 & 0.836 & 0.930 & 0.060 & 0.057 \\
Validation F & 0.500 & 0.935 & 0.909 & 0.845 & 0.705 & 0.980 & 0.060 & 0.057 \\
Test Y & 0.006 & 0.918 & 0.951 & 0.913 & 0.857 & 0.932 & 0.044 & 0.040 \\
Test F & 0.500 & 0.954 & 0.951 & 0.913 & 0.800 & 0.990 & 0.044 & 0.040 \\
\bottomrule\end{tabular}\end{center}
\begin{center}\includegraphics[width=\linewidth]{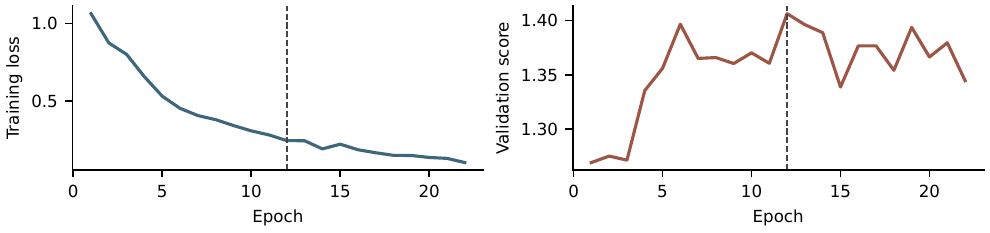}\end{center}
The dashed line marks the selected epoch. Training and validation values describe their respective partitions; only test values estimate held-out performance. Changes to the scheduled loss can alter its scale across epochs.

\clearpage\subsection{No multiplicative}
The multiplicative interaction contribution is zeroed while retaining the controlled branch execution. This tests the prediction/training contribution of feature conjunctions under the specified recipe.

Image size: 512; training/evaluation batch: 16/32; actual patience: 10; selected epoch: 12; final recorded epoch: 22. Training seed is 2028.
\begin{center}\footnotesize
\setlength{\tabcolsep}{3pt}
\begin{tabular}{lrrrrr}
\toprule
DR partition & n & Accuracy & QWK & Macro-F1 & ECE10 \\
\midrule
Train & 3410 & 0.9501 & 0.9578 & 0.8576 & 0.0235 \\
Validation & 732 & 0.8634 & 0.8270 & 0.5826 & 0.0525 \\
Test & 742 & 0.8585 & 0.8356 & 0.5383 & 0.0563 \\
\bottomrule\end{tabular}\end{center}
\noindent\textbf{DME.}\begin{center}\footnotesize
\setlength{\tabcolsep}{3pt}
\begin{tabular}{lrrrrrrrr}
\toprule
Partition/rule & Thr. & Acc. & AUC & AP & Sens. & Spec. & Brier & ECE \\
\midrule
Train Y & 0.034 & 0.945 & 0.997 & 0.968 & 0.997 & 0.939 & 0.020 & 0.024 \\
Train F & 0.500 & 0.975 & 0.997 & 0.968 & 0.991 & 0.974 & 0.020 & 0.024 \\
Validation Y & 0.034 & 0.924 & 0.960 & 0.744 & 0.907 & 0.925 & 0.043 & 0.044 \\
Validation F & 0.500 & 0.955 & 0.960 & 0.744 & 0.837 & 0.962 & 0.043 & 0.044 \\
Test Y & 0.034 & 0.919 & 0.972 & 0.836 & 0.926 & 0.919 & 0.031 & 0.028 \\
Test F & 0.500 & 0.960 & 0.972 & 0.836 & 0.778 & 0.974 & 0.031 & 0.028 \\
\bottomrule\end{tabular}\end{center}
\noindent\textbf{REF.}\begin{center}\footnotesize
\setlength{\tabcolsep}{3pt}
\begin{tabular}{lrrrrrrrr}
\toprule
Partition/rule & Thr. & Acc. & AUC & AP & Sens. & Spec. & Brier & ECE \\
\midrule
Train Y & 0.181 & 0.977 & 0.999 & 0.997 & 0.994 & 0.973 & 0.012 & 0.013 \\
Train F & 0.500 & 0.985 & 0.999 & 0.997 & 0.986 & 0.985 & 0.012 & 0.013 \\
Validation Y & 0.181 & 0.917 & 0.908 & 0.851 & 0.820 & 0.936 & 0.058 & 0.053 \\
Validation F & 0.500 & 0.936 & 0.908 & 0.851 & 0.779 & 0.967 & 0.058 & 0.053 \\
Test Y & 0.181 & 0.920 & 0.950 & 0.900 & 0.857 & 0.935 & 0.058 & 0.053 \\
Test F & 0.500 & 0.930 & 0.950 & 0.900 & 0.829 & 0.953 & 0.058 & 0.053 \\
\bottomrule\end{tabular}\end{center}
\begin{center}\includegraphics[width=\linewidth]{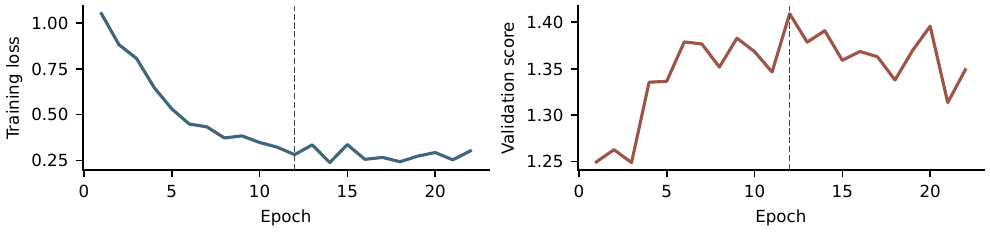}\end{center}
The dashed line marks the selected epoch. Training and validation values describe their respective partitions; only test values estimate held-out performance. Changes to the scheduled loss can alter its scale across epochs.

\clearpage\subsection{No relational}
The relational interaction contribution is zeroed. The change tests metadata-conditioned image-token evidence within the full fusion construction.

Image size: 512; training/evaluation batch: 16/32; actual patience: 10; selected epoch: 15; final recorded epoch: 25. Training seed is 2028.
\begin{center}\footnotesize
\setlength{\tabcolsep}{3pt}
\begin{tabular}{lrrrrr}
\toprule
DR partition & n & Accuracy & QWK & Macro-F1 & ECE10 \\
\midrule
Train & 3410 & 0.9660 & 0.9667 & 0.9114 & 0.0125 \\
Validation & 732 & 0.8607 & 0.7865 & 0.5451 & 0.0770 \\
Test & 742 & 0.8464 & 0.8201 & 0.5293 & 0.0972 \\
\bottomrule\end{tabular}\end{center}
\noindent\textbf{DME.}\begin{center}\footnotesize
\setlength{\tabcolsep}{3pt}
\begin{tabular}{lrrrrrrrr}
\toprule
Partition/rule & Thr. & Acc. & AUC & AP & Sens. & Spec. & Brier & ECE \\
\midrule
Train Y & 0.008 & 0.911 & 0.998 & 0.977 & 0.997 & 0.902 & 0.017 & 0.022 \\
Train F & 0.500 & 0.979 & 0.998 & 0.977 & 0.994 & 0.977 & 0.017 & 0.022 \\
Validation Y & 0.008 & 0.891 & 0.963 & 0.849 & 0.953 & 0.887 & 0.039 & 0.038 \\
Validation F & 0.500 & 0.954 & 0.963 & 0.849 & 0.837 & 0.961 & 0.039 & 0.038 \\
Test Y & 0.008 & 0.867 & 0.956 & 0.818 & 0.889 & 0.865 & 0.031 & 0.029 \\
Test F & 0.500 & 0.965 & 0.956 & 0.818 & 0.815 & 0.977 & 0.031 & 0.029 \\
\bottomrule\end{tabular}\end{center}
\noindent\textbf{REF.}\begin{center}\footnotesize
\setlength{\tabcolsep}{3pt}
\begin{tabular}{lrrrrrrrr}
\toprule
Partition/rule & Thr. & Acc. & AUC & AP & Sens. & Spec. & Brier & ECE \\
\midrule
Train Y & 0.219 & 0.969 & 1.000 & 0.999 & 0.998 & 0.962 & 0.015 & 0.023 \\
Train F & 0.500 & 0.982 & 1.000 & 0.999 & 0.997 & 0.978 & 0.015 & 0.023 \\
Validation Y & 0.219 & 0.902 & 0.915 & 0.816 & 0.787 & 0.925 & 0.078 & 0.074 \\
Validation F & 0.500 & 0.909 & 0.915 & 0.816 & 0.746 & 0.941 & 0.078 & 0.074 \\
Test Y & 0.219 & 0.916 & 0.941 & 0.897 & 0.864 & 0.929 & 0.062 & 0.057 \\
Test F & 0.500 & 0.929 & 0.941 & 0.897 & 0.829 & 0.952 & 0.062 & 0.057 \\
\bottomrule\end{tabular}\end{center}
\begin{center}\includegraphics[width=\linewidth]{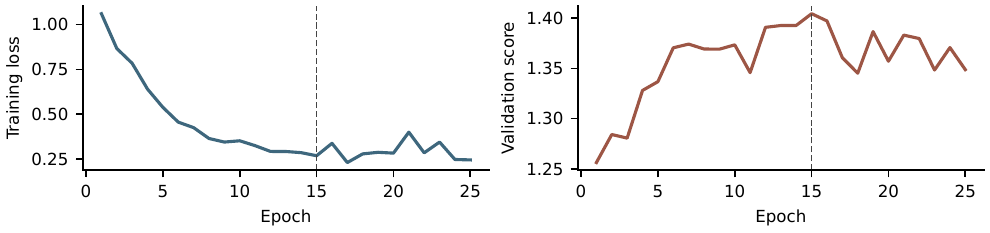}\end{center}
The dashed line marks the selected epoch. Training and validation values describe their respective partitions; only test values estimate held-out performance. Changes to the scheduled loss can alter its scale across epochs.

\clearpage\subsection{Linear heads}
Quadratic heads are replaced by linear heads. Differences concern the output decision surface jointly learned with the shared fusion representation.

Image size: 512; training/evaluation batch: 16/32; actual patience: 10; selected epoch: 12; final recorded epoch: 22. Training seed is 2028.
\begin{center}\footnotesize
\setlength{\tabcolsep}{3pt}
\begin{tabular}{lrrrrr}
\toprule
DR partition & n & Accuracy & QWK & Macro-F1 & ECE10 \\
\midrule
Train & 3410 & 0.9352 & 0.9362 & 0.8062 & 0.0149 \\
Validation & 732 & 0.8661 & 0.7957 & 0.5608 & 0.0451 \\
Test & 742 & 0.8464 & 0.8489 & 0.5254 & 0.0643 \\
\bottomrule\end{tabular}\end{center}
\noindent\textbf{DME.}\begin{center}\footnotesize
\setlength{\tabcolsep}{3pt}
\begin{tabular}{lrrrrrrrr}
\toprule
Partition/rule & Thr. & Acc. & AUC & AP & Sens. & Spec. & Brier & ECE \\
\midrule
Train Y & 0.121 & 0.938 & 0.997 & 0.970 & 1.000 & 0.932 & 0.029 & 0.039 \\
Train F & 0.500 & 0.964 & 0.997 & 0.970 & 1.000 & 0.960 & 0.029 & 0.039 \\
Validation Y & 0.121 & 0.928 & 0.966 & 0.789 & 0.930 & 0.928 & 0.049 & 0.050 \\
Validation F & 0.500 & 0.946 & 0.966 & 0.789 & 0.907 & 0.948 & 0.049 & 0.050 \\
Test Y & 0.121 & 0.908 & 0.965 & 0.840 & 0.889 & 0.910 & 0.054 & 0.059 \\
Test F & 0.500 & 0.931 & 0.965 & 0.840 & 0.870 & 0.936 & 0.054 & 0.059 \\
\bottomrule\end{tabular}\end{center}
\noindent\textbf{REF.}\begin{center}\footnotesize
\setlength{\tabcolsep}{3pt}
\begin{tabular}{lrrrrrrrr}
\toprule
Partition/rule & Thr. & Acc. & AUC & AP & Sens. & Spec. & Brier & ECE \\
\midrule
Train Y & 0.021 & 0.962 & 0.999 & 0.997 & 0.995 & 0.954 & 0.009 & 0.007 \\
Train F & 0.500 & 0.988 & 0.999 & 0.997 & 0.976 & 0.990 & 0.009 & 0.007 \\
Validation Y & 0.021 & 0.900 & 0.912 & 0.832 & 0.820 & 0.917 & 0.064 & 0.064 \\
Validation F & 0.500 & 0.930 & 0.912 & 0.832 & 0.721 & 0.972 & 0.064 & 0.064 \\
Test Y & 0.021 & 0.889 & 0.947 & 0.897 & 0.879 & 0.892 & 0.053 & 0.051 \\
Test F & 0.500 & 0.943 & 0.947 & 0.897 & 0.850 & 0.965 & 0.053 & 0.051 \\
\bottomrule\end{tabular}\end{center}
\begin{center}\includegraphics[width=\linewidth]{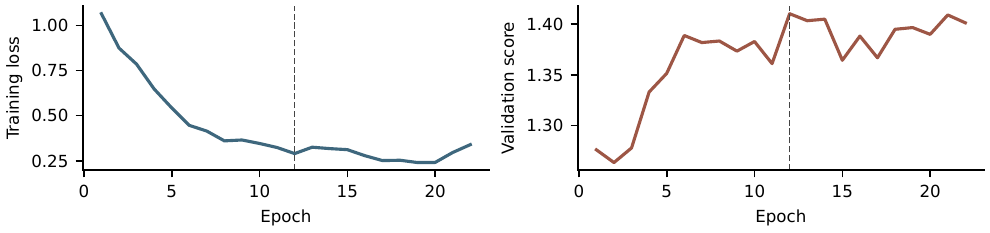}\end{center}
The dashed line marks the selected epoch. Training and validation values describe their respective partitions; only test values estimate held-out performance. Changes to the scheduled loss can alter its scale across epochs.

\clearpage\subsection{No additive}
The additive interaction contribution is zeroed while retaining controlled execution. This changes the right-hand-side evidence beyond the two trust-weighted embedding terms.

Image size: 512; training/evaluation batch: 16/32; actual patience: 10; selected epoch: 19; final recorded epoch: 29. Training seed is 2028.
\begin{center}\footnotesize
\setlength{\tabcolsep}{3pt}
\begin{tabular}{lrrrrr}
\toprule
DR partition & n & Accuracy & QWK & Macro-F1 & ECE10 \\
\midrule
Train & 3410 & 0.9698 & 0.9685 & 0.9235 & 0.0062 \\
Validation & 732 & 0.8538 & 0.7908 & 0.5358 & 0.0928 \\
Test & 742 & 0.8423 & 0.8104 & 0.5013 & 0.1037 \\
\bottomrule\end{tabular}\end{center}
\noindent\textbf{DME.}\begin{center}\footnotesize
\setlength{\tabcolsep}{3pt}
\begin{tabular}{lrrrrrrrr}
\toprule
Partition/rule & Thr. & Acc. & AUC & AP & Sens. & Spec. & Brier & ECE \\
\midrule
Train Y & 0.750 & 0.989 & 0.999 & 0.992 & 0.997 & 0.988 & 0.016 & 0.027 \\
Train F & 0.500 & 0.977 & 0.999 & 0.992 & 0.997 & 0.975 & 0.016 & 0.027 \\
Validation Y & 0.750 & 0.960 & 0.976 & 0.830 & 0.907 & 0.964 & 0.044 & 0.052 \\
Validation F & 0.500 & 0.947 & 0.976 & 0.830 & 0.907 & 0.949 & 0.044 & 0.052 \\
Test Y & 0.750 & 0.965 & 0.970 & 0.835 & 0.852 & 0.974 & 0.039 & 0.044 \\
Test F & 0.500 & 0.949 & 0.970 & 0.835 & 0.852 & 0.956 & 0.039 & 0.044 \\
\bottomrule\end{tabular}\end{center}
\noindent\textbf{REF.}\begin{center}\footnotesize
\setlength{\tabcolsep}{3pt}
\begin{tabular}{lrrrrrrrr}
\toprule
Partition/rule & Thr. & Acc. & AUC & AP & Sens. & Spec. & Brier & ECE \\
\midrule
Train Y & 0.361 & 0.983 & 1.000 & 0.999 & 0.995 & 0.981 & 0.012 & 0.014 \\
Train F & 0.500 & 0.986 & 1.000 & 0.999 & 0.995 & 0.984 & 0.012 & 0.014 \\
Validation Y & 0.361 & 0.920 & 0.935 & 0.856 & 0.844 & 0.935 & 0.073 & 0.072 \\
Validation F & 0.500 & 0.922 & 0.935 & 0.856 & 0.836 & 0.939 & 0.073 & 0.072 \\
Test Y & 0.361 & 0.919 & 0.963 & 0.917 & 0.879 & 0.929 & 0.065 & 0.066 \\
Test F & 0.500 & 0.926 & 0.963 & 0.917 & 0.879 & 0.937 & 0.065 & 0.066 \\
\bottomrule\end{tabular}\end{center}
\begin{center}\includegraphics[width=\linewidth]{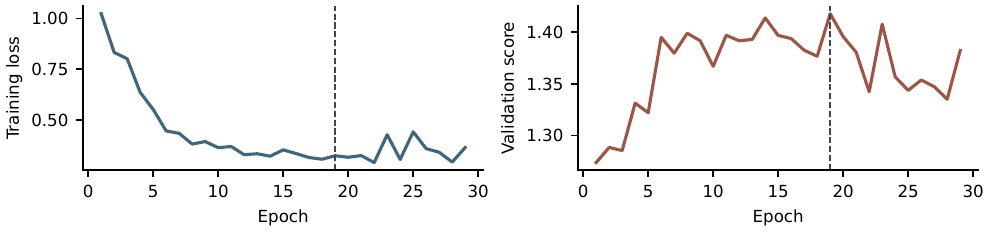}\end{center}
The dashed line marks the selected epoch. Training and validation values describe their respective partitions; only test values estimate held-out performance. Changes to the scheduled loss can alter its scale across epochs.

\clearpage\subsection{No consistency}
The consistency penalty coefficient is zeroed, with its scheduled auxiliary forward computations retained. This controls for the additional forward-path effects while changing the loss contribution.

Image size: 512; training/evaluation batch: 16/32; actual patience: 10; selected epoch: 16; final recorded epoch: 26. Training seed is 2028.
\begin{center}\footnotesize
\setlength{\tabcolsep}{3pt}
\begin{tabular}{lrrrrr}
\toprule
DR partition & n & Accuracy & QWK & Macro-F1 & ECE10 \\
\midrule
Train & 3410 & 0.9660 & 0.9715 & 0.9021 & 0.0249 \\
Validation & 732 & 0.8566 & 0.7675 & 0.5199 & 0.0685 \\
Test & 742 & 0.8369 & 0.8343 & 0.5246 & 0.0858 \\
\bottomrule\end{tabular}\end{center}
\noindent\textbf{DME.}\begin{center}\footnotesize
\setlength{\tabcolsep}{3pt}
\begin{tabular}{lrrrrrrrr}
\toprule
Partition/rule & Thr. & Acc. & AUC & AP & Sens. & Spec. & Brier & ECE \\
\midrule
Train Y & 0.010 & 0.914 & 0.999 & 0.985 & 1.000 & 0.904 & 0.013 & 0.018 \\
Train F & 0.500 & 0.980 & 0.999 & 0.985 & 0.991 & 0.979 & 0.013 & 0.018 \\
Validation Y & 0.010 & 0.887 & 0.962 & 0.644 & 0.930 & 0.884 & 0.040 & 0.040 \\
Validation F & 0.500 & 0.954 & 0.962 & 0.644 & 0.791 & 0.964 & 0.040 & 0.040 \\
Test Y & 0.010 & 0.895 & 0.967 & 0.788 & 0.926 & 0.893 & 0.038 & 0.033 \\
Test F & 0.500 & 0.956 & 0.967 & 0.788 & 0.815 & 0.967 & 0.038 & 0.033 \\
\bottomrule\end{tabular}\end{center}
\noindent\textbf{REF.}\begin{center}\footnotesize
\setlength{\tabcolsep}{3pt}
\begin{tabular}{lrrrrrrrr}
\toprule
Partition/rule & Thr. & Acc. & AUC & AP & Sens. & Spec. & Brier & ECE \\
\midrule
Train Y & 0.157 & 0.984 & 1.000 & 0.998 & 0.997 & 0.981 & 0.006 & 0.007 \\
Train F & 0.500 & 0.993 & 1.000 & 0.998 & 0.994 & 0.993 & 0.006 & 0.007 \\
Validation Y & 0.157 & 0.922 & 0.911 & 0.835 & 0.803 & 0.946 & 0.066 & 0.062 \\
Validation F & 0.500 & 0.925 & 0.911 & 0.835 & 0.754 & 0.959 & 0.066 & 0.062 \\
Test Y & 0.157 & 0.931 & 0.944 & 0.902 & 0.864 & 0.947 & 0.055 & 0.052 \\
Test F & 0.500 & 0.935 & 0.944 & 0.902 & 0.814 & 0.963 & 0.055 & 0.052 \\
\bottomrule\end{tabular}\end{center}
\begin{center}\includegraphics[width=\linewidth]{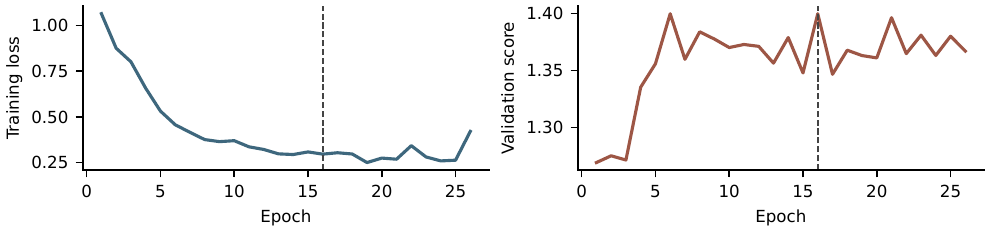}\end{center}
The dashed line marks the selected epoch. Training and validation values describe their respective partitions; only test values estimate held-out performance. Changes to the scheduled loss can alter its scale across epochs.

\clearpage\subsection{No task coupling}
The task-coupling penalty is zeroed. Shared representations, heads, task weights and missing-label masks remain in place, so this is not independent single-task training.

Image size: 512; training/evaluation batch: 16/32; actual patience: 10; selected epoch: 15; final recorded epoch: 25. Training seed is 2028.
\begin{center}\footnotesize
\setlength{\tabcolsep}{3pt}
\begin{tabular}{lrrrrr}
\toprule
DR partition & n & Accuracy & QWK & Macro-F1 & ECE10 \\
\midrule
Train & 3410 & 0.9595 & 0.9588 & 0.8789 & 0.0106 \\
Validation & 732 & 0.8730 & 0.7771 & 0.5690 & 0.0846 \\
Test & 742 & 0.8760 & 0.8422 & 0.5510 & 0.0797 \\
\bottomrule\end{tabular}\end{center}
\noindent\textbf{DME.}\begin{center}\footnotesize
\setlength{\tabcolsep}{3pt}
\begin{tabular}{lrrrrrrrr}
\toprule
Partition/rule & Thr. & Acc. & AUC & AP & Sens. & Spec. & Brier & ECE \\
\midrule
Train Y & 0.083 & 0.955 & 0.999 & 0.987 & 1.000 & 0.950 & 0.012 & 0.018 \\
Train F & 0.500 & 0.985 & 0.999 & 0.987 & 0.994 & 0.984 & 0.012 & 0.018 \\
Validation Y & 0.083 & 0.939 & 0.959 & 0.790 & 0.907 & 0.941 & 0.036 & 0.035 \\
Validation F & 0.500 & 0.956 & 0.959 & 0.790 & 0.814 & 0.965 & 0.036 & 0.035 \\
Test Y & 0.083 & 0.946 & 0.950 & 0.786 & 0.907 & 0.949 & 0.035 & 0.032 \\
Test F & 0.500 & 0.956 & 0.950 & 0.786 & 0.685 & 0.977 & 0.035 & 0.032 \\
\bottomrule\end{tabular}\end{center}
\noindent\textbf{REF.}\begin{center}\footnotesize
\setlength{\tabcolsep}{3pt}
\begin{tabular}{lrrrrrrrr}
\toprule
Partition/rule & Thr. & Acc. & AUC & AP & Sens. & Spec. & Brier & ECE \\
\midrule
Train Y & 0.098 & 0.975 & 0.999 & 0.997 & 0.994 & 0.971 & 0.009 & 0.011 \\
Train F & 0.500 & 0.989 & 0.999 & 0.997 & 0.989 & 0.989 & 0.009 & 0.011 \\
Validation Y & 0.098 & 0.915 & 0.917 & 0.837 & 0.811 & 0.936 & 0.064 & 0.060 \\
Validation F & 0.500 & 0.928 & 0.917 & 0.837 & 0.770 & 0.959 & 0.064 & 0.060 \\
Test Y & 0.098 & 0.915 & 0.937 & 0.893 & 0.850 & 0.930 & 0.052 & 0.048 \\
Test F & 0.500 & 0.945 & 0.937 & 0.893 & 0.829 & 0.972 & 0.052 & 0.048 \\
\bottomrule\end{tabular}\end{center}
\begin{center}\includegraphics[width=\linewidth]{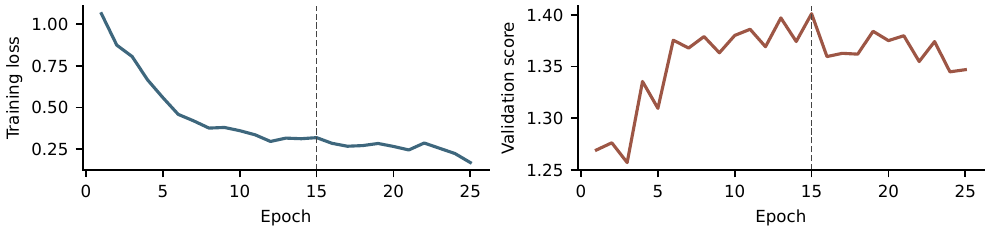}\end{center}
The dashed line marks the selected epoch. Training and validation values describe their respective partitions; only test values estimate held-out performance. Changes to the scheduled loss can alter its scale across epochs.

\clearpage\subsection{Global learned state}
A learned global three-coordinate state replaces sample-specific descriptors. It is optimized during training and is not a fixed hand-selected constant.

Image size: 512; training/evaluation batch: 16/32; actual patience: 10; selected epoch: 14; final recorded epoch: 24. Training seed is 2028.
\begin{center}\footnotesize
\setlength{\tabcolsep}{3pt}
\begin{tabular}{lrrrrr}
\toprule
DR partition & n & Accuracy & QWK & Macro-F1 & ECE10 \\
\midrule
Train & 3410 & 0.9457 & 0.9541 & 0.8549 & 0.0081 \\
Validation & 732 & 0.8511 & 0.8039 & 0.5814 & 0.0761 \\
Test & 742 & 0.8464 & 0.8507 & 0.5734 & 0.0867 \\
\bottomrule\end{tabular}\end{center}
\noindent\textbf{DME.}\begin{center}\footnotesize
\setlength{\tabcolsep}{3pt}
\begin{tabular}{lrrrrrrrr}
\toprule
Partition/rule & Thr. & Acc. & AUC & AP & Sens. & Spec. & Brier & ECE \\
\midrule
Train Y & 0.297 & 0.966 & 0.997 & 0.962 & 0.997 & 0.963 & 0.019 & 0.026 \\
Train F & 0.500 & 0.977 & 0.997 & 0.962 & 0.994 & 0.975 & 0.019 & 0.026 \\
Validation Y & 0.297 & 0.926 & 0.966 & 0.709 & 0.930 & 0.926 & 0.055 & 0.058 \\
Validation F & 0.500 & 0.935 & 0.966 & 0.709 & 0.884 & 0.938 & 0.055 & 0.058 \\
Test Y & 0.297 & 0.934 & 0.968 & 0.783 & 0.889 & 0.938 & 0.042 & 0.041 \\
Test F & 0.500 & 0.950 & 0.968 & 0.783 & 0.852 & 0.958 & 0.042 & 0.041 \\
\bottomrule\end{tabular}\end{center}
\noindent\textbf{REF.}\begin{center}\footnotesize
\setlength{\tabcolsep}{3pt}
\begin{tabular}{lrrrrrrrr}
\toprule
Partition/rule & Thr. & Acc. & AUC & AP & Sens. & Spec. & Brier & ECE \\
\midrule
Train Y & 0.491 & 0.980 & 0.999 & 0.996 & 0.987 & 0.979 & 0.014 & 0.016 \\
Train F & 0.500 & 0.980 & 0.999 & 0.996 & 0.987 & 0.979 & 0.014 & 0.016 \\
Validation Y & 0.491 & 0.925 & 0.925 & 0.855 & 0.828 & 0.944 & 0.069 & 0.066 \\
Validation F & 0.500 & 0.924 & 0.925 & 0.855 & 0.820 & 0.944 & 0.069 & 0.066 \\
Test Y & 0.491 & 0.933 & 0.956 & 0.909 & 0.871 & 0.947 & 0.060 & 0.059 \\
Test F & 0.500 & 0.931 & 0.956 & 0.909 & 0.864 & 0.947 & 0.060 & 0.059 \\
\bottomrule\end{tabular}\end{center}
\begin{center}\includegraphics[width=\linewidth]{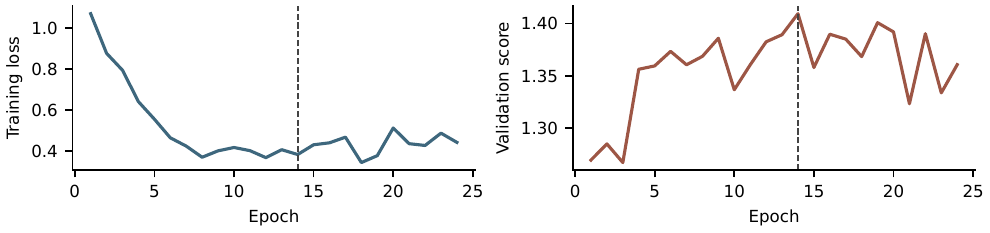}\end{center}
The dashed line marks the selected epoch. Training and validation values describe their respective partitions; only test values estimate held-out performance. Changes to the scheduled loss can alter its scale across epochs.

\clearpage\section{Complete Image Attribution Case Atlas}\label{app:caseatlas}
All 32 selected patients are included. Columns show the input, signed region-Shapley density, signed integrated gradients and layer4 Grad-CAM; rows show the three continuous explanation targets. Signed maps use a symmetric per-map 99th-percentile display limit (saturated colors beyond that limit); Grad-CAM uses its own nonnegative range. Colors therefore compare spatial patterns within a map, not absolute magnitudes across methods or tasks. The numerical table reports the primary blur-reference IG contrast and residual.

\subsection{Case 01: DR grade 0 correct}
Tier: deep; actual regions: 96; antithetic permutation pairs: 16; IG nodes: 256. Selection follows the recorded case rule, not the visual attractiveness of the map.
\begin{center}\footnotesize
\setlength{\tabcolsep}{3pt}
\begin{tabular}{lrrrr}
\toprule
Task & Observed label & Argmax class & Max probability & Explained output \\
\midrule
DR & 0 & 0 & 1.0000 & 0.0000 \\
DME & 0 & 0 & 1.0000 & -11.0856 \\
Ref & 0 & 0 & 1.0000 & -12.4492 \\
\bottomrule\end{tabular}\end{center}
\begin{center}\includegraphics[width=\linewidth]{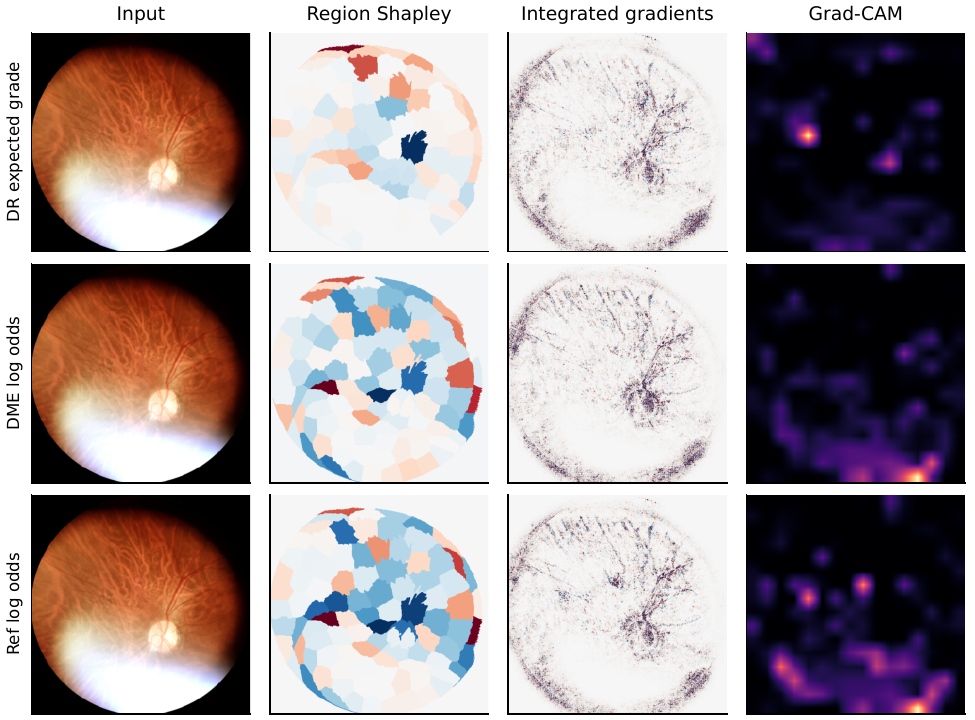}\end{center}
\begin{center}\footnotesize
\setlength{\tabcolsep}{3pt}
\begin{tabular}{lrrr}
\toprule
IG target & Output contrast & Absolute residual & Relative residual \\
\midrule
DR expected grade & -0.00026 & 2.1e-06 & 0.82\% \\
DME log odds & -3.44990 & 0.0111 & 0.32\% \\
Ref log odds & -5.31152 & 0.0177 & 0.33\% \\
\bottomrule\end{tabular}\end{center}
All three primary IG outputs have relative completeness residual at most 10\%. Signed attribution increases or decreases the named output relative to the blurred reference. A high DR expected-grade contribution is not a contribution to every DR class probability. Binary argmax classes shown here use 0.5 and can differ from validation-Youden decisions.

\clearpage\subsection{Case 02: DR grade 1 error}
Tier: deep; actual regions: 97; antithetic permutation pairs: 16; IG nodes: 256. Selection follows the recorded case rule, not the visual attractiveness of the map.
\begin{center}\footnotesize
\setlength{\tabcolsep}{3pt}
\begin{tabular}{lrrrr}
\toprule
Task & Observed label & Argmax class & Max probability & Explained output \\
\midrule
DR & 1 & 0 & 0.9996 & 0.0004 \\
DME & 0 & 0 & 1.0000 & -10.1103 \\
Ref & 0 & 0 & 0.9999 & -9.6483 \\
\bottomrule\end{tabular}\end{center}
\begin{center}\includegraphics[width=\linewidth]{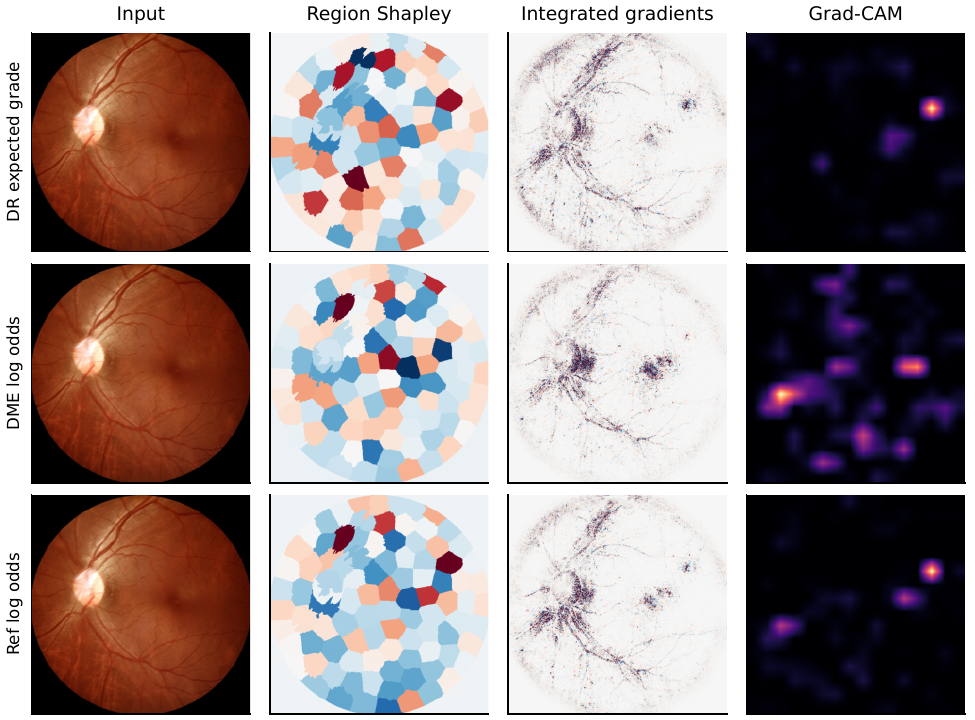}\end{center}
\begin{center}\footnotesize
\setlength{\tabcolsep}{3pt}
\begin{tabular}{lrrr}
\toprule
IG target & Output contrast & Absolute residual & Relative residual \\
\midrule
DR expected grade & -0.00081 & 1.36e-05 & 1.67\% \\
DME log odds & -3.38962 & 0.0096 & 0.28\% \\
Ref log odds & -4.47392 & 0.0153 & 0.34\% \\
\bottomrule\end{tabular}\end{center}
All three primary IG outputs have relative completeness residual at most 10\%. Signed attribution increases or decreases the named output relative to the blurred reference. A high DR expected-grade contribution is not a contribution to every DR class probability. Binary argmax classes shown here use 0.5 and can differ from validation-Youden decisions.

\clearpage\subsection{Case 03: DR grade 2 correct}
Tier: deep; actual regions: 95; antithetic permutation pairs: 16; IG nodes: 256. Selection follows the recorded case rule, not the visual attractiveness of the map.
\begin{center}\footnotesize
\setlength{\tabcolsep}{3pt}
\begin{tabular}{lrrrr}
\toprule
Task & Observed label & Argmax class & Max probability & Explained output \\
\midrule
DR & 2 & 2 & 0.9937 & 1.9934 \\
DME & 0 & 0 & 0.7212 & -0.9502 \\
Ref & 1 & 1 & 0.9997 & 8.0176 \\
\bottomrule\end{tabular}\end{center}
\begin{center}\includegraphics[width=\linewidth]{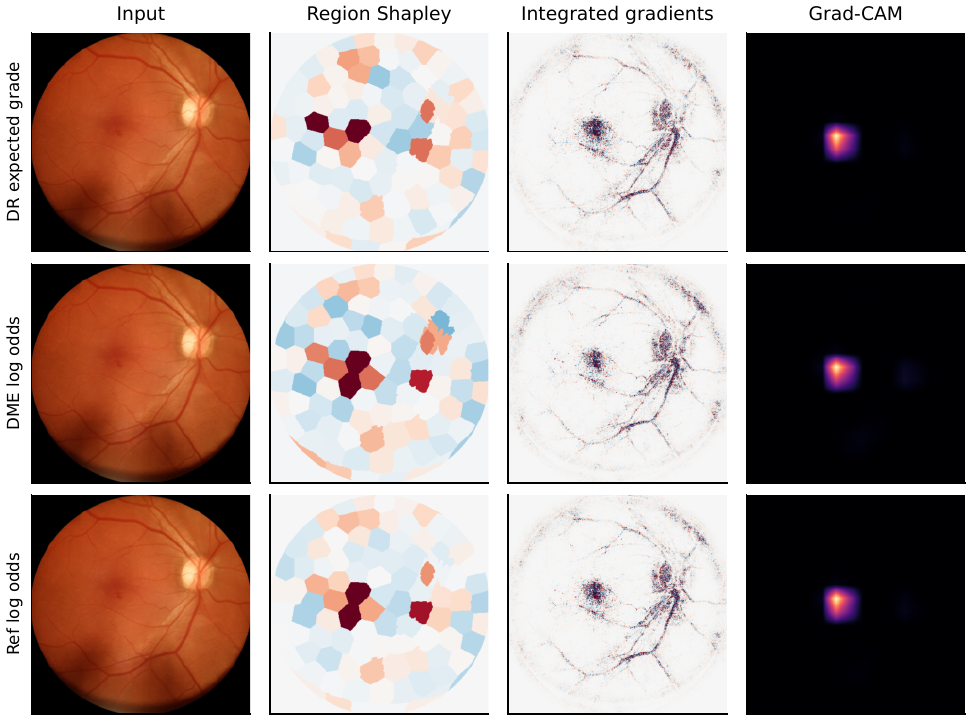}\end{center}
\begin{center}\footnotesize
\setlength{\tabcolsep}{3pt}
\begin{tabular}{lrrr}
\toprule
IG target & Output contrast & Absolute residual & Relative residual \\
\midrule
DR expected grade & 1.98217 & 0.00997 & 0.50\% \\
DME log odds & 4.62220 & 0.0313 & 0.68\% \\
Ref log odds & 10.08974 & 0.0449 & 0.44\% \\
\bottomrule\end{tabular}\end{center}
All three primary IG outputs have relative completeness residual at most 10\%. Signed attribution increases or decreases the named output relative to the blurred reference. A high DR expected-grade contribution is not a contribution to every DR class probability. Binary argmax classes shown here use 0.5 and can differ from validation-Youden decisions.

\clearpage\subsection{Case 04: DR grade 3 error}
Tier: deep; actual regions: 96; antithetic permutation pairs: 16; IG nodes: 256. Selection follows the recorded case rule, not the visual attractiveness of the map.
\begin{center}\footnotesize
\setlength{\tabcolsep}{3pt}
\begin{tabular}{lrrrr}
\toprule
Task & Observed label & Argmax class & Max probability & Explained output \\
\midrule
DR & 3 & 2 & 0.7564 & 2.2836 \\
DME & 1 & 1 & 1.0000 & 10.1204 \\
Ref & 1 & 1 & 1.0000 & 10.5284 \\
\bottomrule\end{tabular}\end{center}
\begin{center}\includegraphics[width=\linewidth]{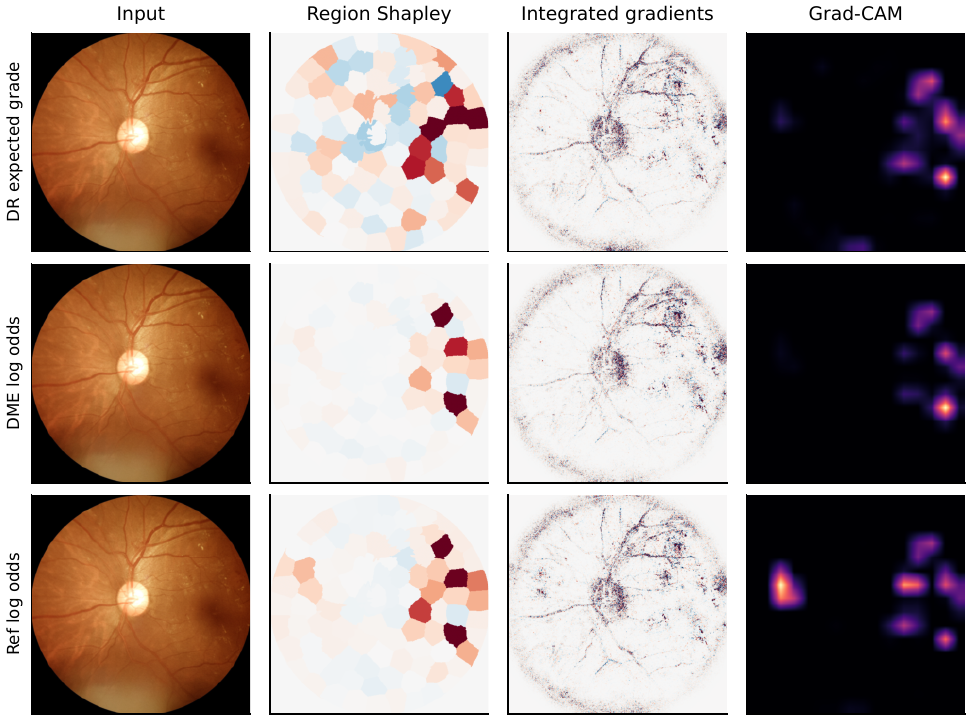}\end{center}
\begin{center}\footnotesize
\setlength{\tabcolsep}{3pt}
\begin{tabular}{lrrr}
\toprule
IG target & Output contrast & Absolute residual & Relative residual \\
\midrule
DR expected grade & 2.27340 & 0.00707 & 0.31\% \\
DME log odds & 15.79419 & 0.0263 & 0.17\% \\
Ref log odds & 12.94804 & 0.013 & 0.10\% \\
\bottomrule\end{tabular}\end{center}
All three primary IG outputs have relative completeness residual at most 10\%. Signed attribution increases or decreases the named output relative to the blurred reference. A high DR expected-grade contribution is not a contribution to every DR class probability. Binary argmax classes shown here use 0.5 and can differ from validation-Youden decisions.

\clearpage\subsection{Case 05: DR grade 4 correct}
Tier: deep; actual regions: 92; antithetic permutation pairs: 16; IG nodes: 256. Selection follows the recorded case rule, not the visual attractiveness of the map.
\begin{center}\footnotesize
\setlength{\tabcolsep}{3pt}
\begin{tabular}{lrrrr}
\toprule
Task & Observed label & Argmax class & Max probability & Explained output \\
\midrule
DR & 4 & 4 & 0.9999 & 3.9999 \\
DME & 0 & 0 & 0.8038 & -1.4101 \\
Ref & 1 & 1 & 0.9998 & 8.5017 \\
\bottomrule\end{tabular}\end{center}
\begin{center}\includegraphics[width=\linewidth]{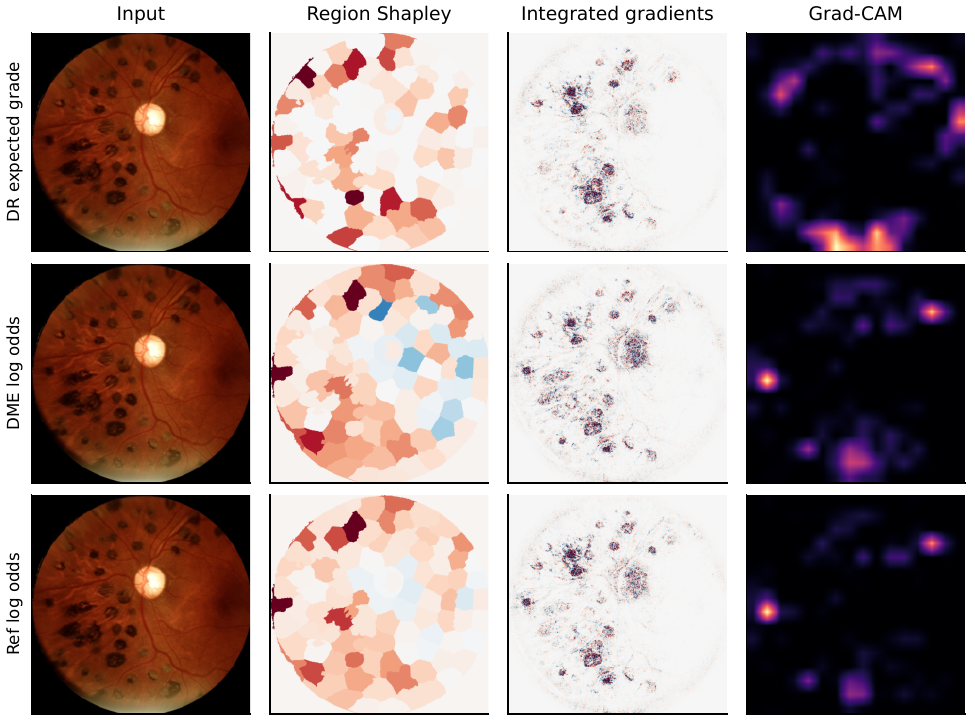}\end{center}
\begin{center}\footnotesize
\setlength{\tabcolsep}{3pt}
\begin{tabular}{lrrr}
\toprule
IG target & Output contrast & Absolute residual & Relative residual \\
\midrule
DR expected grade & 3.96746 & 0.0133 & 0.33\% \\
DME log odds & 7.25146 & 5.74e-05 & 0.00\% \\
Ref log odds & 10.03802 & 0.0138 & 0.14\% \\
\bottomrule\end{tabular}\end{center}
All three primary IG outputs have relative completeness residual at most 10\%. Signed attribution increases or decreases the named output relative to the blurred reference. A high DR expected-grade contribution is not a contribution to every DR class probability. Binary argmax classes shown here use 0.5 and can differ from validation-Youden decisions.

\clearpage\subsection{Case 06: DME false positive}
Tier: deep; actual regions: 97; antithetic permutation pairs: 16; IG nodes: 256. Selection follows the recorded case rule, not the visual attractiveness of the map.
\begin{center}\footnotesize
\setlength{\tabcolsep}{3pt}
\begin{tabular}{lrrrr}
\toprule
Task & Observed label & Argmax class & Max probability & Explained output \\
\midrule
DR & 2 & 2 & 0.9568 & 2.0173 \\
DME & 0 & 1 & 0.9999 & 9.1767 \\
Ref & 1 & 1 & 1.0000 & 11.5191 \\
\bottomrule\end{tabular}\end{center}
\begin{center}\includegraphics[width=\linewidth]{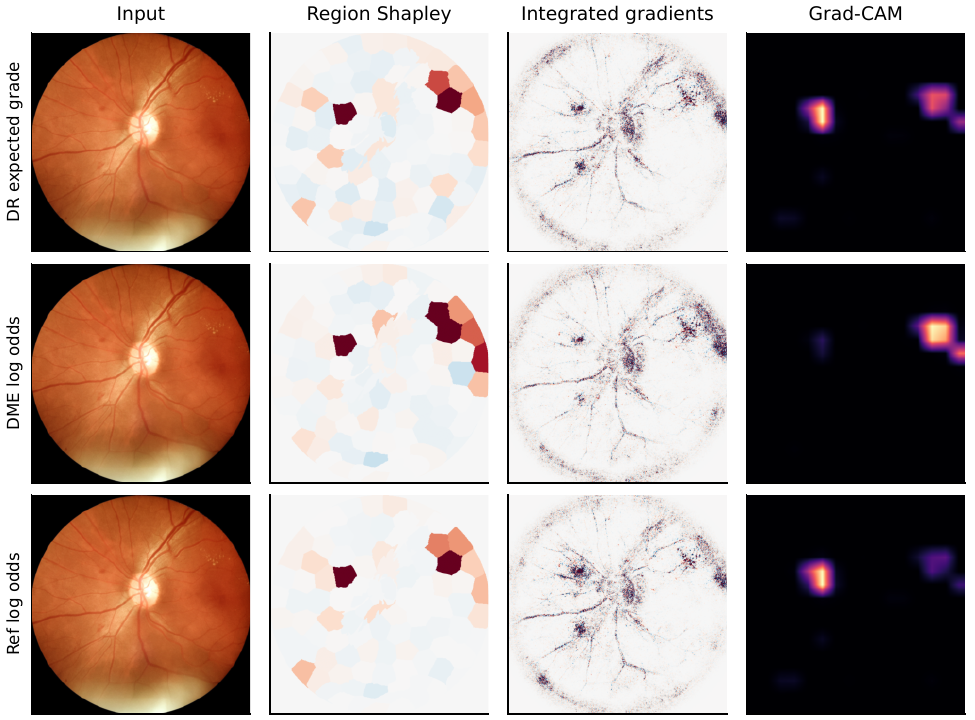}\end{center}
\begin{center}\footnotesize
\setlength{\tabcolsep}{3pt}
\begin{tabular}{lrrr}
\toprule
IG target & Output contrast & Absolute residual & Relative residual \\
\midrule
DR expected grade & 2.00195 & 0.0106 & 0.53\% \\
DME log odds & 15.05790 & 0.037 & 0.25\% \\
Ref log odds & 13.67049 & 0.0244 & 0.18\% \\
\bottomrule\end{tabular}\end{center}
All three primary IG outputs have relative completeness residual at most 10\%. Signed attribution increases or decreases the named output relative to the blurred reference. A high DR expected-grade contribution is not a contribution to every DR class probability. Binary argmax classes shown here use 0.5 and can differ from validation-Youden decisions.

\clearpage\subsection{Case 07: DME false negative}
Tier: deep; actual regions: 96; antithetic permutation pairs: 16; IG nodes: 256. Selection follows the recorded case rule, not the visual attractiveness of the map.
\begin{center}\footnotesize
\setlength{\tabcolsep}{3pt}
\begin{tabular}{lrrrr}
\toprule
Task & Observed label & Argmax class & Max probability & Explained output \\
\midrule
DR & 2 & 0 & 0.8477 & 0.1826 \\
DME & 1 & 0 & 1.0000 & -10.4646 \\
Ref & 1 & 0 & 0.9736 & -3.6092 \\
\bottomrule\end{tabular}\end{center}
\begin{center}\includegraphics[width=\linewidth]{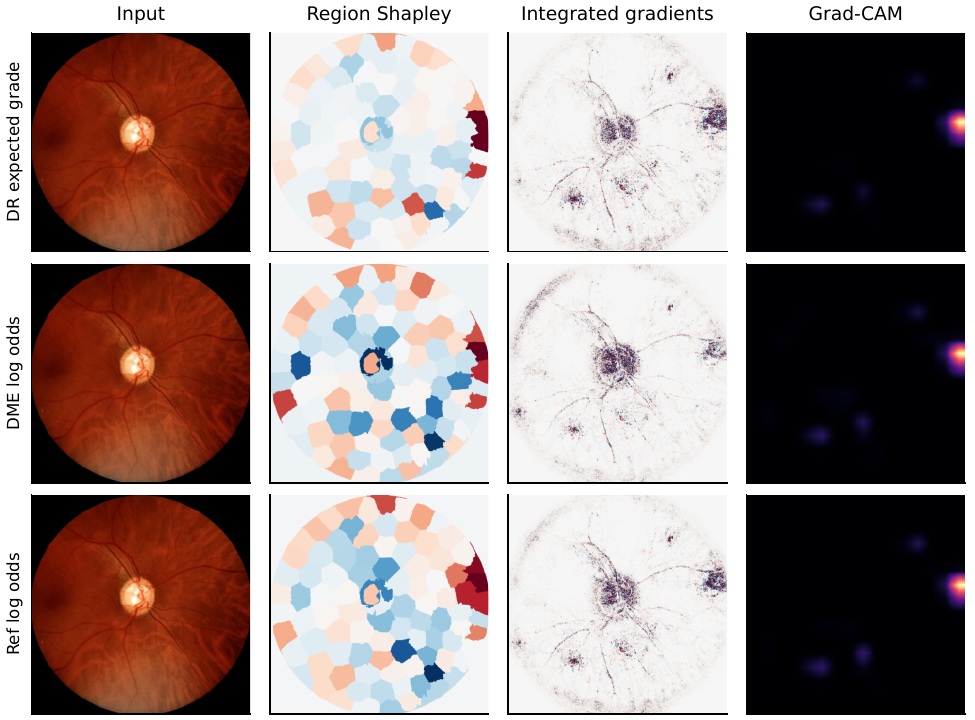}\end{center}
\begin{center}\footnotesize
\setlength{\tabcolsep}{3pt}
\begin{tabular}{lrrr}
\toprule
IG target & Output contrast & Absolute residual & Relative residual \\
\midrule
DR expected grade & 0.18087 & 0.00233 & 1.29\% \\
DME log odds & -2.01084 & 0.0443 & 2.21\% \\
Ref log odds & 1.17031 & 0.03 & 2.56\% \\
\bottomrule\end{tabular}\end{center}
All three primary IG outputs have relative completeness residual at most 10\%. Signed attribution increases or decreases the named output relative to the blurred reference. A high DR expected-grade contribution is not a contribution to every DR class probability. Binary argmax classes shown here use 0.5 and can differ from validation-Youden decisions.

\clearpage\subsection{Case 08: REF error}
Tier: deep; actual regions: 97; antithetic permutation pairs: 16; IG nodes: 256. Selection follows the recorded case rule, not the visual attractiveness of the map.
\begin{center}\footnotesize
\setlength{\tabcolsep}{3pt}
\begin{tabular}{lrrrr}
\toprule
Task & Observed label & Argmax class & Max probability & Explained output \\
\midrule
DR & 2 & 0 & 0.9973 & 0.0028 \\
DME & 0 & 0 & 1.0000 & -10.9642 \\
Ref & 1 & 0 & 1.0000 & -11.4052 \\
\bottomrule\end{tabular}\end{center}
\begin{center}\includegraphics[width=\linewidth]{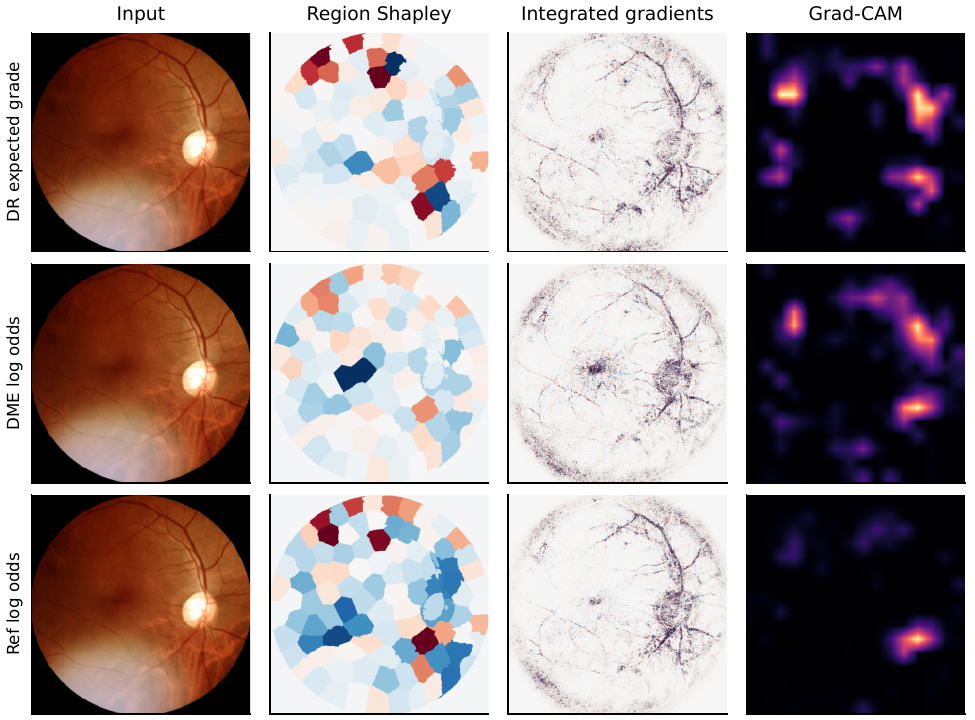}\end{center}
\begin{center}\footnotesize
\setlength{\tabcolsep}{3pt}
\begin{tabular}{lrrr}
\toprule
IG target & Output contrast & Absolute residual & Relative residual \\
\midrule
DR expected grade & 0.00211 & 3.37e-07 & 0.02\% \\
DME log odds & -4.16316 & 0.017 & 0.41\% \\
Ref log odds & -5.01896 & 0.0259 & 0.52\% \\
\bottomrule\end{tabular}\end{center}
All three primary IG outputs have relative completeness residual at most 10\%. Signed attribution increases or decreases the named output relative to the blurred reference. A high DR expected-grade contribution is not a contribution to every DR class probability. Binary argmax classes shown here use 0.5 and can differ from validation-Youden decisions.

\clearpage\subsection{Case 09: Additional DR grade 0}
Tier: broad; actual regions: 48; antithetic permutation pairs: 2; IG nodes: 32. Selection follows the recorded case rule, not the visual attractiveness of the map.
\begin{center}\footnotesize
\setlength{\tabcolsep}{3pt}
\begin{tabular}{lrrrr}
\toprule
Task & Observed label & Argmax class & Max probability & Explained output \\
\midrule
DR & 0 & 0 & 0.8646 & 0.1438 \\
DME & 0 & 0 & 0.9998 & -8.7965 \\
Ref & 0 & 0 & 0.9848 & -4.1730 \\
\bottomrule\end{tabular}\end{center}
\begin{center}\includegraphics[width=\linewidth]{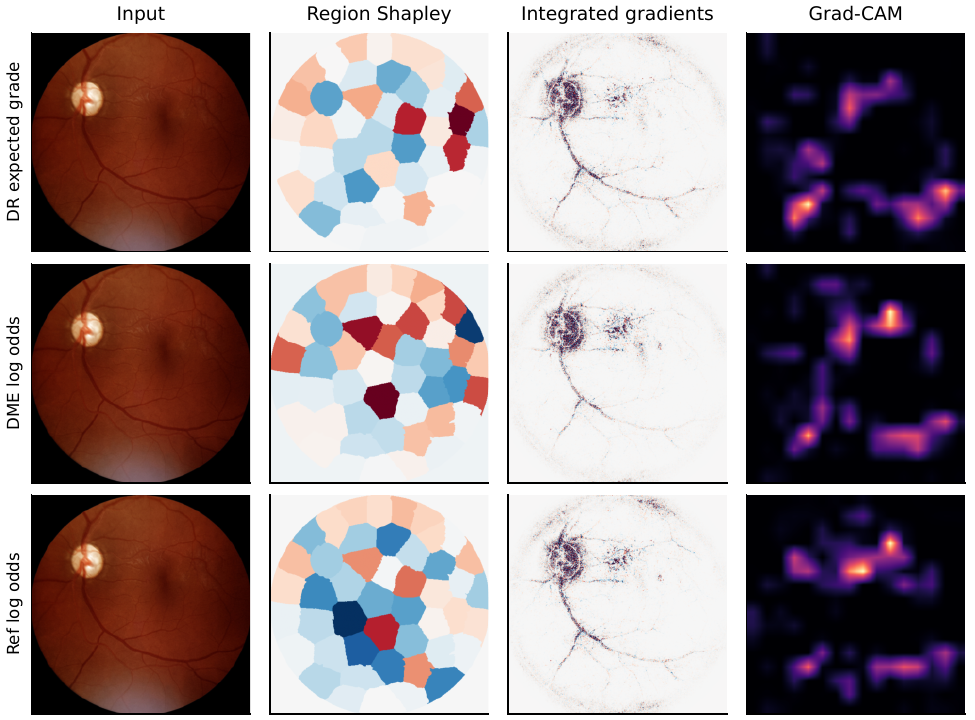}\end{center}
\begin{center}\footnotesize
\setlength{\tabcolsep}{3pt}
\begin{tabular}{lrrr}
\toprule
IG target & Output contrast & Absolute residual & Relative residual \\
\midrule
DR expected grade & 0.02143 & 0.00983 & 45.88\% \\
DME log odds & 0.99324 & 0.207 & 20.87\% \\
Ref log odds & -3.87395 & 0.031 & 0.80\% \\
\bottomrule\end{tabular}\end{center}
2 of three primary IG outputs exceed 10\% relative completeness residual; interpret the corresponding maps as finite-budget exploratory results. Signed attribution increases or decreases the named output relative to the blurred reference. A high DR expected-grade contribution is not a contribution to every DR class probability. Binary argmax classes shown here use 0.5 and can differ from validation-Youden decisions.

\clearpage\subsection{Case 10: Additional DR grade 1}
Tier: broad; actual regions: 48; antithetic permutation pairs: 2; IG nodes: 32. Selection follows the recorded case rule, not the visual attractiveness of the map.
\begin{center}\footnotesize
\setlength{\tabcolsep}{3pt}
\begin{tabular}{lrrrr}
\toprule
Task & Observed label & Argmax class & Max probability & Explained output \\
\midrule
DR & 1 & 0 & 0.5450 & 1.3073 \\
DME & 0 & 0 & 0.9830 & -4.0595 \\
Ref & 0 & 1 & 0.9150 & 2.3765 \\
\bottomrule\end{tabular}\end{center}
\begin{center}\includegraphics[width=\linewidth]{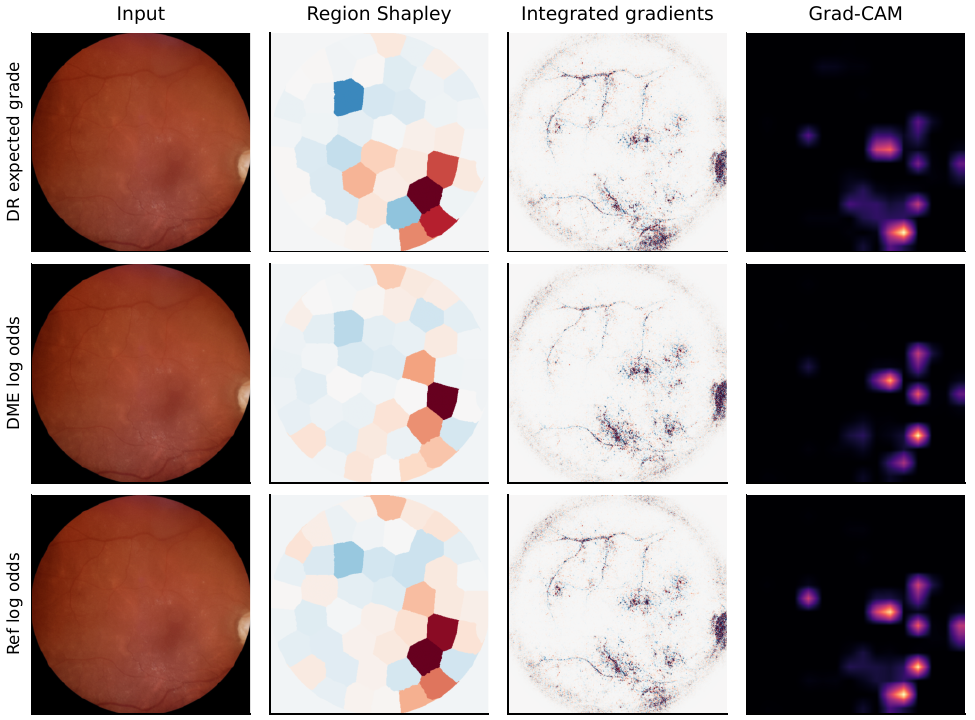}\end{center}
\begin{center}\footnotesize
\setlength{\tabcolsep}{3pt}
\begin{tabular}{lrrr}
\toprule
IG target & Output contrast & Absolute residual & Relative residual \\
\midrule
DR expected grade & 1.24080 & 0.11 & 8.83\% \\
DME log odds & 1.98905 & 0.443 & 22.29\% \\
Ref log odds & 3.73994 & 0.342 & 9.15\% \\
\bottomrule\end{tabular}\end{center}
1 of three primary IG outputs exceed 10\% relative completeness residual; interpret the corresponding maps as finite-budget exploratory results. Signed attribution increases or decreases the named output relative to the blurred reference. A high DR expected-grade contribution is not a contribution to every DR class probability. Binary argmax classes shown here use 0.5 and can differ from validation-Youden decisions.

\clearpage\subsection{Case 11: Additional DR grade 2}
Tier: broad; actual regions: 49; antithetic permutation pairs: 2; IG nodes: 32. Selection follows the recorded case rule, not the visual attractiveness of the map.
\begin{center}\footnotesize
\setlength{\tabcolsep}{3pt}
\begin{tabular}{lrrrr}
\toprule
Task & Observed label & Argmax class & Max probability & Explained output \\
\midrule
DR & 2 & 2 & 0.4061 & 1.0883 \\
DME & 1 & 1 & 0.9363 & 2.6881 \\
Ref & 1 & 1 & 0.9561 & 3.0814 \\
\bottomrule\end{tabular}\end{center}
\begin{center}\includegraphics[width=\linewidth]{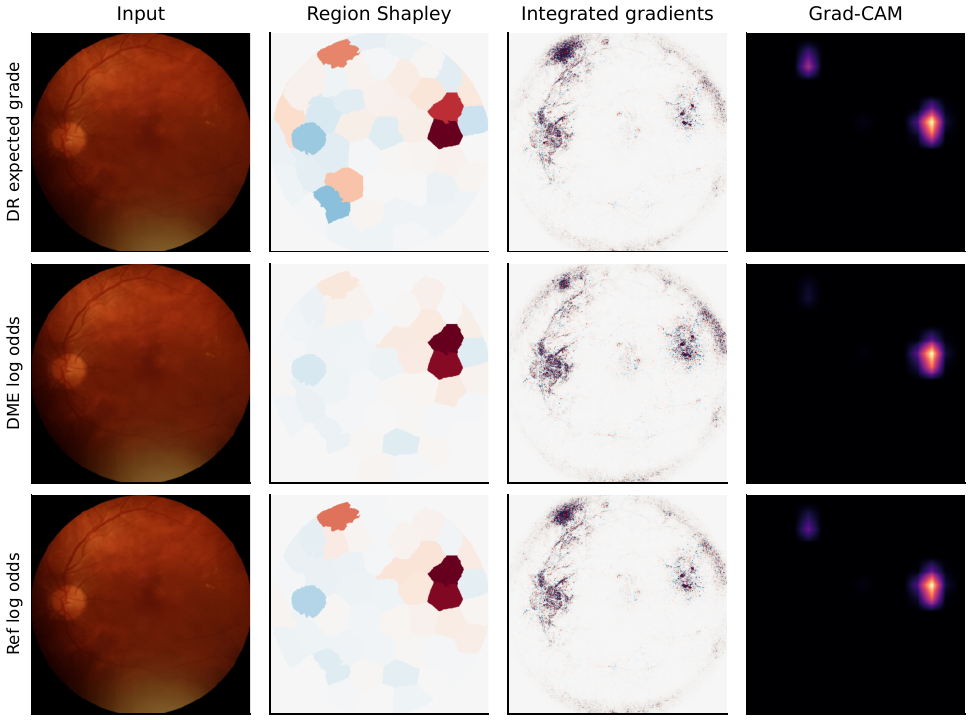}\end{center}
\begin{center}\footnotesize
\setlength{\tabcolsep}{3pt}
\begin{tabular}{lrrr}
\toprule
IG target & Output contrast & Absolute residual & Relative residual \\
\midrule
DR expected grade & 1.08697 & 0.00947 & 0.87\% \\
DME log odds & 11.13920 & 0.208 & 1.87\% \\
Ref log odds & 9.37344 & 0.163 & 1.74\% \\
\bottomrule\end{tabular}\end{center}
All three primary IG outputs have relative completeness residual at most 10\%. Signed attribution increases or decreases the named output relative to the blurred reference. A high DR expected-grade contribution is not a contribution to every DR class probability. Binary argmax classes shown here use 0.5 and can differ from validation-Youden decisions.

\clearpage\subsection{Case 12: Additional DR grade 3}
Tier: broad; actual regions: 47; antithetic permutation pairs: 2; IG nodes: 32. Selection follows the recorded case rule, not the visual attractiveness of the map.
\begin{center}\footnotesize
\setlength{\tabcolsep}{3pt}
\begin{tabular}{lrrrr}
\toprule
Task & Observed label & Argmax class & Max probability & Explained output \\
\midrule
DR & 3 & 3 & 0.4858 & 3.3311 \\
DME & 1 & 1 & 0.9999 & 9.4189 \\
Ref & 1 & 1 & 1.0000 & 15.0671 \\
\bottomrule\end{tabular}\end{center}
\begin{center}\includegraphics[width=\linewidth]{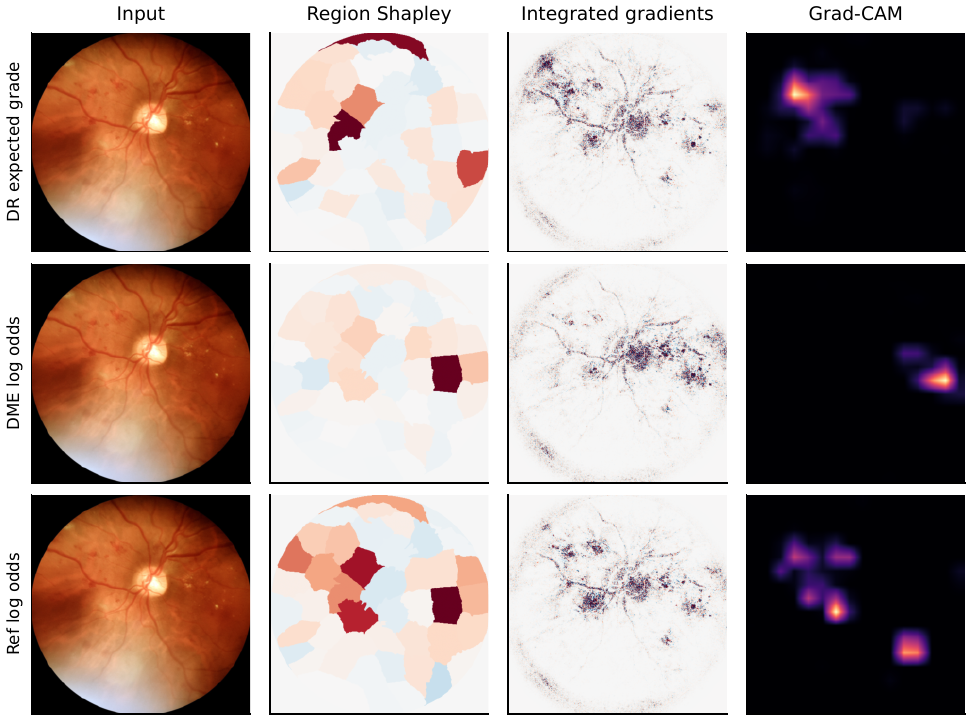}\end{center}
\begin{center}\footnotesize
\setlength{\tabcolsep}{3pt}
\begin{tabular}{lrrr}
\toprule
IG target & Output contrast & Absolute residual & Relative residual \\
\midrule
DR expected grade & 3.17757 & 0.0913 & 2.87\% \\
DME log odds & 14.08498 & 0.574 & 4.07\% \\
Ref log odds & 14.35307 & 0.339 & 2.36\% \\
\bottomrule\end{tabular}\end{center}
All three primary IG outputs have relative completeness residual at most 10\%. Signed attribution increases or decreases the named output relative to the blurred reference. A high DR expected-grade contribution is not a contribution to every DR class probability. Binary argmax classes shown here use 0.5 and can differ from validation-Youden decisions.

\clearpage\subsection{Case 13: Additional DR grade 4}
Tier: broad; actual regions: 49; antithetic permutation pairs: 2; IG nodes: 32. Selection follows the recorded case rule, not the visual attractiveness of the map.
\begin{center}\footnotesize
\setlength{\tabcolsep}{3pt}
\begin{tabular}{lrrrr}
\toprule
Task & Observed label & Argmax class & Max probability & Explained output \\
\midrule
DR & 4 & 4 & 0.5566 & 3.5546 \\
DME & 0 & 1 & 0.9945 & 5.2032 \\
Ref & 1 & 1 & 1.0000 & 12.1713 \\
\bottomrule\end{tabular}\end{center}
\begin{center}\includegraphics[width=\linewidth]{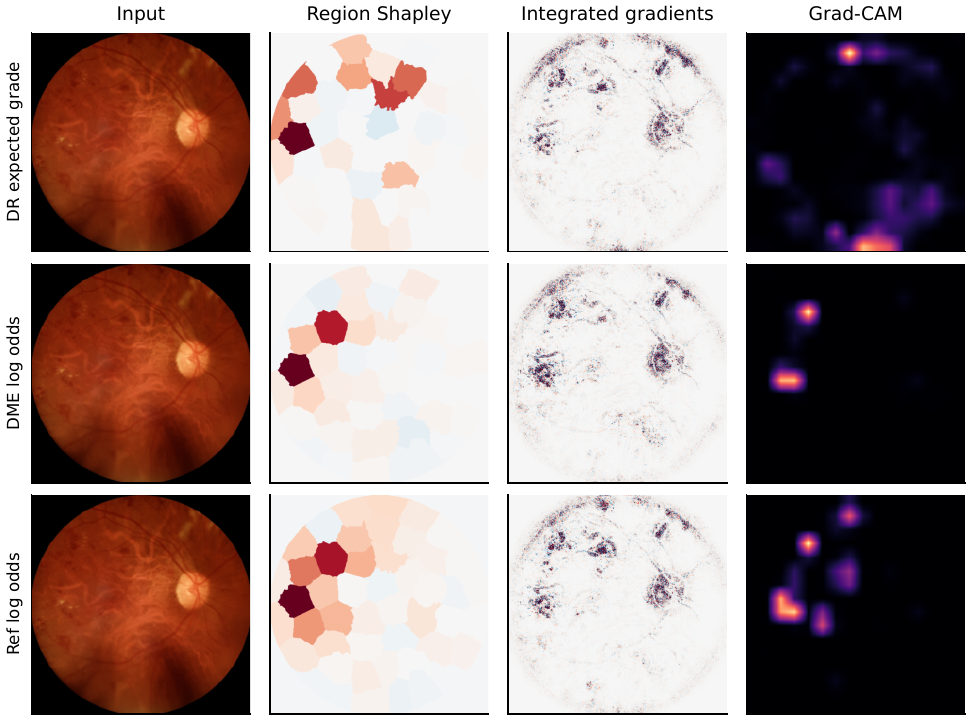}\end{center}
\begin{center}\footnotesize
\setlength{\tabcolsep}{3pt}
\begin{tabular}{lrrr}
\toprule
IG target & Output contrast & Absolute residual & Relative residual \\
\midrule
DR expected grade & 3.52036 & 0.0448 & 1.27\% \\
DME log odds & 15.21908 & 0.886 & 5.82\% \\
Ref log odds & 13.82521 & 0.423 & 3.06\% \\
\bottomrule\end{tabular}\end{center}
All three primary IG outputs have relative completeness residual at most 10\%. Signed attribution increases or decreases the named output relative to the blurred reference. A high DR expected-grade contribution is not a contribution to every DR class probability. Binary argmax classes shown here use 0.5 and can differ from validation-Youden decisions.

\clearpage\subsection{Case 14: Additional dme class 0}
Tier: broad; actual regions: 49; antithetic permutation pairs: 2; IG nodes: 32. Selection follows the recorded case rule, not the visual attractiveness of the map.
\begin{center}\footnotesize
\setlength{\tabcolsep}{3pt}
\begin{tabular}{lrrrr}
\toprule
Task & Observed label & Argmax class & Max probability & Explained output \\
\midrule
DR & 0 & 0 & 0.9997 & 0.0003 \\
DME & 0 & 0 & 1.0000 & -12.9431 \\
Ref & 0 & 0 & 1.0000 & -12.5596 \\
\bottomrule\end{tabular}\end{center}
\begin{center}\includegraphics[width=\linewidth]{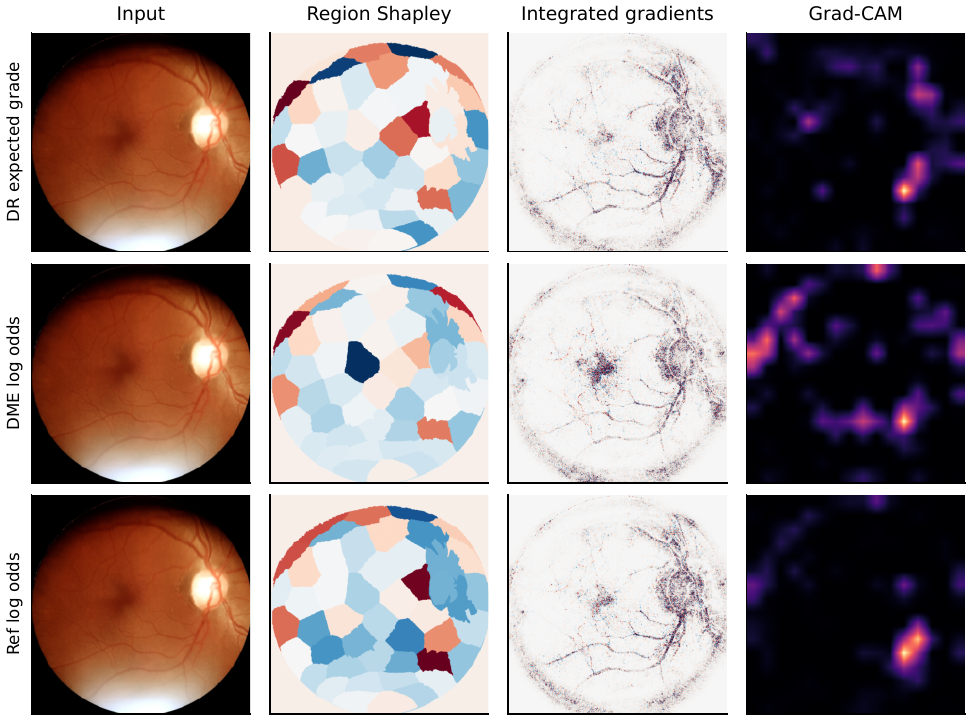}\end{center}
\begin{center}\footnotesize
\setlength{\tabcolsep}{3pt}
\begin{tabular}{lrrr}
\toprule
IG target & Output contrast & Absolute residual & Relative residual \\
\midrule
DR expected grade & 0.00014 & 1.78e-05 & 12.65\% \\
DME log odds & -3.20247 & 0.0647 & 2.02\% \\
Ref log odds & -3.28773 & 0.000504 & 0.02\% \\
\bottomrule\end{tabular}\end{center}
1 of three primary IG outputs exceed 10\% relative completeness residual; interpret the corresponding maps as finite-budget exploratory results. Signed attribution increases or decreases the named output relative to the blurred reference. A high DR expected-grade contribution is not a contribution to every DR class probability. Binary argmax classes shown here use 0.5 and can differ from validation-Youden decisions.

\clearpage\subsection{Case 15: Additional dme class 1}
Tier: broad; actual regions: 49; antithetic permutation pairs: 2; IG nodes: 32. Selection follows the recorded case rule, not the visual attractiveness of the map.
\begin{center}\footnotesize
\setlength{\tabcolsep}{3pt}
\begin{tabular}{lrrrr}
\toprule
Task & Observed label & Argmax class & Max probability & Explained output \\
\midrule
DR & 4 & 3 & 0.9547 & 2.9690 \\
DME & 1 & 1 & 0.9999 & 9.6620 \\
Ref & 1 & 1 & 1.0000 & 14.7561 \\
\bottomrule\end{tabular}\end{center}
\begin{center}\includegraphics[width=\linewidth]{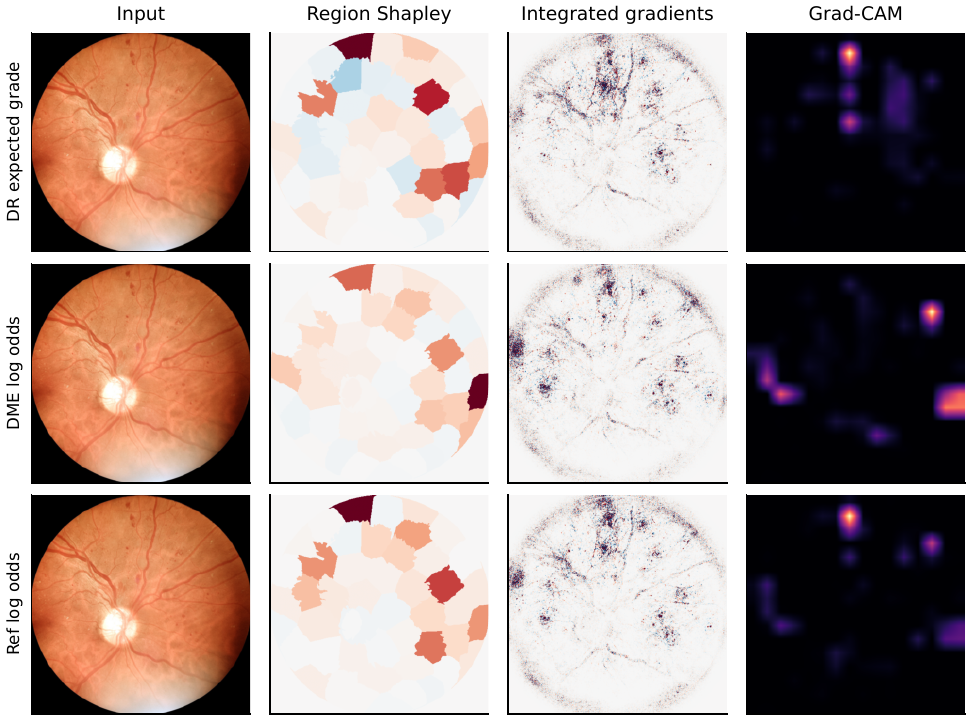}\end{center}
\begin{center}\footnotesize
\setlength{\tabcolsep}{3pt}
\begin{tabular}{lrrr}
\toprule
IG target & Output contrast & Absolute residual & Relative residual \\
\midrule
DR expected grade & 2.92295 & 0.302 & 10.33\% \\
DME log odds & 11.68127 & 0.0169 & 0.14\% \\
Ref log odds & 14.79919 & 1.03 & 6.96\% \\
\bottomrule\end{tabular}\end{center}
1 of three primary IG outputs exceed 10\% relative completeness residual; interpret the corresponding maps as finite-budget exploratory results. Signed attribution increases or decreases the named output relative to the blurred reference. A high DR expected-grade contribution is not a contribution to every DR class probability. Binary argmax classes shown here use 0.5 and can differ from validation-Youden decisions.

\clearpage\subsection{Case 16: Additional ref class 0}
Tier: broad; actual regions: 49; antithetic permutation pairs: 2; IG nodes: 32. Selection follows the recorded case rule, not the visual attractiveness of the map.
\begin{center}\footnotesize
\setlength{\tabcolsep}{3pt}
\begin{tabular}{lrrrr}
\toprule
Task & Observed label & Argmax class & Max probability & Explained output \\
\midrule
DR & 0 & 0 & 0.9997 & 0.0003 \\
DME & 0 & 0 & 1.0000 & -10.0381 \\
Ref & 0 & 0 & 1.0000 & -9.9682 \\
\bottomrule\end{tabular}\end{center}
\begin{center}\includegraphics[width=\linewidth]{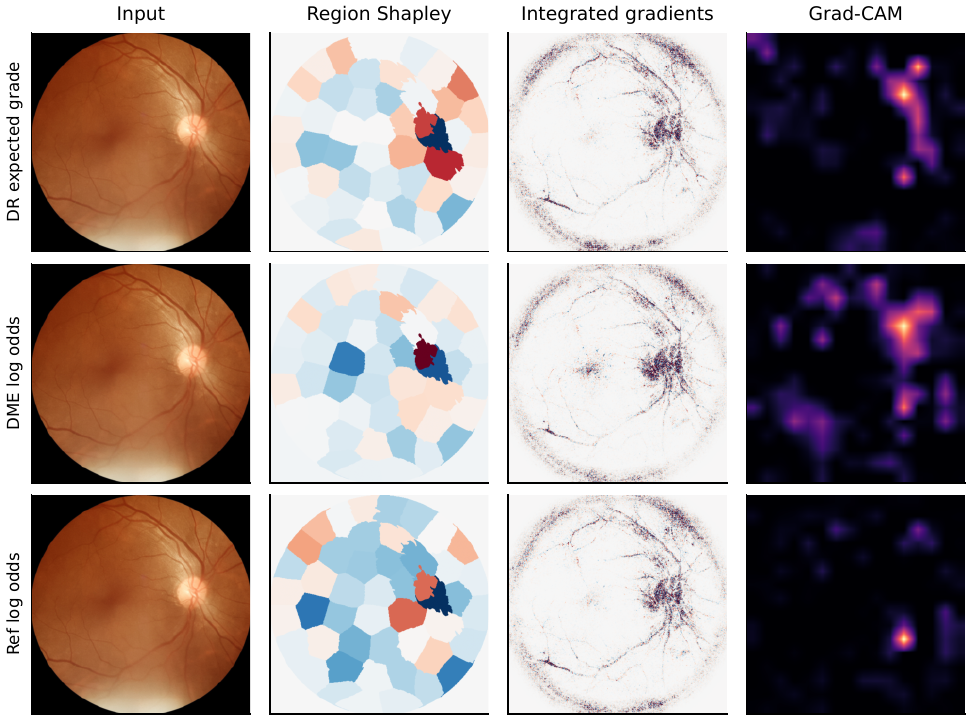}\end{center}
\begin{center}\footnotesize
\setlength{\tabcolsep}{3pt}
\begin{tabular}{lrrr}
\toprule
IG target & Output contrast & Absolute residual & Relative residual \\
\midrule
DR expected grade & -0.00033 & 0.000384 & 116.49\% \\
DME log odds & -4.48291 & 0.631 & 14.09\% \\
Ref log odds & -5.16276 & 0.587 & 11.38\% \\
\bottomrule\end{tabular}\end{center}
3 of three primary IG outputs exceed 10\% relative completeness residual; interpret the corresponding maps as finite-budget exploratory results. Signed attribution increases or decreases the named output relative to the blurred reference. A high DR expected-grade contribution is not a contribution to every DR class probability. Binary argmax classes shown here use 0.5 and can differ from validation-Youden decisions.

\clearpage\subsection{Case 17: Additional ref class 1}
Tier: broad; actual regions: 48; antithetic permutation pairs: 2; IG nodes: 32. Selection follows the recorded case rule, not the visual attractiveness of the map.
\begin{center}\footnotesize
\setlength{\tabcolsep}{3pt}
\begin{tabular}{lrrrr}
\toprule
Task & Observed label & Argmax class & Max probability & Explained output \\
\midrule
DR & 4 & 3 & 0.5307 & 3.0794 \\
DME & 1 & 1 & 0.9999 & 9.5095 \\
Ref & 1 & 1 & 1.0000 & 10.7863 \\
\bottomrule\end{tabular}\end{center}
\begin{center}\includegraphics[width=\linewidth]{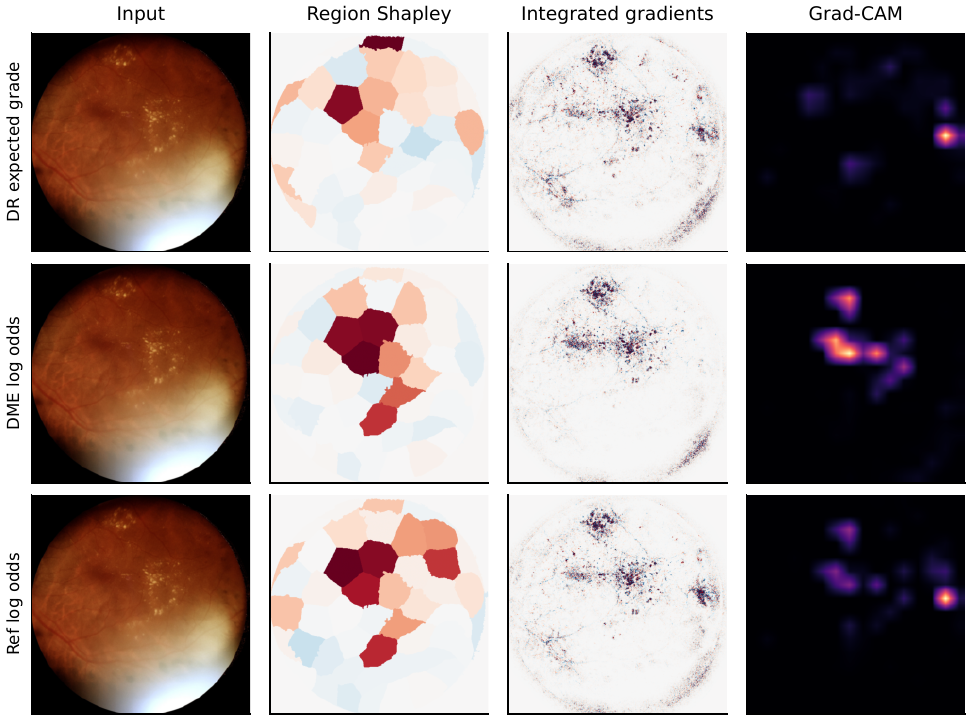}\end{center}
\begin{center}\footnotesize
\setlength{\tabcolsep}{3pt}
\begin{tabular}{lrrr}
\toprule
IG target & Output contrast & Absolute residual & Relative residual \\
\midrule
DR expected grade & 3.06716 & 0.00658 & 0.21\% \\
DME log odds & 18.10924 & 0.0428 & 0.24\% \\
Ref log odds & 13.46716 & 0.349 & 2.59\% \\
\bottomrule\end{tabular}\end{center}
All three primary IG outputs have relative completeness residual at most 10\%. Signed attribution increases or decreases the named output relative to the blurred reference. A high DR expected-grade contribution is not a contribution to every DR class probability. Binary argmax classes shown here use 0.5 and can differ from validation-Youden decisions.

\clearpage\subsection{Case 18: Additional DR grade 0}
Tier: broad; actual regions: 48; antithetic permutation pairs: 2; IG nodes: 32. Selection follows the recorded case rule, not the visual attractiveness of the map.
\begin{center}\footnotesize
\setlength{\tabcolsep}{3pt}
\begin{tabular}{lrrrr}
\toprule
Task & Observed label & Argmax class & Max probability & Explained output \\
\midrule
DR & 0 & 0 & 0.9375 & 0.0728 \\
DME & 0 & 0 & 0.9999 & -9.5518 \\
Ref & 0 & 0 & 0.9719 & -3.5438 \\
\bottomrule\end{tabular}\end{center}
\begin{center}\includegraphics[width=\linewidth]{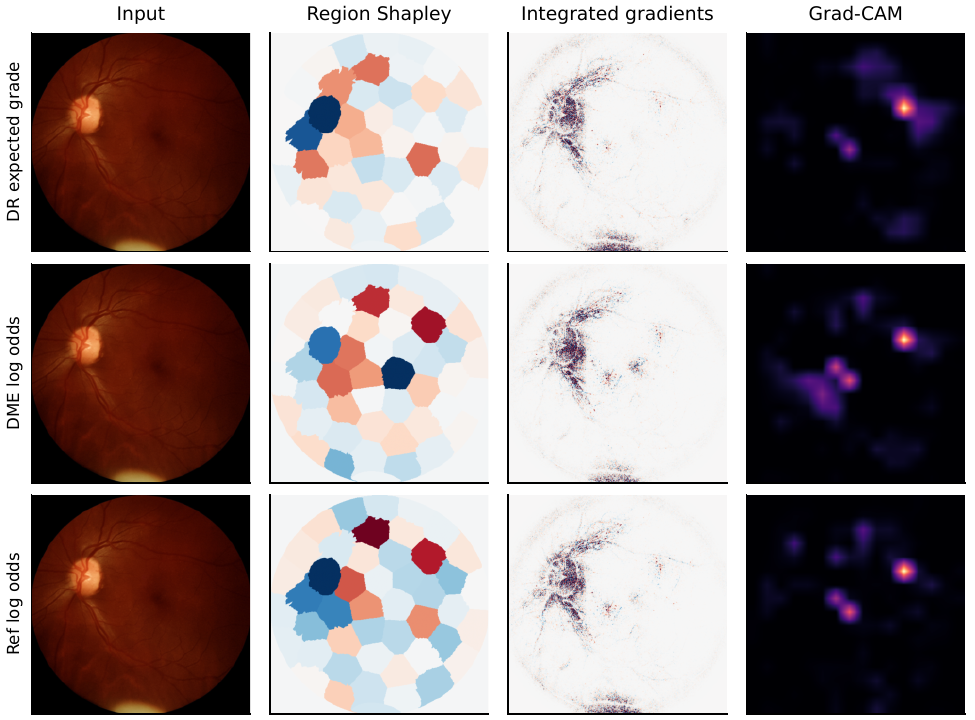}\end{center}
\begin{center}\footnotesize
\setlength{\tabcolsep}{3pt}
\begin{tabular}{lrrr}
\toprule
IG target & Output contrast & Absolute residual & Relative residual \\
\midrule
DR expected grade & 0.01311 & 0.00488 & 37.20\% \\
DME log odds & 0.32631 & 0.249 & 76.34\% \\
Ref log odds & -3.25444 & 0.0386 & 1.19\% \\
\bottomrule\end{tabular}\end{center}
2 of three primary IG outputs exceed 10\% relative completeness residual; interpret the corresponding maps as finite-budget exploratory results. Signed attribution increases or decreases the named output relative to the blurred reference. A high DR expected-grade contribution is not a contribution to every DR class probability. Binary argmax classes shown here use 0.5 and can differ from validation-Youden decisions.

\clearpage\subsection{Case 19: Additional DR grade 1}
Tier: broad; actual regions: 48; antithetic permutation pairs: 2; IG nodes: 32. Selection follows the recorded case rule, not the visual attractiveness of the map.
\begin{center}\footnotesize
\setlength{\tabcolsep}{3pt}
\begin{tabular}{lrrrr}
\toprule
Task & Observed label & Argmax class & Max probability & Explained output \\
\midrule
DR & 1 & 0 & 0.4539 & 1.5713 \\
DME & 0 & 0 & 0.9254 & -2.5181 \\
Ref & 0 & 0 & 0.5110 & -0.0438 \\
\bottomrule\end{tabular}\end{center}
\begin{center}\includegraphics[width=\linewidth]{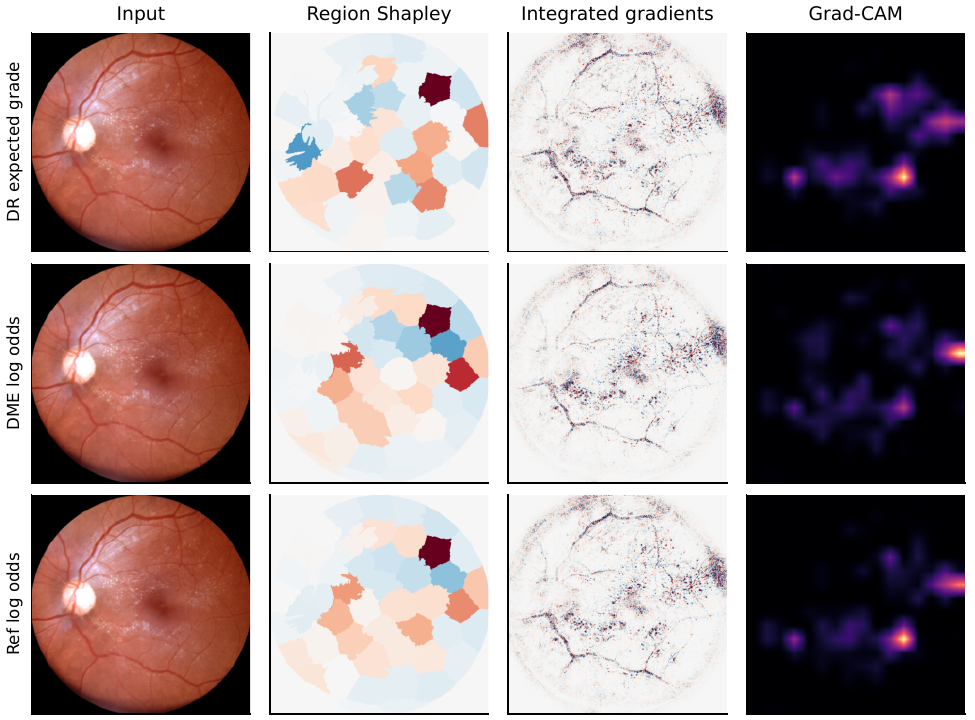}\end{center}
\begin{center}\footnotesize
\setlength{\tabcolsep}{3pt}
\begin{tabular}{lrrr}
\toprule
IG target & Output contrast & Absolute residual & Relative residual \\
\midrule
DR expected grade & 1.56271 & 0.287 & 18.36\% \\
DME log odds & 3.60576 & 0.703 & 19.49\% \\
Ref log odds & 3.47867 & 0.781 & 22.44\% \\
\bottomrule\end{tabular}\end{center}
3 of three primary IG outputs exceed 10\% relative completeness residual; interpret the corresponding maps as finite-budget exploratory results. Signed attribution increases or decreases the named output relative to the blurred reference. A high DR expected-grade contribution is not a contribution to every DR class probability. Binary argmax classes shown here use 0.5 and can differ from validation-Youden decisions.

\clearpage\subsection{Case 20: Additional DR grade 2}
Tier: broad; actual regions: 49; antithetic permutation pairs: 2; IG nodes: 32. Selection follows the recorded case rule, not the visual attractiveness of the map.
\begin{center}\footnotesize
\setlength{\tabcolsep}{3pt}
\begin{tabular}{lrrrr}
\toprule
Task & Observed label & Argmax class & Max probability & Explained output \\
\midrule
DR & 2 & 2 & 0.9885 & 2.0012 \\
DME & 0 & 1 & 0.9529 & 3.0072 \\
Ref & 1 & 1 & 0.9999 & 9.1409 \\
\bottomrule\end{tabular}\end{center}
\begin{center}\includegraphics[width=\linewidth]{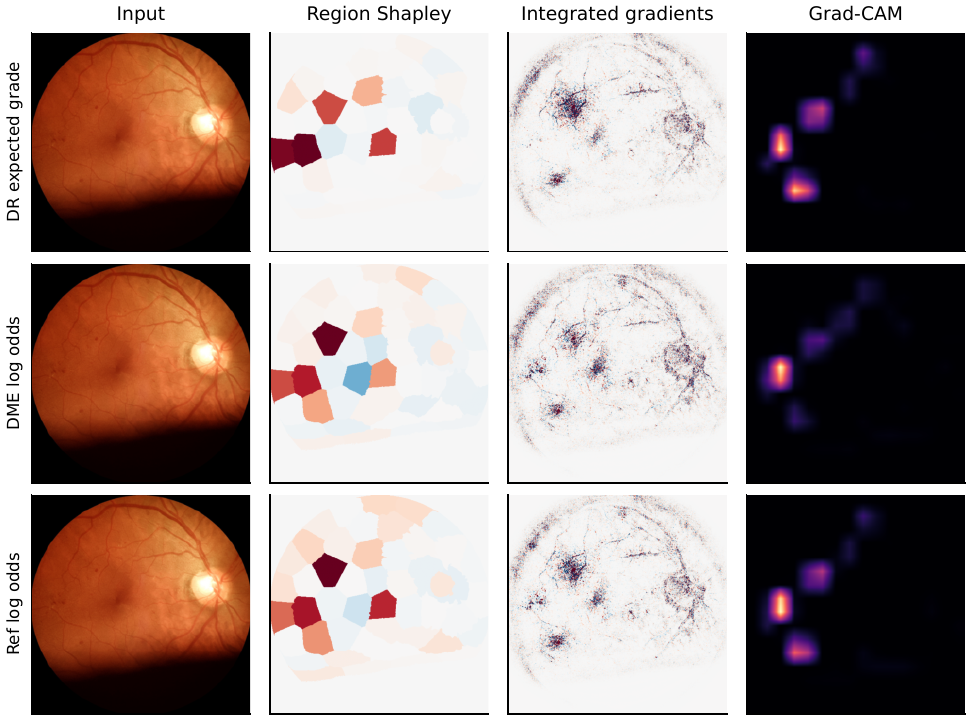}\end{center}
\begin{center}\footnotesize
\setlength{\tabcolsep}{3pt}
\begin{tabular}{lrrr}
\toprule
IG target & Output contrast & Absolute residual & Relative residual \\
\midrule
DR expected grade & 1.99881 & 0.0459 & 2.30\% \\
DME log odds & 8.57186 & 0.232 & 2.70\% \\
Ref log odds & 13.54633 & 0.163 & 1.20\% \\
\bottomrule\end{tabular}\end{center}
All three primary IG outputs have relative completeness residual at most 10\%. Signed attribution increases or decreases the named output relative to the blurred reference. A high DR expected-grade contribution is not a contribution to every DR class probability. Binary argmax classes shown here use 0.5 and can differ from validation-Youden decisions.

\clearpage\subsection{Case 21: Additional DR grade 3}
Tier: broad; actual regions: 49; antithetic permutation pairs: 2; IG nodes: 32. Selection follows the recorded case rule, not the visual attractiveness of the map.
\begin{center}\footnotesize
\setlength{\tabcolsep}{3pt}
\begin{tabular}{lrrrr}
\toprule
Task & Observed label & Argmax class & Max probability & Explained output \\
\midrule
DR & 3 & 2 & 0.5040 & 1.8590 \\
DME & 1 & 0 & 0.9598 & -3.1720 \\
Ref & 1 & 1 & 0.9582 & 3.1324 \\
\bottomrule\end{tabular}\end{center}
\begin{center}\includegraphics[width=\linewidth]{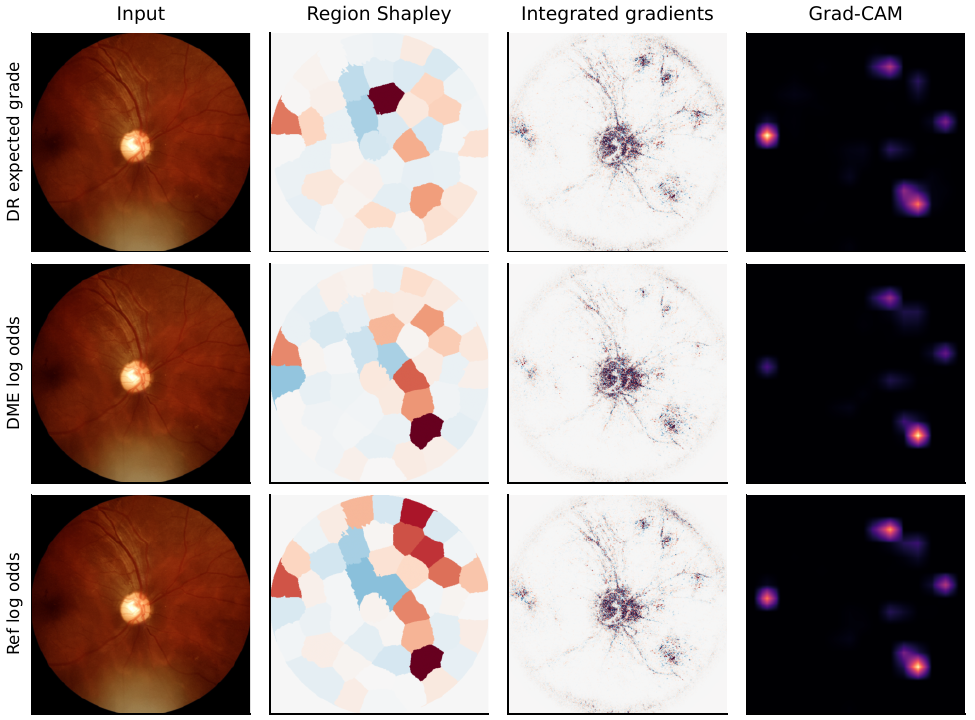}\end{center}
\begin{center}\footnotesize
\setlength{\tabcolsep}{3pt}
\begin{tabular}{lrrr}
\toprule
IG target & Output contrast & Absolute residual & Relative residual \\
\midrule
DR expected grade & 1.83364 & 0.0654 & 3.57\% \\
DME log odds & 3.96472 & 0.572 & 14.43\% \\
Ref log odds & 5.91917 & 0.348 & 5.88\% \\
\bottomrule\end{tabular}\end{center}
1 of three primary IG outputs exceed 10\% relative completeness residual; interpret the corresponding maps as finite-budget exploratory results. Signed attribution increases or decreases the named output relative to the blurred reference. A high DR expected-grade contribution is not a contribution to every DR class probability. Binary argmax classes shown here use 0.5 and can differ from validation-Youden decisions.

\clearpage\subsection{Case 22: Additional DR grade 4}
Tier: broad; actual regions: 49; antithetic permutation pairs: 2; IG nodes: 32. Selection follows the recorded case rule, not the visual attractiveness of the map.
\begin{center}\footnotesize
\setlength{\tabcolsep}{3pt}
\begin{tabular}{lrrrr}
\toprule
Task & Observed label & Argmax class & Max probability & Explained output \\
\midrule
DR & 4 & 4 & 0.9990 & 3.9989 \\
DME & 0 & 0 & 0.5327 & -0.1311 \\
Ref & 1 & 1 & 0.9996 & 7.7758 \\
\bottomrule\end{tabular}\end{center}
\begin{center}\includegraphics[width=\linewidth]{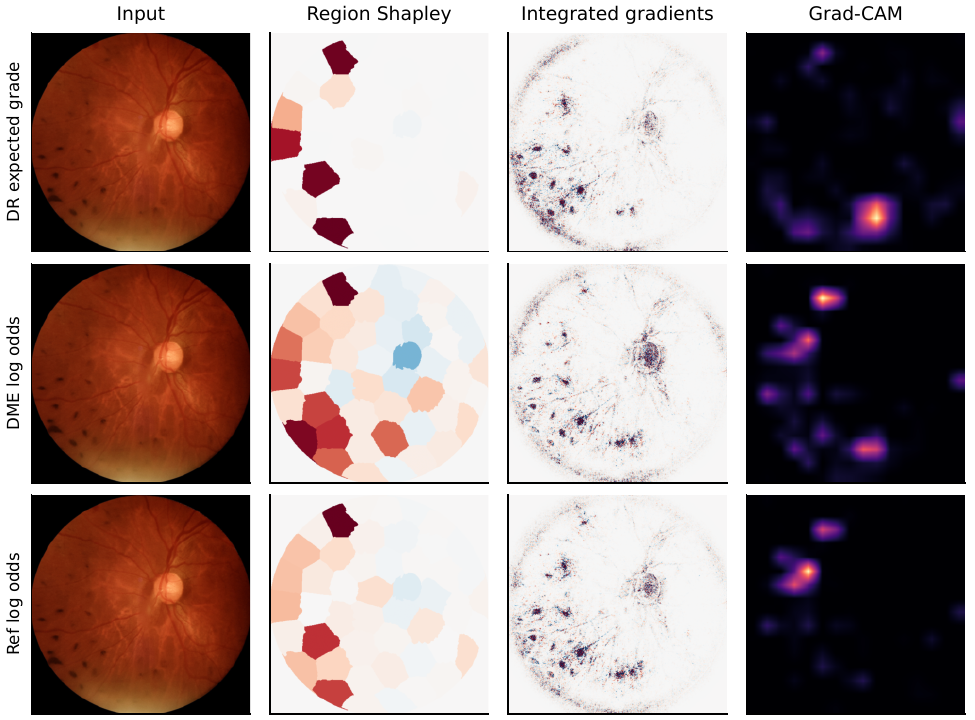}\end{center}
\begin{center}\footnotesize
\setlength{\tabcolsep}{3pt}
\begin{tabular}{lrrr}
\toprule
IG target & Output contrast & Absolute residual & Relative residual \\
\midrule
DR expected grade & 3.99057 & 0.74 & 18.54\% \\
DME log odds & 7.05352 & 1.11 & 15.75\% \\
Ref log odds & 11.65995 & 1.28 & 11.02\% \\
\bottomrule\end{tabular}\end{center}
3 of three primary IG outputs exceed 10\% relative completeness residual; interpret the corresponding maps as finite-budget exploratory results. Signed attribution increases or decreases the named output relative to the blurred reference. A high DR expected-grade contribution is not a contribution to every DR class probability. Binary argmax classes shown here use 0.5 and can differ from validation-Youden decisions.

\clearpage\subsection{Case 23: Additional dme class 0}
Tier: broad; actual regions: 47; antithetic permutation pairs: 2; IG nodes: 32. Selection follows the recorded case rule, not the visual attractiveness of the map.
\begin{center}\footnotesize
\setlength{\tabcolsep}{3pt}
\begin{tabular}{lrrrr}
\toprule
Task & Observed label & Argmax class & Max probability & Explained output \\
\midrule
DR & 0 & 0 & 0.9993 & 0.0009 \\
DME & 0 & 0 & 0.9987 & -6.6508 \\
Ref & 0 & 0 & 0.9975 & -5.9872 \\
\bottomrule\end{tabular}\end{center}
\begin{center}\includegraphics[width=\linewidth]{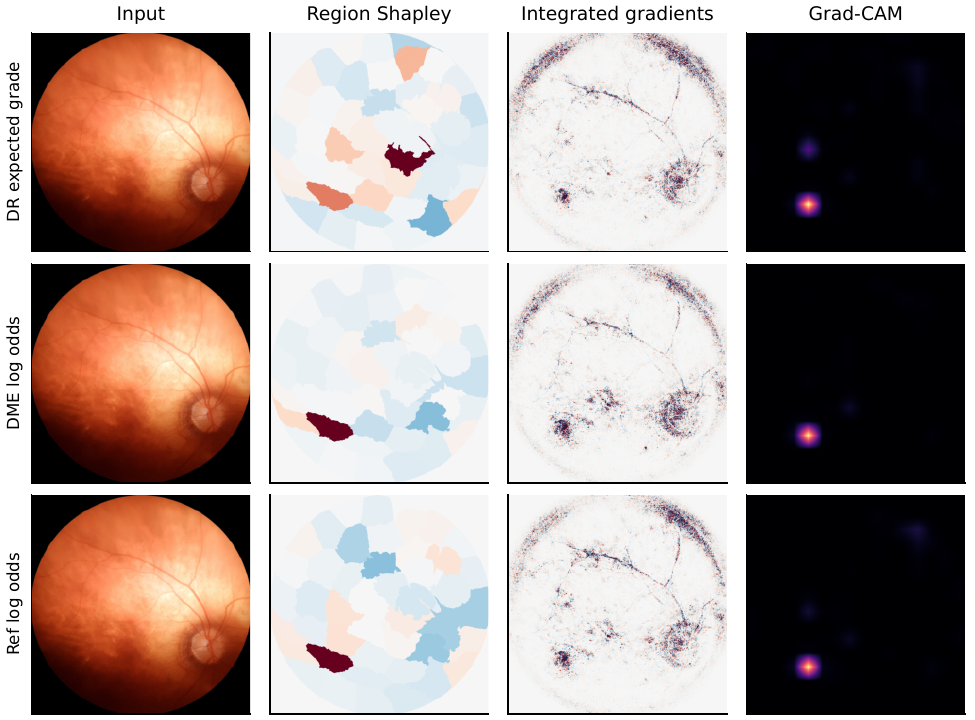}\end{center}
\begin{center}\footnotesize
\setlength{\tabcolsep}{3pt}
\begin{tabular}{lrrr}
\toprule
IG target & Output contrast & Absolute residual & Relative residual \\
\midrule
DR expected grade & -0.01100 & 0.00164 & 14.88\% \\
DME log odds & -3.90492 & 0.229 & 5.85\% \\
Ref log odds & -4.04423 & 0.229 & 5.66\% \\
\bottomrule\end{tabular}\end{center}
1 of three primary IG outputs exceed 10\% relative completeness residual; interpret the corresponding maps as finite-budget exploratory results. Signed attribution increases or decreases the named output relative to the blurred reference. A high DR expected-grade contribution is not a contribution to every DR class probability. Binary argmax classes shown here use 0.5 and can differ from validation-Youden decisions.

\clearpage\subsection{Case 24: Additional dme class 1}
Tier: broad; actual regions: 49; antithetic permutation pairs: 2; IG nodes: 32. Selection follows the recorded case rule, not the visual attractiveness of the map.
\begin{center}\footnotesize
\setlength{\tabcolsep}{3pt}
\begin{tabular}{lrrrr}
\toprule
Task & Observed label & Argmax class & Max probability & Explained output \\
\midrule
DR & 2 & 2 & 0.9217 & 1.8731 \\
DME & 1 & 1 & 0.9998 & 8.3803 \\
Ref & 1 & 1 & 0.9995 & 7.5862 \\
\bottomrule\end{tabular}\end{center}
\begin{center}\includegraphics[width=\linewidth]{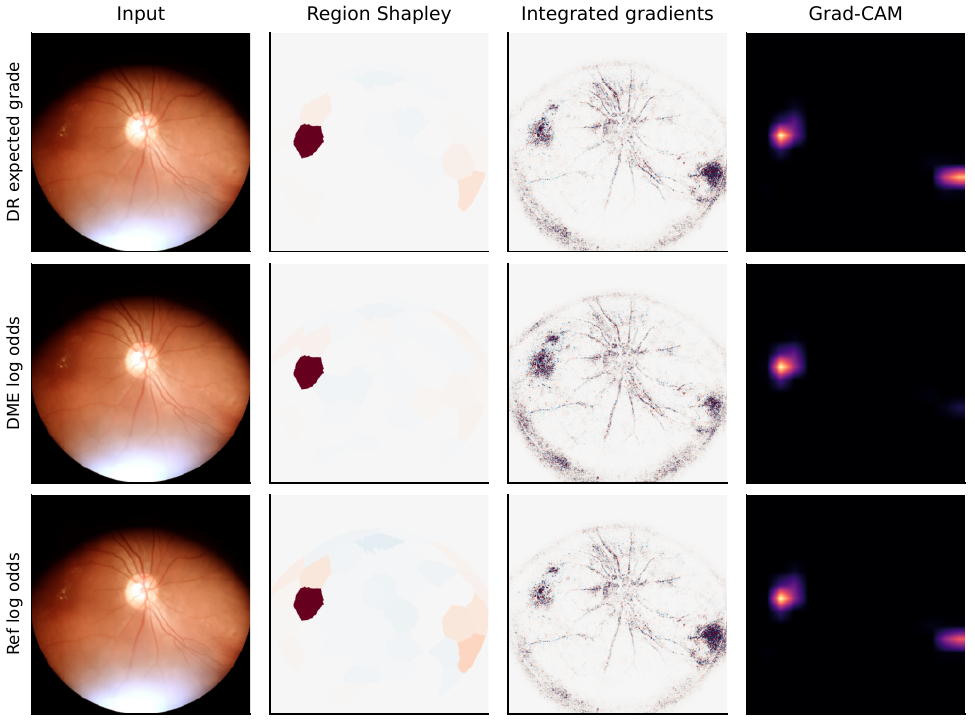}\end{center}
\begin{center}\footnotesize
\setlength{\tabcolsep}{3pt}
\begin{tabular}{lrrr}
\toprule
IG target & Output contrast & Absolute residual & Relative residual \\
\midrule
DR expected grade & 1.85353 & 0.118 & 6.39\% \\
DME log odds & 11.67322 & 0.847 & 7.25\% \\
Ref log odds & 8.92409 & 0.449 & 5.03\% \\
\bottomrule\end{tabular}\end{center}
All three primary IG outputs have relative completeness residual at most 10\%. Signed attribution increases or decreases the named output relative to the blurred reference. A high DR expected-grade contribution is not a contribution to every DR class probability. Binary argmax classes shown here use 0.5 and can differ from validation-Youden decisions.

\clearpage\subsection{Case 25: Additional ref class 0}
Tier: broad; actual regions: 49; antithetic permutation pairs: 2; IG nodes: 32. Selection follows the recorded case rule, not the visual attractiveness of the map.
\begin{center}\footnotesize
\setlength{\tabcolsep}{3pt}
\begin{tabular}{lrrrr}
\toprule
Task & Observed label & Argmax class & Max probability & Explained output \\
\midrule
DR & 0 & 0 & 0.9995 & 0.0005 \\
DME & 0 & 0 & 1.0000 & -11.7806 \\
Ref & 0 & 0 & 1.0000 & -12.2757 \\
\bottomrule\end{tabular}\end{center}
\begin{center}\includegraphics[width=\linewidth]{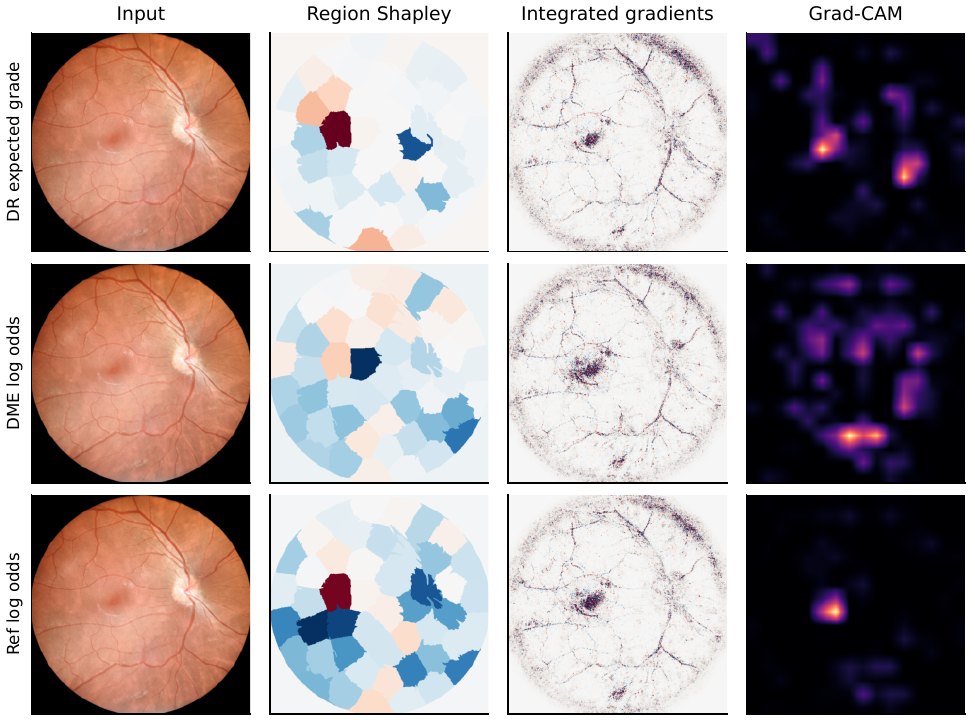}\end{center}
\begin{center}\footnotesize
\setlength{\tabcolsep}{3pt}
\begin{tabular}{lrrr}
\toprule
IG target & Output contrast & Absolute residual & Relative residual \\
\midrule
DR expected grade & -0.00107 & 0.000113 & 10.52\% \\
DME log odds & -4.69424 & 0.051 & 1.09\% \\
Ref log odds & -6.78804 & 0.114 & 1.68\% \\
\bottomrule\end{tabular}\end{center}
1 of three primary IG outputs exceed 10\% relative completeness residual; interpret the corresponding maps as finite-budget exploratory results. Signed attribution increases or decreases the named output relative to the blurred reference. A high DR expected-grade contribution is not a contribution to every DR class probability. Binary argmax classes shown here use 0.5 and can differ from validation-Youden decisions.

\clearpage\subsection{Case 26: Additional ref class 1}
Tier: broad; actual regions: 46; antithetic permutation pairs: 2; IG nodes: 32. Selection follows the recorded case rule, not the visual attractiveness of the map.
\begin{center}\footnotesize
\setlength{\tabcolsep}{3pt}
\begin{tabular}{lrrrr}
\toprule
Task & Observed label & Argmax class & Max probability & Explained output \\
\midrule
DR & 2 & 2 & 0.7953 & 2.2543 \\
DME & 1 & 1 & 1.0000 & 10.3280 \\
Ref & 1 & 1 & 1.0000 & 13.4167 \\
\bottomrule\end{tabular}\end{center}
\begin{center}\includegraphics[width=\linewidth]{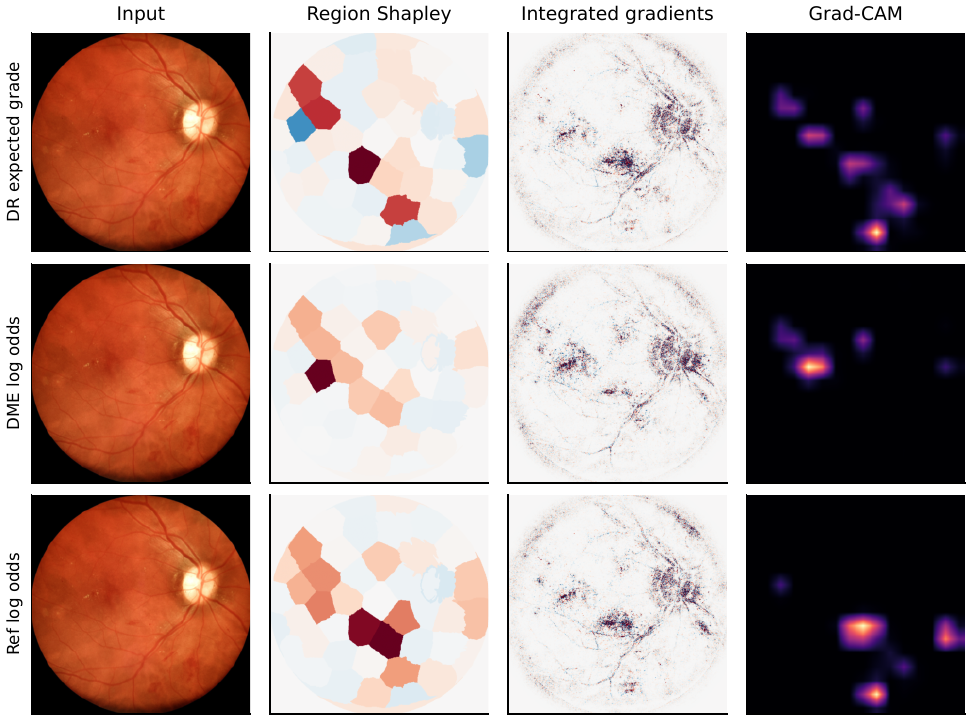}\end{center}
\begin{center}\footnotesize
\setlength{\tabcolsep}{3pt}
\begin{tabular}{lrrr}
\toprule
IG target & Output contrast & Absolute residual & Relative residual \\
\midrule
DR expected grade & 2.19834 & 0.19 & 8.65\% \\
DME log odds & 14.98254 & 0.754 & 5.03\% \\
Ref log odds & 13.89305 & 0.00662 & 0.05\% \\
\bottomrule\end{tabular}\end{center}
All three primary IG outputs have relative completeness residual at most 10\%. Signed attribution increases or decreases the named output relative to the blurred reference. A high DR expected-grade contribution is not a contribution to every DR class probability. Binary argmax classes shown here use 0.5 and can differ from validation-Youden decisions.

\clearpage\subsection{Case 27: Additional DR grade 0}
Tier: broad; actual regions: 49; antithetic permutation pairs: 2; IG nodes: 32. Selection follows the recorded case rule, not the visual attractiveness of the map.
\begin{center}\footnotesize
\setlength{\tabcolsep}{3pt}
\begin{tabular}{lrrrr}
\toprule
Task & Observed label & Argmax class & Max probability & Explained output \\
\midrule
DR & 0 & 0 & 0.9948 & 0.0052 \\
DME & 0 & 0 & 1.0000 & -10.8024 \\
Ref & 0 & 0 & 1.0000 & -10.3282 \\
\bottomrule\end{tabular}\end{center}
\begin{center}\includegraphics[width=\linewidth]{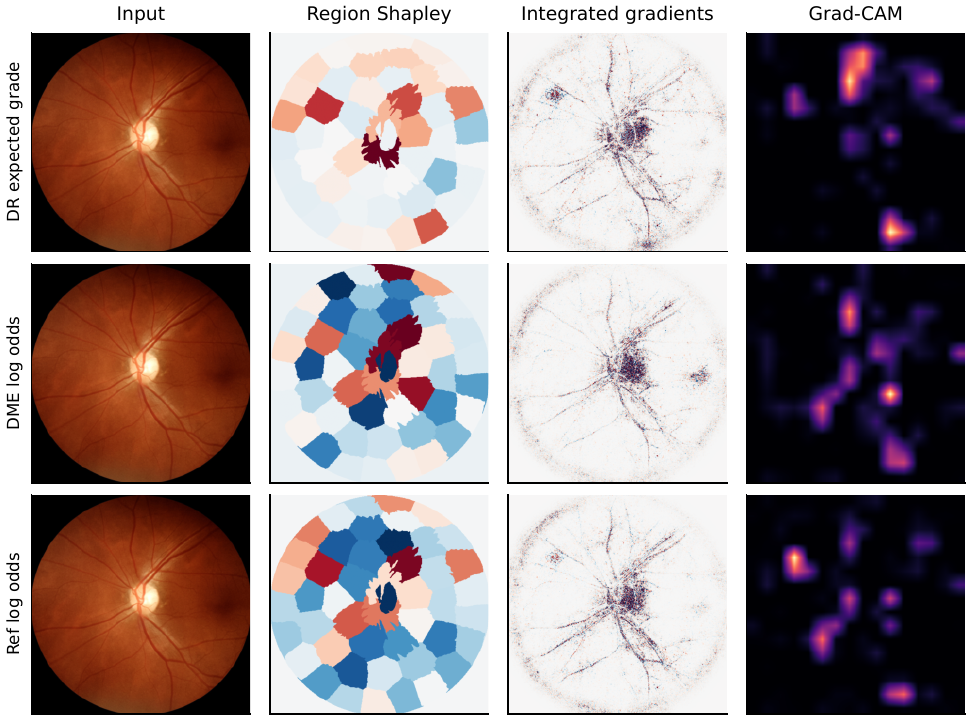}\end{center}
\begin{center}\footnotesize
\setlength{\tabcolsep}{3pt}
\begin{tabular}{lrrr}
\toprule
IG target & Output contrast & Absolute residual & Relative residual \\
\midrule
DR expected grade & 0.00442 & 8.32e-05 & 1.88\% \\
DME log odds & -3.01008 & 0.155 & 5.16\% \\
Ref log odds & -4.54410 & 0.31 & 6.82\% \\
\bottomrule\end{tabular}\end{center}
All three primary IG outputs have relative completeness residual at most 10\%. Signed attribution increases or decreases the named output relative to the blurred reference. A high DR expected-grade contribution is not a contribution to every DR class probability. Binary argmax classes shown here use 0.5 and can differ from validation-Youden decisions.

\clearpage\subsection{Case 28: Additional DR grade 1}
Tier: broad; actual regions: 48; antithetic permutation pairs: 2; IG nodes: 32. Selection follows the recorded case rule, not the visual attractiveness of the map.
\begin{center}\footnotesize
\setlength{\tabcolsep}{3pt}
\begin{tabular}{lrrrr}
\toprule
Task & Observed label & Argmax class & Max probability & Explained output \\
\midrule
DR & 1 & 0 & 0.8397 & 0.2337 \\
DME & 0 & 0 & 0.9999 & -9.2563 \\
Ref & 0 & 0 & 0.7292 & -0.9906 \\
\bottomrule\end{tabular}\end{center}
\begin{center}\includegraphics[width=\linewidth]{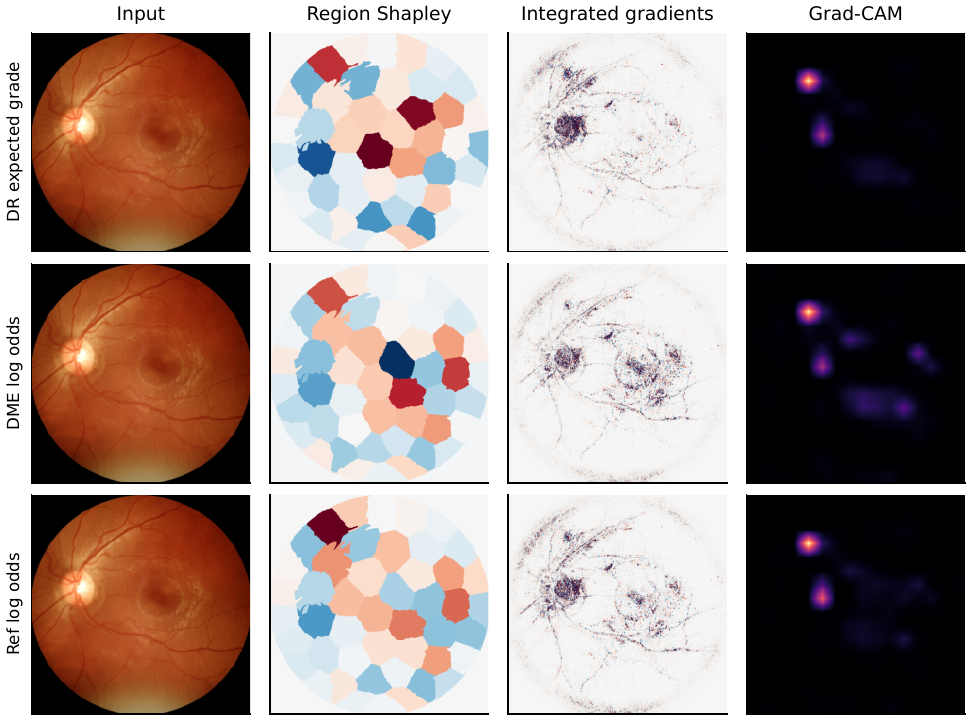}\end{center}
\begin{center}\footnotesize
\setlength{\tabcolsep}{3pt}
\begin{tabular}{lrrr}
\toprule
IG target & Output contrast & Absolute residual & Relative residual \\
\midrule
DR expected grade & 0.22365 & 0.00604 & 2.70\% \\
DME log odds & -1.06883 & 0.0288 & 2.70\% \\
Ref log odds & 0.69018 & 0.233 & 33.83\% \\
\bottomrule\end{tabular}\end{center}
1 of three primary IG outputs exceed 10\% relative completeness residual; interpret the corresponding maps as finite-budget exploratory results. Signed attribution increases or decreases the named output relative to the blurred reference. A high DR expected-grade contribution is not a contribution to every DR class probability. Binary argmax classes shown here use 0.5 and can differ from validation-Youden decisions.

\clearpage\subsection{Case 29: Additional DR grade 2}
Tier: broad; actual regions: 49; antithetic permutation pairs: 2; IG nodes: 32. Selection follows the recorded case rule, not the visual attractiveness of the map.
\begin{center}\footnotesize
\setlength{\tabcolsep}{3pt}
\begin{tabular}{lrrrr}
\toprule
Task & Observed label & Argmax class & Max probability & Explained output \\
\midrule
DR & 2 & 1 & 0.4570 & 1.3991 \\
DME & 0 & 0 & 0.9041 & -2.2438 \\
Ref & 1 & 1 & 0.8360 & 1.6287 \\
\bottomrule\end{tabular}\end{center}
\begin{center}\includegraphics[width=\linewidth]{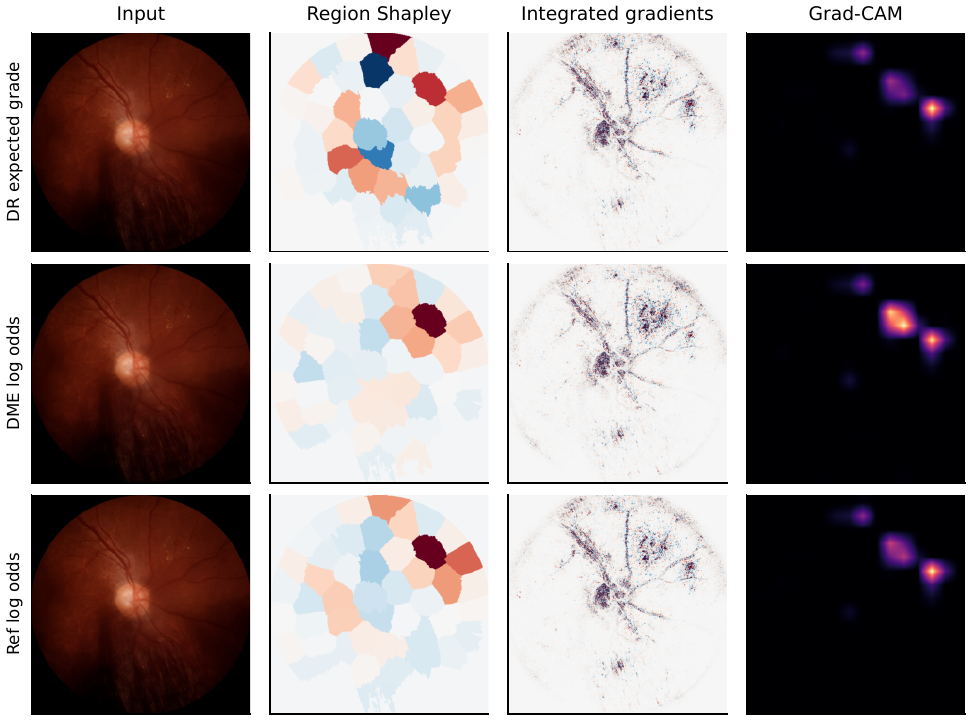}\end{center}
\begin{center}\footnotesize
\setlength{\tabcolsep}{3pt}
\begin{tabular}{lrrr}
\toprule
IG target & Output contrast & Absolute residual & Relative residual \\
\midrule
DR expected grade & 1.35176 & 0.0244 & 1.81\% \\
DME log odds & 5.97862 & 0.85 & 14.21\% \\
Ref log odds & 3.29715 & 0.77 & 23.34\% \\
\bottomrule\end{tabular}\end{center}
2 of three primary IG outputs exceed 10\% relative completeness residual; interpret the corresponding maps as finite-budget exploratory results. Signed attribution increases or decreases the named output relative to the blurred reference. A high DR expected-grade contribution is not a contribution to every DR class probability. Binary argmax classes shown here use 0.5 and can differ from validation-Youden decisions.

\clearpage\subsection{Case 30: Additional DR grade 4}
Tier: broad; actual regions: 49; antithetic permutation pairs: 2; IG nodes: 32. Selection follows the recorded case rule, not the visual attractiveness of the map.
\begin{center}\footnotesize
\setlength{\tabcolsep}{3pt}
\begin{tabular}{lrrrr}
\toprule
Task & Observed label & Argmax class & Max probability & Explained output \\
\midrule
DR & 4 & 4 & 0.9989 & 3.9984 \\
DME & 0 & 1 & 0.8779 & 1.9728 \\
Ref & 1 & 1 & 1.0000 & 11.0345 \\
\bottomrule\end{tabular}\end{center}
\begin{center}\includegraphics[width=\linewidth]{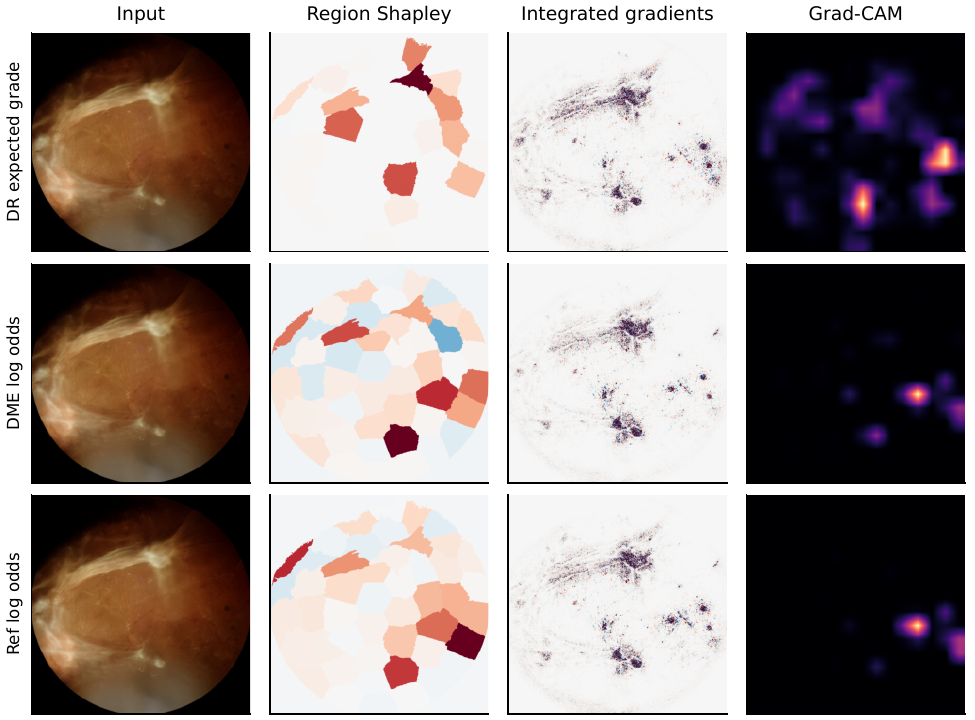}\end{center}
\begin{center}\footnotesize
\setlength{\tabcolsep}{3pt}
\begin{tabular}{lrrr}
\toprule
IG target & Output contrast & Absolute residual & Relative residual \\
\midrule
DR expected grade & 3.96642 & 0.277 & 6.99\% \\
DME log odds & 7.06635 & 0.235 & 3.32\% \\
Ref log odds & 10.90055 & 0.197 & 1.81\% \\
\bottomrule\end{tabular}\end{center}
All three primary IG outputs have relative completeness residual at most 10\%. Signed attribution increases or decreases the named output relative to the blurred reference. A high DR expected-grade contribution is not a contribution to every DR class probability. Binary argmax classes shown here use 0.5 and can differ from validation-Youden decisions.

\clearpage\subsection{Case 31: Additional dme class 0}
Tier: broad; actual regions: 48; antithetic permutation pairs: 2; IG nodes: 32. Selection follows the recorded case rule, not the visual attractiveness of the map.
\begin{center}\footnotesize
\setlength{\tabcolsep}{3pt}
\begin{tabular}{lrrrr}
\toprule
Task & Observed label & Argmax class & Max probability & Explained output \\
\midrule
DR & 0 & 0 & 0.9995 & 0.0005 \\
DME & 0 & 0 & 1.0000 & -13.1922 \\
Ref & 0 & 0 & 1.0000 & -12.9775 \\
\bottomrule\end{tabular}\end{center}
\begin{center}\includegraphics[width=\linewidth]{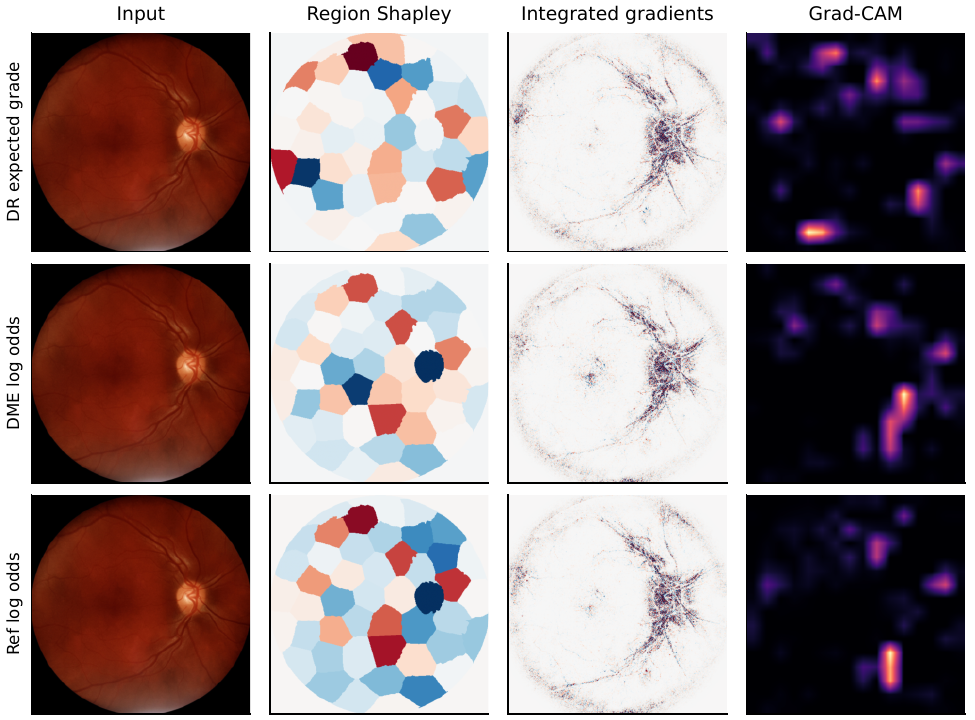}\end{center}
\begin{center}\footnotesize
\setlength{\tabcolsep}{3pt}
\begin{tabular}{lrrr}
\toprule
IG target & Output contrast & Absolute residual & Relative residual \\
\midrule
DR expected grade & 0.00026 & 8.5e-05 & 32.38\% \\
DME log odds & -1.52833 & 0.294 & 19.23\% \\
Ref log odds & -3.42315 & 0.137 & 4.00\% \\
\bottomrule\end{tabular}\end{center}
2 of three primary IG outputs exceed 10\% relative completeness residual; interpret the corresponding maps as finite-budget exploratory results. Signed attribution increases or decreases the named output relative to the blurred reference. A high DR expected-grade contribution is not a contribution to every DR class probability. Binary argmax classes shown here use 0.5 and can differ from validation-Youden decisions.

\clearpage\subsection{Case 32: Additional dme class 1}
Tier: broad; actual regions: 49; antithetic permutation pairs: 2; IG nodes: 32. Selection follows the recorded case rule, not the visual attractiveness of the map.
\begin{center}\footnotesize
\setlength{\tabcolsep}{3pt}
\begin{tabular}{lrrrr}
\toprule
Task & Observed label & Argmax class & Max probability & Explained output \\
\midrule
DR & 4 & 4 & 0.9624 & 3.9335 \\
DME & 1 & 1 & 0.5090 & 0.0360 \\
Ref & 1 & 1 & 1.0000 & 10.6877 \\
\bottomrule\end{tabular}\end{center}
\begin{center}\includegraphics[width=\linewidth]{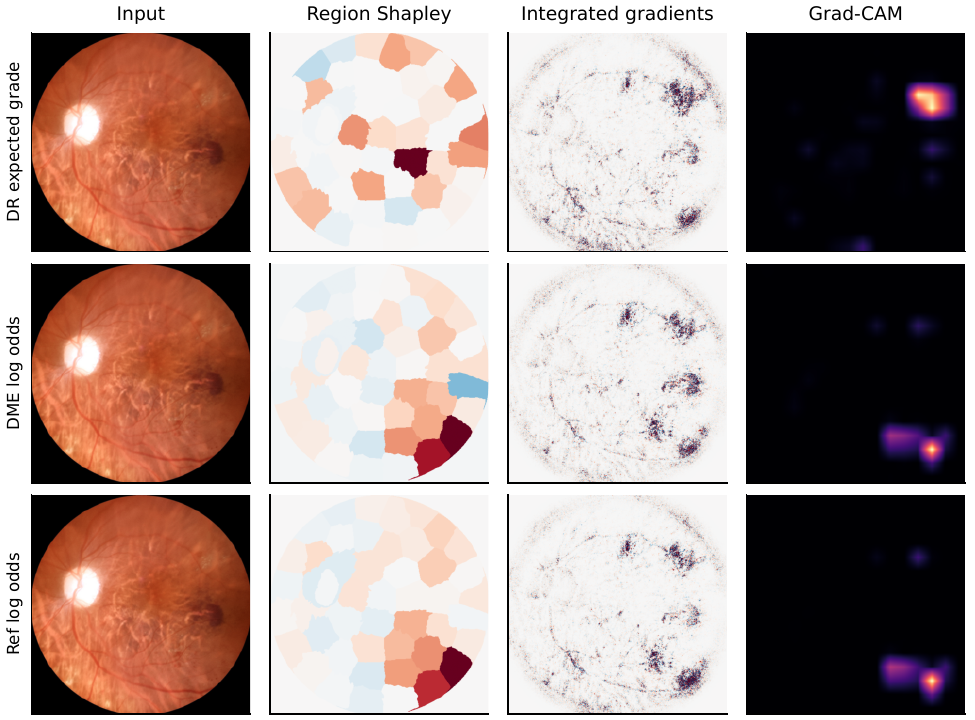}\end{center}
\begin{center}\footnotesize
\setlength{\tabcolsep}{3pt}
\begin{tabular}{lrrr}
\toprule
IG target & Output contrast & Absolute residual & Relative residual \\
\midrule
DR expected grade & 3.91129 & 0.0738 & 1.89\% \\
DME log odds & 9.14785 & 0.265 & 2.89\% \\
Ref log odds & 12.88617 & 0.133 & 1.03\% \\
\bottomrule\end{tabular}\end{center}
All three primary IG outputs have relative completeness residual at most 10\%. Signed attribution increases or decreases the named output relative to the blurred reference. A high DR expected-grade contribution is not a contribution to every DR class probability. Binary argmax classes shown here use 0.5 and can differ from validation-Youden decisions.

\clearpage
\section{Artifact Traceability and Interpretation Boundaries}\label{app:reproducibility}
\subsection{Frozen reference and independent checks}
The reported main checkpoint is the actual seed-2028, patience-10 run, selected at epoch 13 and stopped at epoch 23. Its checkpoint SHA256 is
\begin{center}\small\texttt{e2361edc1ba94437f2b97210d619e58bee07e5e6dec6a1e99e0d5949507f0d07}.\end{center}
The earlier patience-20 trajectory is historical. Exact agreement of its selected tensors with the later patience-10 run does not constitute an additional independent training seed. Main results use one explicit checkpoint identity rather than combining the best metrics from different runs.

Each complete experiment bundle records configuration, preprocessing, source hashes, train/validation/test identifiers and masks, logits and probabilities, optimization history, environment and model state. Receipt checks validate archive paths and hashes, reconstruct metrics from raw predictions, and check model-specific forward paths. For the external adaptations, numerical tolerances are interpreted with the scale of the inspected tensor. Previously recorded raw-feature tolerance exceptions for HEALNet and TIP are retained in the internal audit, while their final prediction replay checks pass. They are not described as perfect bitwise cross-device equality.

The manuscript tables are generated from a single long-form metrics source covering the complete frozen comparison ledger. Display rounding occurs only after the numerical values are read. Where used in non-ablation comparisons, boldface denotes the largest displayed-group point estimate, not statistical significance. Ablation values use uniform font weight. Tables separate the two baseline recipe groups and the main-model supervised group. The supplied source tree keeps the bibliography, equations, generated tables and figure files separate to support review of both content and provenance.

\subsection{Scope of measured runtime}
The final mechanism package ran on an NVIDIA A100-SXM4-80GB. The four logged analysis actions total approximately 17.74 minutes; elapsed session time through summarization is approximately 26.50 minutes and excludes final export/download. The supplementary attribution actions total approximately 26.89 minutes, again excluding preparation/export. These are analysis runtimes on already-trained weights. Training durations include different histories and, for TIP, an extra self-supervised stage; they are not controlled throughput measurements and do not establish an efficiency ranking.

The saved best/last checkpoints and preprocessing permit later inference and analysis without retraining. Attribution integrations may require additional compute if a different baseline, larger sampling budget or new output target is requested. Saved maps and token caches cover the declared experiments, not every possible future explanation.

\subsection{What the evidence establishes}
The central empirical finding is strong DR prediction under the reported protocol, including higher accuracy, QWK and macro-F1 point estimates than the three evaluated complete external adaptations. The central methodological finding is a structured, positive-definite fusion system whose solution, spectra and local derivatives can be independently checked. The observed image-side matrix intervention effects connect directional geometry to specific DR decisions at the selected checkpoint.

These findings do not require every auxiliary metric to improve. DME and Ref rankings differ across objectives and thresholds, some ablations improve particular metrics, and internal state adaptation can help one task while harming another under the same corruption. Those outcomes delimit the complete model's demonstrated benefits.

\subsection{What remains outside the demonstrated scope}
The study does not establish external-cohort clinical generalization, prospective benefit, calibration of q/u/r as correctness probabilities, uncertainty decomposition into epistemic and aleatoric terms, or superiority over external methods under a common corruption suite. The learned state has intended semantics but only task- and condition-specific empirical support. Completeness-anchored r should not be interpreted as a detector of all incorrect metadata.

The principal comparisons use a single matched training seed. The clean224 baselines and historical clean unimodal references remain useful contextual controls, but their differing augmentation, resolution, batch size and stopping rules prevent an architecture-only causal comparison with the matched512 group.

No assertion of lesion localization validity follows from visually plausible maps. Image attribution uses a limited inspection cohort with deliberately varied success/error cases. Background choice, segmentation, quadrature precision and output definition can change an explanation. Global metadata importance can reflect correlated covariates or repeated patient context and should not be equated with the effect of a clinical intervention.

The supplied analysis therefore supports a precise claim: \method{} is a useful and inspectable image--metadata fusion configuration for the evaluated DR-centered multitask problem, with a mathematically explicit latent estimate and a fully recorded empirical characterization. Broader deployment or universal robustness claims require additional evidence.

\label{paper:end}
\end{document}